\documentclass[nohyperref]{article}

\usepackage{microtype}
\usepackage{graphicx}
\usepackage{booktabs} 

\usepackage{hyperref}

\usepackage[accepted]{icml2022}

\usepackage{amsmath}
\usepackage{amssymb}
\usepackage{mathtools}
\usepackage{amsthm}

\usepackage[capitalize,noabbrev]{cleveref}

\theoremstyle{plain}
\newtheorem{theorem}{Theorem}[section]

\theoremstyle{definition}
\newtheorem{definition}[theorem]{Definition}

\theoremstyle{remark}
\newtheorem{remark}[theorem]{Remark}

\usepackage[textsize=tiny]{todonotes}

\usepackage{amsmath,amsfonts,bm}

\def\eqref#1{equation~\ref{#1}}

\def\1{\bm{1}}

\def\ra{{\textnormal{a}}}

\def\rc{{\textnormal{c}}}

\def\ro{{\textnormal{o}}}

\def\rs{{\textnormal{s}}}

\def\rva{{\mathbf{a}}}

\def\vtheta{{\bm{\theta}}}
\def\va{{\bm{a}}}
\def\vb{{\bm{b}}}
\def\vc{{\bm{c}}}

\def\vg{{\bm{g}}}

\def\vp{{\bm{p}}}
\def\vq{{\bm{q}}}
\def\vr{{\bm{r}}}
\def\vs{{\bm{s}}}

\def\vv{{\bm{v}}}

\def\vx{{\bm{x}}}

\def\mA{{\bm{A}}}
\def\mB{{\bm{B}}}

\def\mH{{\bm{H}}}

\def\mS{{\bm{S}}}

\DeclareMathAlphabet{\mathsfit}{\encodingdefault}{\sfdefault}{m}{sl}
\SetMathAlphabet{\mathsfit}{bold}{\encodingdefault}{\sfdefault}{bx}{n}

\newcommand{\E}{\mathbb{E}}

\DeclareMathOperator*{\argmax}{arg\,max}

\usepackage{url}
\usepackage{subcaption}
\usepackage{algorithmic}
\usepackage{amssymb}
\usepackage{xcolor}   
\usepackage{amsmath}
\usepackage{amsthm,amsmath,amssymb}
\usepackage{mathrsfs}
\usepackage{xcolor}

\usepackage[flushleft]{threeparttable}
\usepackage{graphicx}
\usepackage{tcolorbox}
\usepackage[page, header]{appendix}
\usepackage{titletoc}
\usepackage{lipsum}
\usepackage{hyperref}
\usepackage{multirow}
\usepackage{booktabs}
\usepackage{threeparttable}

\usepackage[flushleft]{threeparttable}
\usepackage{threeparttable}
\usepackage{booktabs}
\usepackage{comment}
\usepackage{amsthm}
\usepackage{amsmath}
\usepackage{apxproof}
\usepackage{thmtools, thm-restate}
\usepackage{wrapfig}
\usepackage{adjustbox}
\usepackage[flushleft]{threeparttable}
\usepackage{graphicx}

\usepackage{dashrule}

\usepackage{ulem}
\usepackage{cancel}

\newcommand{\shangding}[1]{{ \color{magenta}[shangding: #1]}}

\DeclareMathAlphabet\mathbfcal{OMS}{cmsy}{b}{n}
\newenvironment{smalleralign}[1][\small]
 {\par\nopagebreak\leavevmode\vspace*{-\baselineskip}%
  \skip0=\abovedisplayskip
  #1%
  \def\maketag@@@##1{\hbox{\m@th\normalfont\normalsize##1}}%
  \abovedisplayskip=\skip0
  \align}
 {\endalign\ignorespacesafterend}

\begin{document}

\twocolumn[
\icmltitle{Multi-Agent Constrained Policy Optimisation}



\icmlsetsymbol{equal}{*}

\begin{icmlauthorlist}
\icmlauthor{Shangding Gu}{equal,yyy}
\icmlauthor{Jakub Grudzien Kuba}{equal,comp}
\icmlauthor{Munning Wen}{sch}
\icmlauthor{Ruiqing Chen}{shtech}
\icmlauthor{Ziyan Wang}{ucl}
\icmlauthor{Zheng Tian}{shtech}
\icmlauthor{Jun Wang}{ucl}
\icmlauthor{Alois Knoll}{yyy}
\icmlauthor{Yaodong Yang}{pku}
\end{icmlauthorlist}

\icmlaffiliation{yyy}{Technical University of Munich}
\icmlaffiliation{comp}{University of Oxford}
\icmlaffiliation{sch}{Shanghai Jiao Tong University}
\icmlaffiliation{shtech}{ShanghaiTech University}
\icmlaffiliation{ucl}{University College London}
\icmlaffiliation{pku}{Institute for AI, Peking University \& BIGAI}

\icmlcorrespondingauthor{Yaodong Yang}{yaodong.yang@pku.edu.cn}

\icmlkeywords{Machine Learning, ICML}

\vskip 0.3in
]



\printAffiliationsAndNotice{\icmlEqualContribution} 

\begin{abstract}
Developing intelligent agents that satisfy safety constraints is becoming increasingly important for many real-world applications. In multi-agent reinforcement learning (MARL) settings, policy optimisation with safety awareness is particularly challenging because each individual agent has to not only meet its own safety constraints, but also consider those of others so that their joint behaviours are guaranteed safe.    
Despite its importance, the research problem of safe multi-agent learning  has not been rigorously studied, nor a shareable testing environment or benchmarks.   
To fill these gaps, 
in this work, we formulate the safe MARL problem as a constrained Markov game and solve it with policy optimisation methods.
We propose two algorithms, \textsl{Multi-Agent Constrained Policy Optimisation} (MACPO) and \textsl{MAPPO-Lagrangian}, which leverage the theories from both  constrained policy optimisation and multi-agent trust region learning. Crucially, our methods enjoy theoretical guarantees of  both  monotonic improvement in reward and satisfaction of safety constraints  at every iteration. 
To examine the effectiveness of our methods, we develop the benchmark suites of \textsl{Safe Multi-Agent MuJoCo}  and \textsl{Safe Multi-Agent Robosuite} that involve a variety of  MARL baselines.   Experimental results justify that MACPO and MAPPO-Lagrangian can consistently satisfy safety constraints, meanwhile achieving comparable performance in rewards. Videos and code are released at 
the link~\footnote{ \url{https://github.com/chauncygu/Multi-Agent-Constrained-Policy-Optimisation.git}}.
\end{abstract}



\vspace{-15pt}
\section{Introduction}
In recent years,  reinforcement learning (RL) techniques  have achieved remarkable successes on a variety of  complex tasks \citep{silver2016mastering,silver2017mastering, vinyals2019grandmaster,yang2017study,zhou2020smarts}.  
Powered by deep neural networks, deep RL enables learning sophisticated behaviours. 
On the other hand, deploying neural networks turns the optimisation procedure from policy space to  parameter space; this enables gradient-based methods to be applied \citep{sutton1999policy, lillicrap2015continuous, schulman2017proximal}. For policy gradient methods, at every iteration, the parameters of a policy network are updated in the direction of the gradient that maximises return. 

However, policies that are purely optimised for reward maximisation are rarely applicable to real-world problems. 
In many  applications, an agent is often required not to 
visit certain states or take certain actions, which are thought of as ``unsafe'' either for itself or for other elements in the background \citep{moldovan2012safe, achiam2017constrained}.
For instance, a robot carrying materials
in a warehouse should not damage its parts while delivering an item to a shelf, nor should a self-driving car cross on the red light while rushing towards its destination \citep{shalev2016safe}.
To tackle these issues, \textsl{Safe RL} \citep{moldovan2012safe, garcia2015comprehensive} is  proposed, aiming to develop algorithms that learn policies that satisfy safety constraints. Despite the additional requirement of safety on solutions, algorithms with convergence guarantees have been proposed \citep{xu2021crpo, wei2021provably}.

Developing safe policies for multi-agent systems is  a challenging task. 
Part of the difficulty comes from solving multi-agent reinforcement learning (MARL) problems itself \citep{deng2021complexity, kuba2021settling}; more importantly, tackling safety in MARL is hard  because  each individual agent has to not only consider its own safety constraints, which  already may conflict its  reward maximisation, but also consider the safety constraints of others so that their joint behaviour is guaranteed to be safe. 
As a result, there are very few solutions that offer effective learning algorithms for safe MARL problems. 

We believe agents'  coordinations is critical to guarantee a safe multi-agent system, since agents'  joint  actions can still cause hazards  even though each individual one is  locally safe \footnote{Think about two cars that both decide to cross an unprotected junction without being aware of the existence of the other car.}.   
In fact, many of the existing methods focus on learning to cooperate \citep{yang2020overview}. 
 For example, \citet{rashid2018qmix} and \citet{yang2020multi} adopt greedy maximisation  on the local component of a monotonic joint value function, \citet{foerster2018counterfactual} estimates the policy gradient
based on the counterfactual value from a joint critic, and \citet{yu2021surprising} also employs a joint critic for proximal policy updates.   
Despite their effectiveness on cooperative tasks, it is  unclear yet how to directly incorporate safety constraints into these solution frameworks. 
Consequently, how to coordinate agents  towards reward maximisation meanwhile satisfying  safety constraints remains an unsolved  problem. 

The goal of this paper is to increase  practicality of MARL algorithms through endowing them with safety awareness. For this purpose, we introduce a general framework to formulate safe MARL problems, and solve them through multi-agent policy optimisation methods. 
Our solutions manage to extend the constrained policy optimisation technique \citep{achiam2017constrained} to multi-agent settings. Crucially,   
the resulting  algorithm attains properties of both monotonic improvement guarantee and constraints satisfaction guarantee at every iteration during training. 
Our key contributions are  two safe MARL algorithms: MACPO and MAPPO-Lagrangian. 
Additionally, we also develop the first safe MARL benchmark suite within the MuJoCo and Robosuite environments and implement a variety of MARL baselines.  
We evaluate MACPO and MAPPO-Lagrangian on a series of multi-agent control tasks, and results clearly confirm the effectiveness of our solutions both in terms of constraints satisfaction and reward maximisation. 
To our best knowledge, MACPO and MAPPO-Lagrangian are the first safety-aware model-free MARL algorithms that work effectively in the challenging continuous control tasks such as  MuJoCo and Robosuite.   


\section{Related Work}
Considering safety in the development of AI is a long-standing topic \citep{amodei2016concrete}. 
When it comes to safe reinforcement learning \citep{garcia2015comprehensive}, a commonly used  framework is  \textsl{Constrained Markov Decision Processes} (CMDPs) \citep{altman1999constrained}. In a CMDP, at every step, in addition to the reward, the environment emits  costs associated with certain constraints. As a result, the learning agent must  try to satisfy those   constraints   while maximising the total reward.  
In general, the cost from the environment can be thought of as a measure of safety. Under the framework of CMDP,  a safe policy is the one that explores the environment safely by keeping the total costs under certain thresholds. To tackle the learning problem in CMDPs, \citet{achiam2017constrained} introduced \textsl{Constrained Policy Optimisation} (CPO), which updates agent's policy under the trust region constraint \citep{schulman2015trust} to maximise surrogate return while obeying surrogate cost constraints. 
However, solving a constrained optimisation at every iteration of CPO can be 
cumbersome for implementation. An alternative solution is to apply primal-dual methods,  giving rise to methods like TRPO-Lagrangian and PPO-Lagrangian \citep{ray2019benchmarking}. 
Although these methods achieve impressive performance in terms of safety, the performance  in terms of reward is poor \citep{ray2019benchmarking}. 
Another class of algorithms that solves CMDPs is by  \cite{chow2018lyapunov,chow2019lyapunov}; these algorithms leverage the theoretical property of the Lyapunov functions and propose safe value iteration and policy gradient  procedures. In contrast to CPO, \citet{chow2018lyapunov,chow2019lyapunov} can  work with off-policy methods and can be trained end-to-end with no need for line search. 

 
Safe multi-agent learning is an emerging research domain. Despite its importance \citep{shalev2016safe},  there are few solutions that work with MARL in a model-free setting. 
CMIX~\citep{liu2021cmix} extends QMIX~\citep{rashid2018qmix} by amending the reward function to take peak constraint violations into account  to account for multi-objective constraint, yet this cannot deliver rigorous safety guarantee  during training.  
In fact, the majority of methods are designed for robotics learning. For example, the technique of   barrier certificates   \citep{borrmann2015control, ames2016control,qin2020learning} or model predictive shielding \citep{zhang2019mamps} from  control theory is used to model safety. 
These methods, however, are specifically derived for  robotics applications; they  either are  supervised learning based approaches, or require specific  assumptions on the state space and environment dynamics. 
Moreover, due to the lack of a benchmark suite for safe MARL algorithms,  
the generalisation ability of those methods is unclear. 

The most related work to ours is Safe Dec-PG \citep{lu2021decentralized} where they used the primal-dual  framework to find the saddle point between maximising reward and minimising cost. 
In particular, they proposed a  decentralised  policy descent-ascent method through a consensus network. 
However, reaching a consensus equivalently  imposes an extra constraint of parameter sharing among neighbouring agents, which could yield suboptimal solutions \citep{kuba2021trust}. 
Furthermore, multi-agent policy gradient methods can suffer from  high variance \citep{ kuba2021settling}. In contrast, our methods  employ trust region policy optimisation and do not assume any parameter sharing.  



HATRPO  \citep{kuba2021trust} introduced the first multi-agent trust region  method that  enjoys theoretically-justified monotonic improvement guarantee. 
Its key idea is to make agents follow a sequential policy update scheme so that the expected joint advantage will always be positive, thus increasing reward.  In this work, 
we show how to further develop this theory and derive a protocol which, in addition to the monotonic improvement, also guarantees to satisfy the safety constraint at every iteration during learning.  
The resulting algorithm (Algorithm \ref{algorithm:theoretical-safe-matrpo}) successfully attains theoretical guarantees of both monotonic improvement in reward and satisfaction of safety constraints.

\section{Problem Formulation}
\label{sec:problem-formulation}

We formulate the  safe MARL problem as a \textsl{constrained Markov game} $\langle \mathcal{N}, \mathcal{S}, \boldsymbol{\mathcal{A}}, \mathrm{p}, \rho^0, \gamma, R, \boldsymbol{C}, \boldsymbol{c}\rangle$. Here, $\mathcal{N}= \{1, \dots, n\}$ is the set of agents, $\mathcal{S}$ is the state space, $\boldsymbol{\mathcal{A}} = \prod_{i=1}^{n}\mathcal{A}^i$ is the product of the agents' action spaces, known as the joint action space, $\mathrm{p}: \mathcal{S} \times \boldsymbol{\mathcal{A}} \times \mathcal{S} \rightarrow \mathbb{R}$ is the probabilistic transition function, $\rho^0$ is the initial state distribution, $\gamma \in [0,1)$ is the discount factor, $R:\mathcal{S}\times\boldsymbol{\mathcal{A}}\rightarrow\mathbb{R}$ is the joint reward function, { $\boldsymbol{C} = \{C^i_j\}^{i\in\mathcal{N}}_{1 \leq j\leq m^i}$} is the set of sets of cost functions (every agent $i$ has $m^i$ cost functions) of the form $C^i_j:\mathcal{S}\times\mathcal{A}^i\rightarrow \mathbb{R}$, and finally the set of corresponding cost-constraining values is given by $\boldsymbol{c} = \{c^i_j\}^{i\in\mathcal{N}}_{1 \leq j\leq m^i}$. At time step $t$, the agents are in a state $\rs_t$, and every agent $i$ takes an action $\ra^i_t$ according to its policy $\pi^i(\ra^i|\rs_t)$.  Together with other agents' actions, it gives a joint action $\rva_t = (\ra^1_t, \dots, \ra^n_t)$ and the joint policy $\boldsymbol{\pi}(\rva|\rs) = \prod_{i=1}^{n}\pi^i(\ra^i|\rs)$. The agents receive the reward $R(\rs_t, \rva_t)$, meanwhile each agent $i$ pays the costs $C^i_j(\rs_t, \ra^i_t)$, $\forall j=1, \dots, m^i$. The environment then transits to a new state $\rs_{t+1} \sim \mathrm{p}(\cdot|\rs_t, \rva_t)$. 

In this paper, we consider a \textsl{fully-cooperative} setting where all agents share the same reward function, aiming to maximise the expected total reward of
\begin{align}
    J(\boldsymbol{\pi}) \triangleq \E_{\rs_0\sim\rho^0, \rva_{0:\infty}\sim\boldsymbol{\pi}, \rs_{1:\infty} \sim \mathrm{p}}\Big[ \sum\limits_{t=0}^{\infty}\gamma^t R(\rs_t, \rva_t) \Big], \nonumber
\end{align}
meanwhile trying to satisfy every agent $i$'s safety constraints, 
\begin{align}
    \label{eq:safety-objective}
    J^i_j(\boldsymbol{\pi}) & \triangleq
    \E_{\rs_0\sim\rho^0, \rva_{0:\infty}\sim\boldsymbol{\pi}, \rs_{1:\infty} \sim \mathrm{p}}\Big[ \sum\limits_{t=0}^{\infty}\gamma^t C^i_j(\rs_t, \ra^i_t) \Big] \leq c^i_j, \ \ \ \ \  \nonumber \\ 
    & \quad \quad \quad \quad  \quad  \quad \quad \forall j=1, \dots, m^i.
\end{align}

We define the state-action value and the state-value functions in terms of reward as
{\small
\begin{align}
Q_{\boldsymbol{\pi}}(s, \va) & \triangleq 
\E_{\rs_{1:\infty} \sim \mathrm{p}, \rva_{1:\infty}\sim\boldsymbol{\pi}}\Big[ \sum\limits_{t=0}^{\infty}\gamma^t R(\rs_t, \rva_t) \big| \rs_0 = s, \rva_0 = \va \Big], \ \ \ \ \nonumber
\\
&  \quad \text{and}  \quad  \quad 
V_{\boldsymbol{\pi}}(s) \triangleq \E_{\rva\sim\boldsymbol{\pi}}\big[ Q_{\boldsymbol{\pi}}(s, \rva) \big]. \nonumber
\end{align} }
The joint policies $\boldsymbol{\pi}$ that satisfy the Inequality (\ref{eq:safety-objective}) are referred to as \textbf{feasible}.
Notably,  in the above formulation, although the action $\ra_t^i$ of agent $i$ does not directly influence the costs $\{C^k_j(\rs_t, \ra_t^k)\}_{j=1}^{m^k}$ of other agents $k\neq i$, the action $\ra^i_t$ will implicitly influence their total costs due to the dependence on the next state $\rs_{t+1}$  \footnote{We believe that this formulation  realistically describes  multi-agent interactions in the real-world; an action of an agent has an instantaneous effect on the system only locally, but the rest of agents may suffer from its consequences at later stages. For example, consider a car that crosses on the red light, although other cars may not be at risk of riding into pedestrians immediately, the induced  traffic  may cause hazards  soon later.}. 
For the $j^{\text{th}}$ cost function of agent $i$, we define the $j^{\text{th}}$ state-action cost value function and the state cost value function as 
{\small
\begin{align}
    &Q^{i}_{j, \boldsymbol{\pi}}(s, a^i) \triangleq \nonumber \\&
   \quad \  \E_{\rva^{-i}\sim\boldsymbol{\pi}^{-i}, \rs_{1:\infty}\sim\mathrm{p}, \rva_{1:\infty}\sim\boldsymbol{\pi}}
    \Big[ \sum_{t=0}^{\infty}\gamma^t C^i_j(\rs_t, \ra^i_t) \ \big|
\rs_0 = s, \ra^i_0 = a^i \Big], \nonumber\\
    & V^{i}_{j, \boldsymbol{\pi}}(s) \triangleq 
    \E_{\rva\sim\boldsymbol{\pi}, \rs_{1:\infty}\sim\mathrm{p}, \rva_{1:\infty}\sim\boldsymbol{\pi}}
    \Big[ \sum_{t=0}^{\infty}\gamma^t C^i_j(\rs_t, \ra^i_t) \ \big|
    \ \rs_0 = s \Big]. \ \ \ \  \nonumber
\end{align}}
Notably,  the cost value functions $Q^i_{j, \boldsymbol{\pi}}$ and  $V^i_{j, \boldsymbol{\pi}}$, although  similar to traditional  $Q_{\boldsymbol{\pi}}$ and $V_{\boldsymbol{\pi}}$, involve extra indices $i$ and $j$; the superscript $i$ denotes an agent, and the subscript $j$ denotes its $j^{\text{th}}$ cost.

Throughout this work,  we pay a close attention  to the contribution to performance from  different subsets of agents, therefore, we introduce the following notations. We denote an arbitrary subset $\{i_1, \dots, i_h\}$ of agents as $i_{1:h}$; we write $-i_{1:h}$ to refer to its complement. 
Given the agent subset $i_{1:h}$, we define the  \textsl{multi-agent state-action value function}:  
\begin{align}
    Q^{i_{1:h}}_{\boldsymbol{\pi}}(s, \va^{i_{1:h}}) \triangleq 
    \E_{\rva^{-i_{1:h}}\sim\boldsymbol{\pi}^{-i_{1:h}}}\left[ Q_{\boldsymbol{\pi}}(\rs, \va^{i_{1:h}}, \rva^{-i_{1:h}}) \right]. 
    \nonumber
\end{align}

On top of it, for disjoint sets $j_{1: k}$ and $i_{1: m}$, the \textsl{multi-agent advantage function} \footnote{We would like to highlight that these multi-agent functions of   $Q_{\boldsymbol{\pi}}^{i_{1:h}}$ and $A_{\boldsymbol{\pi}}^{i_{1:h}}$, although involve agents in superscripts, describe values \textsl{w.r.t} the reward rather than costs since they do not involve cost subscripts.} is defined as follows, 
\begin{align}
    \label{eq:maad-definition}
   &  A^{i_{1:h}}_{\boldsymbol{\pi}}  \big(s,   \va^{j_{1:k}},  \va^{i_{1:h}}\big)  
    \triangleq \nonumber \\
    & \qquad \qquad  
    Q^{j_{1:k}, i_{1:h}}_{\boldsymbol{\pi}}\big(s, \va^{j_{1:k}}, \va^{i_{1:h}}\big) - 
    Q^{j_{1:k}}_{\boldsymbol{\pi}}\big(s, \va^{j_{1:k}}\big). \nonumber
\end{align} 

An interesting and important fact about the above multi-agent advantage function is that the  advantage $A^{i_{1:h}}_{\boldsymbol{\pi}}$ can be written as a sum of sequentially-unfolding multi-agent advantages of individual agents with no need for any assumptions on the joint value function (e.g., VDN \cite{sunehag2018value} and QMIX \cite{rashid2018qmix}), that is,  
\begin{restatable}[Multi-Agent Advantage Decomposition,  \cite{kuba2021settling}]{lemma}{maadlemma}
\label{lemma:maad}
For any state $s\in\mathcal{S}$, subset of agents $i_{1:h}\subseteq \mathcal{N}$, and joint action $\va^{i_{1:h}}$, the following identity holds
\begin{align}
    A_{\boldsymbol{\pi}}^{i_{1:h}}\big(s, \va^{i_{1:h}}\big) = \sum_{j=1}^{h}A_{\boldsymbol{\pi}}^{i_j}\big(s, \va^{i_{1:j-1}}, a^{i_j}\big).\nonumber
\end{align}
\vspace{-15pt}
\end{restatable}

\section{Multi-Agent Constrained Policy Optimisation}
In this section, we first present a theoretically-justified safe multi-agent policy iteration procedure, which leverages multi-agent trust region learning and constrained policy optimisation to solve constrained Markov games. Based on this, we propose two practical deep MARL algorithms, enabling optimising  neural-network based policies that satisfy safety constraints.  Throughout this work, we refer the symbols $\boldsymbol{\pi}$ and $\boldsymbol{\bar{\pi}}$ to be the ``current" and the ``new" joint policies, respectively.

\subsection{Multi-Agent Trust Region Learning With Constraints}
The first multi-agent trust region method that enjoys theoretically-justified monotonic improvement guarantee was first introduced by \citet{kuba2021trust}.  
Specifically, it is built on the multi-agent advantage decomposition in Lemma \ref{lemma:maad}, and the ``surrogate'' return given as follows. 
 
\begin{definition}
    \label{definition:localsurrogate}
     Let $\boldsymbol{\pi}$ be a joint policy, $\boldsymbol{\bar{\pi}}^{i_{1:h-1}}$ be some other joint policy of agents $i_{1:h-1}$, and $\hat{\pi}^{i_{h}}$ be a policy of agent $i_{h}$. Then we define
     \begin{align}
        & L^{i_{1:h}}_{\boldsymbol{\pi}}\left(  \boldsymbol{\bar{\pi}}^{i_{1:h-1}},  \hat{\pi}^{i_{h}} \right) 
         \triangleq 
         \nonumber \\ & \qquad 
         \E_{\rs \sim \rho_{\boldsymbol{\pi}}, \rva^{i_{1:h-1}}\sim\boldsymbol{\bar{\pi}}^{i_{1:h-1}}, \ra^{i_h}\sim\hat{\pi}^{i_{h}}}
         \left[  A_{\boldsymbol{\pi}}^{i_{h}}\left(\rs, 
         \rva^{i_{1:h-1}}, \ra^{i_{h}}\right) \right].
         \nonumber
     \end{align}
\end{definition}

With the above definition, we can see that 
Lemma \ref{lemma:maad} allows for decomposing the joint surrogate return  {$L_{\boldsymbol{\pi}}(\boldsymbol{\bar{\pi}}) \triangleq \E_{\rs\sim\rho_{\boldsymbol{\pi}}, \rva\sim\boldsymbol{\bar{\pi}}}[ A_{\boldsymbol{\pi}}(\rs, \rva) ]$} into a sum over surrogates of $L^{i_{1:h}}_{\boldsymbol{\pi}}(\boldsymbol{\bar{\pi}}^{i_{1:h-1}}, \bar{\pi}^{i_h})$, for $h=1, \dots, n$. This  can be used to  justify  that if agents, with a joint policy $\boldsymbol{\pi}$, update their policies by following a \textsl{sequential update scheme}, that is, if each agent  in the subset $i_{1:h}$ sequentially solves the following optimisation problem: 
\begin{align}
    & \bar{\pi}^{i_h} = \max_{\hat{\pi}^{i_h}} L^{i_{1:h}}_{\boldsymbol{\pi}}\left( \boldsymbol{\bar{\pi}}^{i_{1:h-1}}, \hat{\pi}^{i_{h}} \right) 
    - \nu D_{\text{KL}}^{\text{max}}\left( \pi^{i_h}, \hat{\pi}^{i_h} \right), \nonumber\\
   & \quad  \quad  \quad  \quad \text{where} \ \ \nu = \frac{4\gamma \max_{s, \va}|A_{\boldsymbol{\pi}}(s, \va)|}{(1-\gamma)^2}, \ \ \ \nonumber\\ & \text{and} \ \ \   D^{\text{max}}_{\text{KL}}(\pi^{i_h}, \hat{\pi}^{i_h}) \triangleq \max_{s}D_{\text{KL}}(\pi^{i_h}(\cdot|s), \hat{\pi}^{i_h}(\cdot|s)),\nonumber
\end{align}
then the resulting joint policy $\boldsymbol{\bar{\pi}}$ will surely improve the expected return, i.e.,  $J(\boldsymbol{\bar{\pi}}) \geq J(\boldsymbol{\pi})$ (see the proof in  \citet[Lemma 2]{kuba2021trust}). We know that due to the penalty term $D_{\text{KL}}^{\text{max}}( \pi^{i_h}, \hat{\pi}^{i_h} )$, the new policy $\bar{\pi}^{i_h}$ will stay close (\textsl{w.r.t} max-KL distance) to $\pi^{i_h}$. 

For the safety constraints, we can extend Definition \ref{definition:localsurrogate} to incorporate the ``surrogate" cost, thus allowing us to study the cost functions in addition to the return.    
\begin{definition}
    \label{definition:localcostsurrogate}
     Let $\boldsymbol{\pi}$ be a joint policy, and $\bar{\pi}^i$ be some other policy of agent $i$. Then, for any of its costs of index $j\in\{1, \dots, m^i\}$, we define
     \begin{align}
         L^{i}_{j, \boldsymbol{\pi}}\left( \bar{\pi}^i \right) 
         =
         \E_{\rs \sim \rho_{\boldsymbol{\pi}}, \ra^i\sim\bar{\pi}^i}
         \left[  A^i_{j, \boldsymbol{\pi}}\left(\rs, \ra^i \right) \right].
         \nonumber
     \end{align}
\end{definition}

By generalising the result about the surrogate return in Equation (\ref{definition:localsurrogate}), we can derive how the expected costs change when the agents update their policies. Specifically, we provide  the following lemma. 
\begin{restatable}{lemma}{surrogatecostlemma}
\label{lemma:surrogate-cost}
Let $\boldsymbol{\pi}$ and $\bar{\boldsymbol{\pi}}$ be joint policies. Let $i\in\mathcal{N}$ be an agent, and $j\in\{1, \dots, m^i\}$ be an index of one of its costs. The following inequality holds
\begin{align}
    & J^i_j(\bar{\boldsymbol{\pi}}) \leq J^i_j(\boldsymbol{\pi})
    + L^i_{j, \boldsymbol{\pi}}\big(\bar{\pi}^i\big) + \nu^i_j\sum\limits_{h=1}^{n}D^{\text{max}}_{\text{KL}}\big( \pi^h, \bar{\pi}^h\big),\nonumber\\
    & \quad  \quad  \quad \text{where} \ \nu^i_j = \frac{4\gamma\max_{s, a^i}|A^i_{j, \boldsymbol{\pi}}(s, a^i)|}{(1-\gamma)^2}.
\end{align}
\end{restatable}
See proof in Appendix \ref{appendix:preliminaries}. 
The above lemma suggests that, as long as the distances between the policies $\pi^h$ and $\bar{\pi}^h$, $\forall h\in\mathcal{N}$, are sufficiently small, then the change in the $j^{\text{th}}$ cost of agent $i$, i.e., $J^i_j(\bar{\boldsymbol{\pi}}) - J^i_j(\boldsymbol{\pi})$,   is controlled by the surrogate $L_{j, \boldsymbol{\pi}}^i(\bar{\pi}^i)$. Importantly, this surrogate is independent of other agents' new policies. 
Hence, when the changes in policies of all agents are sufficiently small, each agent can learn a better policy $\bar{\pi}^{i}$ by only considering its own surrogate return and surrogate costs. To summarise, we provide the pseudocode in {Algorithm \ref{algorithm:theoretical-safe-matrpo}} 
that guarantees both safety constraints satisfaction and monotonic  improvement. 

In {Algorithm \ref{algorithm:theoretical-safe-matrpo}}, in addition to sequentially maximising agents'   surrogate returns, the agents must assure that their  surrogate costs stay below the corresponding safety thresholds. Meanwhile, they have to constrain their policy search to small local neighbourhoods (\textsl{w.r.t} max-KL distance). As such,  Algorithm \ref{algorithm:theoretical-safe-matrpo} demonstrates two desirable properties: reward performance improvement and satisfaction of safety constraints, which we justify in the  following theorem. 

\begin{restatable}{theorem}{thsafematrpo}
\label{theorem:theoretical-safe-matrpo}
If a sequence of joint policies $(\boldsymbol{\pi}_k)_{k=0}^{\infty}$ is obtained from Algorithm \ref{algorithm:theoretical-safe-matrpo}, then it has the monotonic improvement property, $J(\boldsymbol{\pi}_{k+1})\geq J(\boldsymbol{\pi}_k)$, as well as it satisfies the safety constraints, $J^i_j(\boldsymbol{\pi}_k)\leq c^i_j$, for all $k\in\mathbb{N}, i\in\mathcal{N}$, and $j\in\{1, \dots, m^i\}$.
\end{restatable} 

See proof in Appendix \ref{appendix:results-safe-matrpo}. 
The above theorem assures that agents that follow Algorithm \ref{algorithm:theoretical-safe-matrpo} will only explore safe policies; meanwhile, every new policy will be guaranteed to result in performance improvement. These two properties hold under the conditions that only restrictive policy updates are made; this is due to the KL-penalty term in every agent's objective (i.e., $ \nu D_{\text{KL}}^{\text{max}}(\pi^{i_{h}}_{k}, \pi^{i_{h}})$), as well as the constraints on cost surrogates  (i.e., the conditions in $\overline{\Pi}^{i_h}$).  
In practice, it can be intractable to evaluate $D_{\text{KL}}\big(\pi^{i_h}_k(\cdot|s), \pi^{i_h}(\cdot|s)\big)$ at every state in order to compute $D_{\text{KL}}^{\text{max}}(\pi^{i_h}_k, \pi^{i_h})$. 
In the following subsections, we describe how we can approximate  Algorithm \ref{algorithm:theoretical-safe-matrpo} in the case of parameterised policies, similar to TRPO/PPO implementations \citep{schulman2015trust, schulman2017proximal}. 

\subsection{MACPO: Multi-Agent Constrained Policy Optimisation} 
Here we focus on the practical settings where large state and action spaces  prevent agents from
designating policies $\pi^i(\cdot|s)$ for each state  separately.  To handle this, we parameterise each agent's $\pi^{i}_{\theta^{i}}$ by a neural network $\theta^{i}$. Correspondingly, the joint policies $\boldsymbol{\pi}_{\vtheta}$ are  parametrised by $\vtheta = (\theta^1, \dots, \theta^n)$. 

Let's recall that at every iteration of Algorithm \ref{algorithm:theoretical-safe-matrpo}, every agent $i_h$ maximises its surrogate return  with a KL-penalty, subject to surrogate cost constraint. Yet, direct computation of the max-KL constraint is intractable in practical settings, as it would require computation of KL-divergence at every single state.  
Instead, one can relax it by adopting a form of expected KL-constraint $\overline{D}_{\text{KL}}(\pi^{i_h}_k, \pi^{i_h})\leq \delta$ where $\overline{D}_{\text{KL}}(\pi^{i_h}_k, \pi^{i_h}) \triangleq \E_{\rs\sim\rho_{\boldsymbol{\pi}_k}}\big[ D_{\text{KL}}(\pi^{i_h}_k(\cdot|\rs), \pi^{i_h}(\cdot|\rs)) \big]$. Such an expectation can be approximated by stochastic sampling. As a result, the optimisation problem solved by agent $i_h$ can be  written as 
\begin{align}
\label{eq:goal-function-with-parameterized-policy}
& \theta^{i_h}_{k+1} = \argmax_{\theta^{i_h}} \nonumber \\& \quad \E_{\rs\sim\rho_{\boldsymbol{\pi}_{\vtheta_k} }, \rva^{i_{1:h-1}}\sim\boldsymbol{\pi}^{i_{1:h-1}}_{\vtheta^{i_{1:h-1}}_{k+1}}, \ra^{i_h}\sim\pi^{i_h}_{\theta^{i_h}} }\left[ 
A^{i_h}_{\boldsymbol{\pi}_{\vtheta_k}}\left(\rs,
\rva^{i_{1:h-1}}, \ra^{i_h} \right)\right]
\nonumber\\
&   \text { s.t. }   \ 
J^{i_h}_j\left( \boldsymbol{\pi}_{\vtheta_k}\right)+ 
\E_{\rs\sim\rho_{\boldsymbol{\pi}_{\vtheta_k}}, \ra^{i_h}\sim\pi^{i_h}_{\theta^{i_h}_{k}} }\left[ 
A^{i_h}_{j, \boldsymbol{\pi}_{\vtheta_k}}\left(\rs, \ra^{i_h} \right)\right]
\leq c^{i_h}_j, \ \nonumber\\
& \qquad  \forall j\in\{1, \dots, m^{i_h}\}, \ \  \text{and} \ \  \ \ \overline{D}_{\text{KL}}\big(\pi^{i_h}_k, \pi^{i_h}\big) \leq \delta.
\end{align}


We can further approximate Equation (\ref{eq:goal-function-with-parameterized-policy})  by  Taylor expansion of the optimisation objective and cost constraints up to the first order, and the KL-divergence up to the second order. 
Consequently, the optimisation problem can be written as 
\begin{align}
\label{eq:parameterized-policy-equation}
&  \theta^{i_h}_{k+1}=\arg \max _{\theta^{i_h}}  \big(\vg^{i_h}\big)^{T}\left(\theta^{i_h}-\theta^{i_h}_{k}\right) \nonumber\\
&   \text { s.t. } \ \  d_{j}^{i_h}+ \big(\vb_{j}^{i_h}\big)^T\left(\theta^{i_h}-\theta^{i_h}_{k}\right)  \leq 0,  \quad j=1, \ldots, m \nonumber\\
&   \text{and} \  \  \ \frac{1}{2}\left(\theta^{i_h}-\theta^{i_h}_{k}\right)^{T} \mH^{i_h}\left(\theta^{i_h}-\theta^{i_h}_{k}\right) \leq \delta,
\end{align}
where $\vg^{i_h}$ is the gradient of the  objective  of agent $i_h$ in Equation (\ref{eq:goal-function-with-parameterized-policy}), $d_{j}^{i} = J_{j}^{i}(\boldsymbol{\pi}_{\vtheta_k}) - c_{j}^{i}$, and $\mH^{i_h}=\nabla^2_{\theta^{i_h}}\overline{D}_{\text{KL}}(\pi^{i_h}_{\theta^{i_h}_k}, \pi^{i_h})\big|_{\theta^{i_h} = \theta^{i_h}_k }$ is the Hessian of the average KL divergence of agent $i_h$, and $\vb_{j}^{i_h}$ is the gradient of agent of the $j^{\text{th}}$ constraint of agent $i_h$.

\begin{algorithm}[t!]
\caption{Safe Multi-Agent Policy Iteration with Monotonic Improvement Property}
\label{algorithm:theoretical-safe-matrpo}
\begin{algorithmic}[1]
\STATE Initialise a  joint policy {\small $\boldsymbol{\pi}_{0} = (\pi^{1}_{0}, \dots, \pi^{n}_{0})$}. {\color{gray}{// initial unsafe policies are allowed, see discussion in experiments.}}
\FOR{$k=0, 1, \dots$}
    \STATE Compute the advantage functions {\small $A_{\boldsymbol{\pi}_{k}}(s, \va)$} and {\small $A^i_{j, \boldsymbol{\pi}_k}(s, a^i)$}, for all state-(joint)action pairs {\small $(s, \va)$},  agents {\small $i$}, and constraints {\small $j\in\{1, \dots, m^i\}$}.
    \STATE Compute {\small $\forall i\in\mathcal{N}, j=1, \dots, m^i$}\\ {\small $\nu= \frac{4\gamma \max_{s, \va}|A_{\boldsymbol{\pi}_k}(s, \va)|}{(1-\gamma)^2}, \nu^i_j = \frac{4\gamma \max_{s, a^i}|A^i_{j, \boldsymbol{\pi}_k}(s, a^i)| }{(1-\gamma)^2 }$}.
    \STATE Draw a permutaion $i_{1:n}$ of agents at random.
    \FOR{$h=1:n$}
        \STATE  \textsl{{\color{gray}{// see Appendix \ref{appendix:results-safe-matrpo} for the setup of $\delta^{i_h}$.}}} \\  Compute the radius of the KL-constraint $\delta^{i_h}$   
        \STATE Make an update {\small $\pi^{i_{h}}_{k+1} =$}\\ {\small$ \argmax_{\pi^{i_{h}}\in \overline{\Pi}^{i_h}}\left[ 
        L^{i_{1:h}}_{\boldsymbol{\pi}_k}\left(\pi^{i_{1:h-1}}_{k+1}, \pi^{i_{h}}\right) - \nu D_{\text{KL}}^{\text{max}}\Big(\pi^{i_{h}}_{k}, \pi^{i_{h}}\Big)\right]$},\\
        \text{where {\small $\overline{\Pi}^{i_h}$} is a subset of safe policies of agent $i_h$,} 
        \begin{smalleralign}
           &\overline{\Pi}^{i_h} = \Big\{\pi^{i_h}\in\Pi^{i_h} \ \big| \ D_{\text{KL}}^{\text{max}}(\pi^{i_h}_k, \pi^{i_h}) \leq
            \delta^{i_h}, \ \text{and} \nonumber\\
            &  \ \ \ \ \ \ \ \ \ \ \ \ \  J^{i_h}_j (\boldsymbol{\pi}_k) + L^{i_h}_{j, \boldsymbol{\pi}_k}(\pi^{i_h}) + \nu^{i_h}_j D^{\text{max}}_{\text{KL}}(\pi^{i_h}_k, \pi^{i_h}) \nonumber \\
              & \ \ \ \ \ \ \ \ \ \ \ \ 
           \leq c^{i_h}_j - \sum_{l=1}^{h-1}\nu^{i_l}_j D^{\text{max}}_{\text{KL}}(\pi^{i_l}_k, \pi^{i_l}) , \forall j=1, \dots, m^{i_h} \Big\}.\nonumber
        \end{smalleralign}
    \ENDFOR
\ENDFOR
\end{algorithmic}
\end{algorithm}

Similar to \citet{chow2017risk} and \citet{achiam2017constrained}, one can take a primal-dual optimisation approach to solve 
the linear quadratic optimisation in Equation  (\ref{eq:parameterized-policy-equation}). Specifically, the dual form can be written as:  
\begin{align}
    \label{eq:dual-parameterized-policy-equation}
    & \max_{\lambda^{i_h} \geq 0, \textbf{v}^{i_h} \succeq 0} 
    \frac{-1}{2\lambda^{i_h}}\left[
    (\vg^{i_h})^T (\mH^{i_h})^{-1}\vg^{i_h} - 2(\vr^{i_h})^T\textbf{v}^{i_h} +
    \notag\right. 
\\ &
\phantom{=\;\;}
                \left.  
   \qquad \qquad \qquad \qquad
    (\textbf{v}^{i_h})^T \mS^{i_h}
    \right] +  (\textbf{v}^{i_h})^T\vc^{i_h} - \frac{\lambda^{i_h} \delta}{2}, \nonumber\\
   & \quad \text{where} \ \  \vr^{i_h} \triangleq (\vg^{i_h})^T (\mH^{i_h})^{-1} \mB^{i_h}, \mB^{i_h} = \left[ \vb^{i_h}_1, \dots, \vb^{i_h}_m  \right] \ \nonumber\\ & \quad \text{and}
    \  \mS^{i_h} \triangleq (\mB^{i_h})^T (\mH^{i_h})^{-1} \mB^{i_h}.
\end{align}

Given the solution to the dual form in Equation  (\ref{eq:dual-parameterized-policy-equation}), i.e., $\lambda^{i_h}_*$ and $\textbf{v}^{i_h}_*$, the solution to the primal problem in Equation (\ref{eq:parameterized-policy-equation}) can thus be  written by
\begin{align}
    \theta^{i_h}_* = \theta^{i_h}_k + \frac{1}{\lambda^{i_h}_*}\big(\mH^{i_h}\big)^{-1}\big(\vg^{i_h} - \mB^{i_h}\textbf{v}^{i_h}_*\big).\nonumber
\end{align} 

In practice, we use backtracking line search starting at $1/\lambda^{i_h}_*$ to choose the step size of the above update.
Furthermore, we note that the optimisation step in Equation
(\ref{eq:parameterized-policy-equation}) is an approximation to the original problem from Equation (\ref{eq:goal-function-with-parameterized-policy}); therefore, it is possible that  an infeasible policy $\pi_{\theta^{i_h}_{k+1}}$ will be generated. Fortunately, as the policy optimisation takes place in the trust region of $\pi^{i_h}_{\theta^{i_h}_k}$, the size of update is small, and a feasible policy can be easily recovered. In particular, for problems with one safety constraint, i.e., $m^{i_h} = 1$, one can recover a feasible policy by applying a TRPO step  on the cost  surrogate, written as 
\begin{align}
    \label{eq:no-proof-recover-policy-from-unfeasible-point}
    \theta^{i_h}_{k+1} = \theta^{i_h}_{k}-\alpha^j\sqrt{\frac{2 \delta}{{\vb^{i_h}}^{T} (\mH^{i_h})^{-1} \vb^{i_h}}} \big(\mH^{i_h}\big)^{-1} \vb^{i_h}
\end{align}
where $\alpha^j$ is adjusted through backtracking line search. To put it together, we refer to this algorithm as \textsl{MACPO}, and provide its pseudocode in Appendix \ref{appendix:mactrpo}.

\subsection{MAPPO-Lagrangian}
In addition to MACPO, one can use Lagrangian multipliers in place of optimisation with linear and quadratic constraints to solve  Equation (\ref{eq:goal-function-with-parameterized-policy}). 
The Lagrangian method is simple to implement, and it does not require computations of the Hessian $\mH^{i_h}$ whose size grows quadratically with the dimension of the parameter vector $\theta^{i_h}$.   

Let us briefly recall the optimisation procedure with a Lagrangian multiplier. Suppose that our goal is to maximise a bounded real-valued function $f(x)$ under a constraint $g(x)$;  $\max_{x} f(x), \text{s.t.} \ g(x)\leq 0$. Consider a scalar variable $\lambda$ and an alternative optimisation problem, given by 
\begin{align}
    \label{eq:general-lagrangian}
    \max_{x}\min_{\lambda\geq 0} f(x) - \lambda g(x).
\end{align}
Suppose that $x_+$ satisfies $g(x_+) > 0$. This immediately implies that $-\lambda g(x_+) \to -\infty$, as $\lambda\to +\infty$, and so Equation (\ref{eq:general-lagrangian}) equals $-\infty$ for $x=x_+$. On the other hand, if $x_-$ satisfies $g(x_-)\leq 0$, we have that $-\lambda g(x_-) \geq 0$, with equality only for $\lambda = 0$. In that case, the optimisation objective's value equals $f(x_-)>-\infty$. Hence, the only candidate solutions to the problem are those $x$ that satisfy the constraint $g(x)\leq 0$, and the  objective  matches with $f(x)$. 

We can employ the above trick to the constrained optimisation problem from Equation (\ref{eq:goal-function-with-parameterized-policy}) by subsuming the cost constraints into the optimisation objective with Lagrangian multipliers. As such, agent $i_h$ computes $\bar{\lambda}^{i_h}_{1:m^{i_h}}$ and $\theta^{i_h}_{k+1}$ to solve the following min-max optimisation problem
{\small
\begin{align}
    \label{eq:simplified-lagrangian}
    &\min_{\lambda^{i_h}_{1:m^{i_h}} \geq 0}\max_{\theta^{i_h}} \nonumber \\& \quad  \Bigg[
    \E_{\rs\sim\rho_{\boldsymbol{\pi}_{\vtheta_k}}, \rva^{i_{1:h-1}}\sim\boldsymbol{\pi}^{i_{1:h-1}}_{\vtheta^{i_{1:h-1}}_{k+1}}, \ra^{i_h}\sim\pi^{i_h}_{\theta^{i_h}}}\left[ A_{\boldsymbol{\pi}_{\vtheta_k}}^{i_h}\left( \rs, 
    \rva^{i_{1:h-1}}, 
               \ra^{i_h} \right) \right]
    \nonumber \\ & \quad  
    - \sum\limits_{u=1}^{m^{i_h}}\lambda^{i_h}_u\left(
    \E_{\rs\sim\rho_{\boldsymbol{\pi}_{\vtheta_k}},  \ra^{i_h}\sim\pi^{i_h}_{\theta^{i_h}}}\left[ A^{i_h}_{u, \boldsymbol{\pi}_{\vtheta_k}}\left( \rs, \ra^{i_h} \right) \right]
    +d^{i_h}_u\right)\Bigg], \nonumber\\
    &   \quad  \quad  \quad  \quad  \quad  \quad  \quad  \text{s.t.} \ \overline{D}_{\text{KL}}\Big( \pi^{i_h}_{\theta^{i_h}_k}, \pi^{i_h}_{\theta^{i_h}} \Big)\leq \delta.
\end{align}
}
Although the objective from Equation (\ref{eq:simplified-lagrangian}) is affine in the Lagrangian multipliers $\lambda^{i_h}_u$ ($u= 1, \dots, m^{i_h}$), which enables  gradient-based optimisation solutions,   computing the KL-divergence constraint still complicates the overall process. 
To handle this, one can further simplify it by adopting the \textsl{PPO-clip} objective \citep{schulman2017proximal}, which enables replacing the KL-divergence constraint with the \textsl{clip} operator, and update the policy parameter with first-order methods. We do so by defining
\begin{align}
    &A^{i_h, (\lambda)}_{\boldsymbol{\pi}_{\vtheta_k}}\left(s, \va^{i_{1:h-1}}, a^{i_h}\right) \triangleq  \nonumber \\ 
    &\quad A_{\boldsymbol{\pi}_{\vtheta_k}}^{i_h}\left( s, \va^{i_{1:h-1}}, a^{i_h} \right)
     - 
    \sum\limits_{u=1}^{m^{i_h}} \lambda^{i_h}_u \left( A^{i_h}_{u, \boldsymbol{\pi}_{\vtheta_k}}\left( s, a^{i_h} \right) 
    +d^{i_h}_u \right), \nonumber
\end{align}
and rewriting the Equation (\ref{eq:simplified-lagrangian}) as 
\begin{align}
     \label{eq:very-simplified-lagrangian}
    & \min_{\lambda^{i_h}_{1:m^{i_h}} \geq 0}\max_{\theta^{i_h}}
    \E_{\rs\sim\rho_{\boldsymbol{\pi}_{\vtheta_k}}, \rva^{i_{1:h-1}}\sim\boldsymbol{\pi}^{i_{1:h-1}}_{\vtheta^{i_{1:h-1}}_{k+1}}, \ra^{i_h}\sim\pi^{i_h}_{\theta^{i_h}}}
    \nonumber \\& \qquad \qquad \qquad \qquad \qquad \qquad 
    \left[ A_{\boldsymbol{\pi}_{\vtheta_k}}^{i_h, (\lambda)}\left( \rs,
   \rva^{i_{1:h-1}}, \ra^{i_h} \right) \right],\nonumber\\ 
    &  \quad  \quad  \quad   \quad  \quad  \quad   \text{s.t.} \ \overline{D}_{\text{KL}}\big( \pi^{i_h}_{\theta^{i_h}_k}, \pi^{i_h}_{\theta^{i_h}} \big)\leq \delta.
\end{align}

The objective in Equation (\ref{eq:very-simplified-lagrangian}) takes a form of an expectation with quadratic constraint on the policy. Up to the error of approximation of KL-constraint with the \textsl{clip} operator, it can be equivalently transformed into an optimisation of a clipping objective. Finally, the objective takes the  form of  
{ \begin{align}
     \label{eq:simplified-lagrangian-clip}
    &\E_{\rs\sim\rho_{\boldsymbol{\pi}_{\vtheta_k}}, \rva^{i_{1:h-1}}\sim\boldsymbol{\pi}^{i_{1:h-1}}_{\vtheta^{i_{1:h-1}}_{k+1}}, \ra^{i_h}\sim\pi^{i_h}_{\theta^{i_h}_k}}\Bigg[ \notag \\ & \ \ \  \text{min}\Bigg( \frac{\pi^{i_h}_{\theta^{i_h}}(\ra^{i_h}|\rs)}{\pi^{i_h}_{\theta^{i_h}_k}(\ra^{i_h}|\rs)} A_{\boldsymbol{\pi}_{\vtheta_k}}^{i_h, (\lambda)}\left( \rs, 
                  \rva^{i_{1:h-1}}, \ra^{i_h} \right), \notag \\ & \qquad \ \ \   \  \text{clip}\bigg( \frac{\pi^{i_h}_{\theta^{i_h}}(\ra^{i_h}|\rs)}{\pi^{i_h}_{\theta^{i_h}_k}(\ra^{i_h}|\rs)}, 1\pm\epsilon\bigg)A_{\boldsymbol{\pi}_{\vtheta_k}}^{i_h, (\lambda)}\left( \rs, \rva^{i_{1:h-1}}, \ra^{i_h} \right)
    \Bigg)\Bigg].
\end{align}} 

The above clip operator replaces the policy ratio with $1-\epsilon$, or $1+\epsilon$, depending on whether its value is below or above the threshold interval. As such, agent $i_h$ can learn within its trust region by updating $\theta^{i_h}$ to maximise Equation (\ref{eq:simplified-lagrangian-clip}), while the Lagrangian multipliers are updated towards the direction opposite to their gradients of Equation (\ref{eq:simplified-lagrangian}), which can be computed analytically. We refer to this algorithm as \textsl{MAPPO-Lagrangian}, and give a detailed pseudocode of it in Appendix \ref{appendix:mappo-lagrangian} due to space limit.

\begin{figure*}[t!]
 \centering
 \vspace{0pt}
\subcaptionbox{}
{
\includegraphics[width=0.22\linewidth]{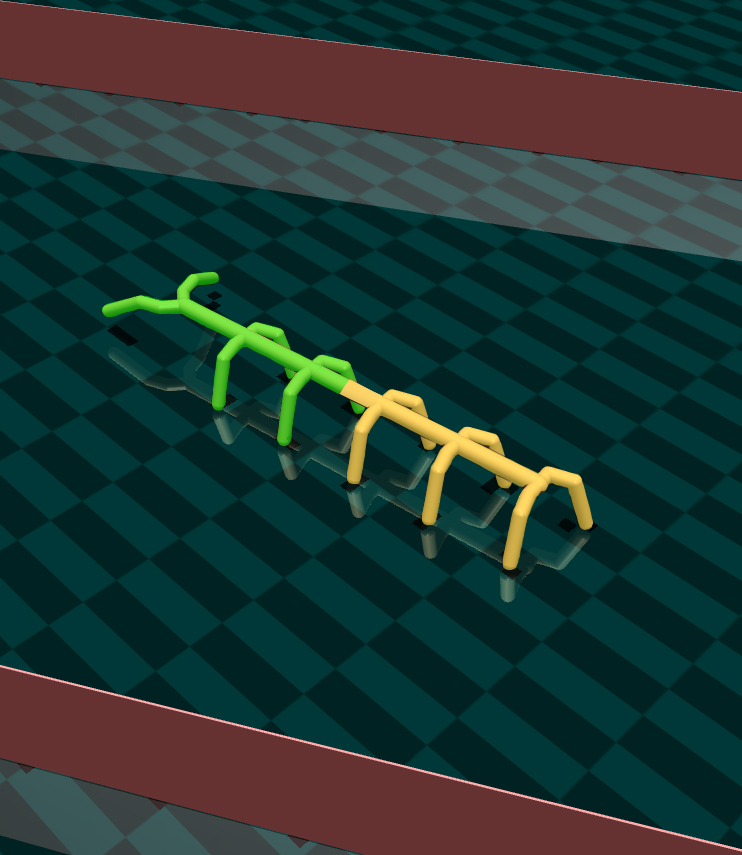}
}\quad \quad \quad
\subcaptionbox{}
{
\includegraphics[width=0.24\linewidth]{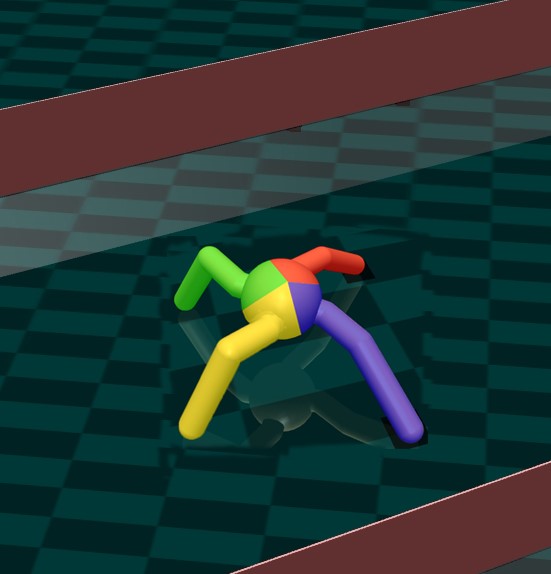}
}\quad \quad \quad
\subcaptionbox{}
{
\includegraphics[width=0.28\linewidth]{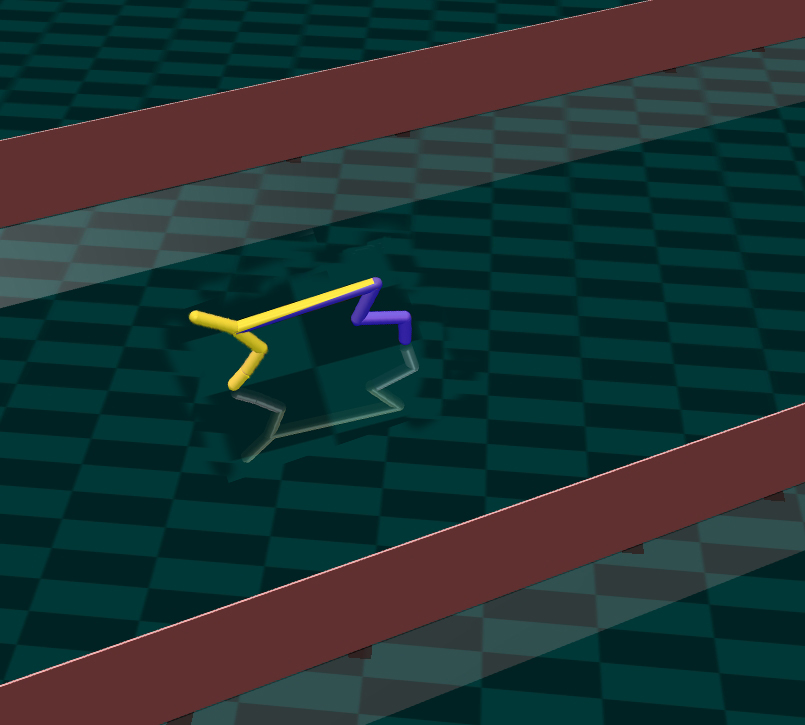}
}
 	\caption{\normalsize Example tasks in SMAMuJoCo Environment. (a): Safe 2x3-ManyAgent Ant, (b): Safe 4x2-Ant, (c): Safe 2x3-HalfCheetah. Body parts of different colours are controlled by different agents. Agents jointly learn to manipulate the robot, while avoiding crashing into unsafe red areas. 
 	} 
 	\label{fig:Safety-Ant-Environment}
 \end{figure*}

 \begin{figure*}[t!]
 \centering
 \vspace{0pt}
\subcaptionbox{ 2x3-Agent (1st column), 3x2-Agent (2nd column), 6x1-Agent (3rd column), 2x4d-Agent (4th column)}[1.\linewidth]
{
\includegraphics[width=0.24\linewidth]{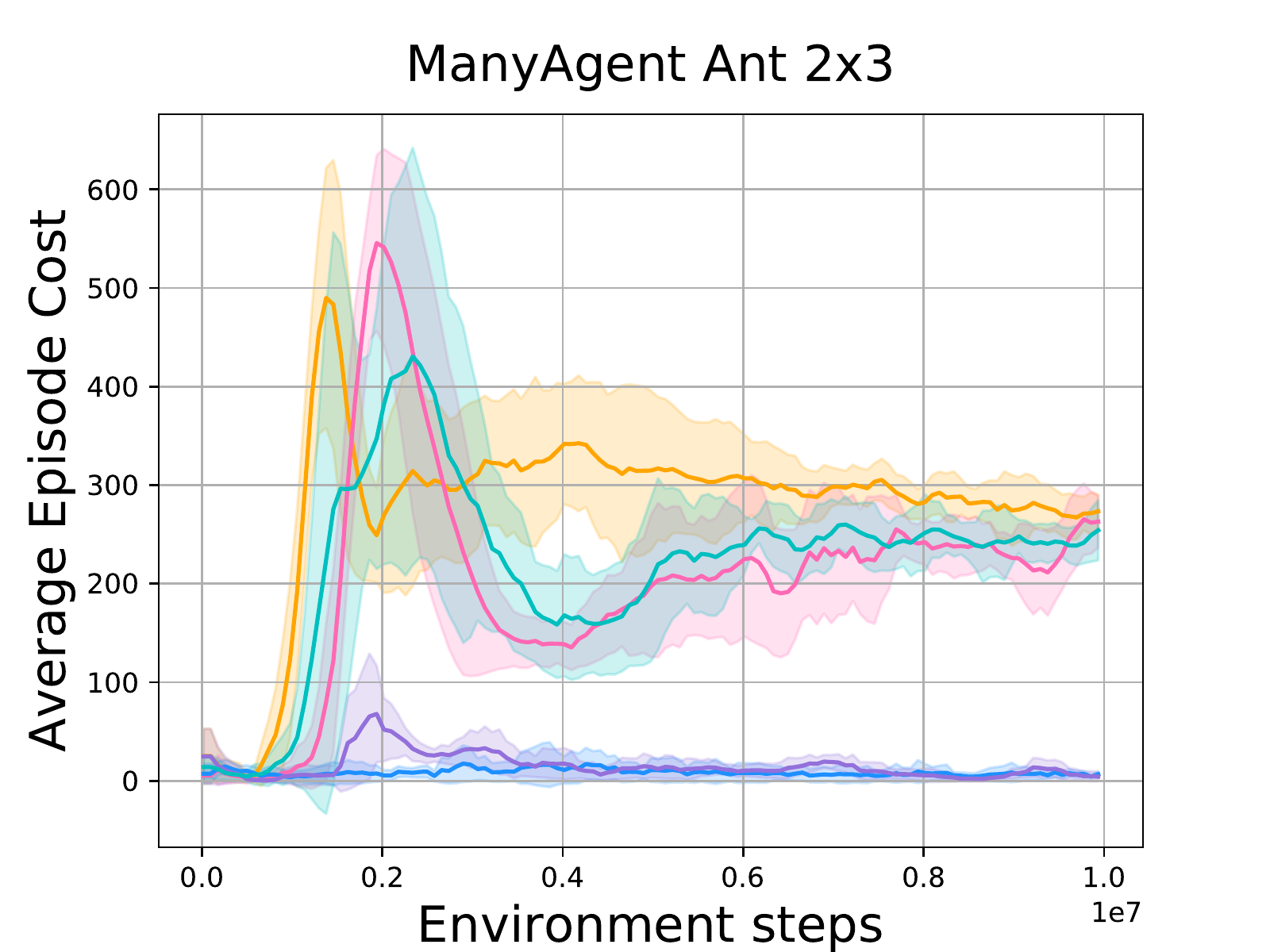}
\includegraphics[width=0.24\linewidth]{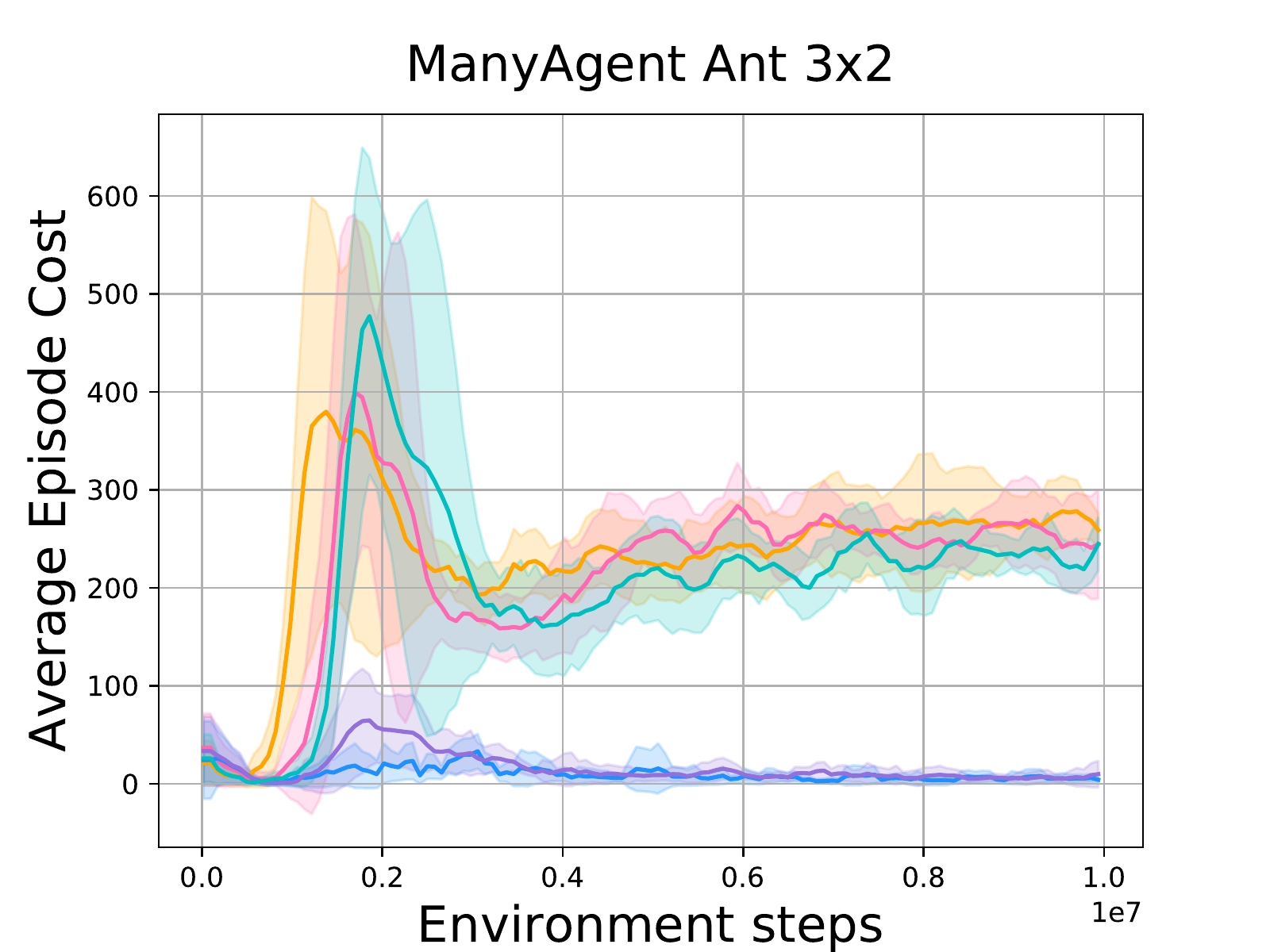}
\includegraphics[width=0.24\linewidth]{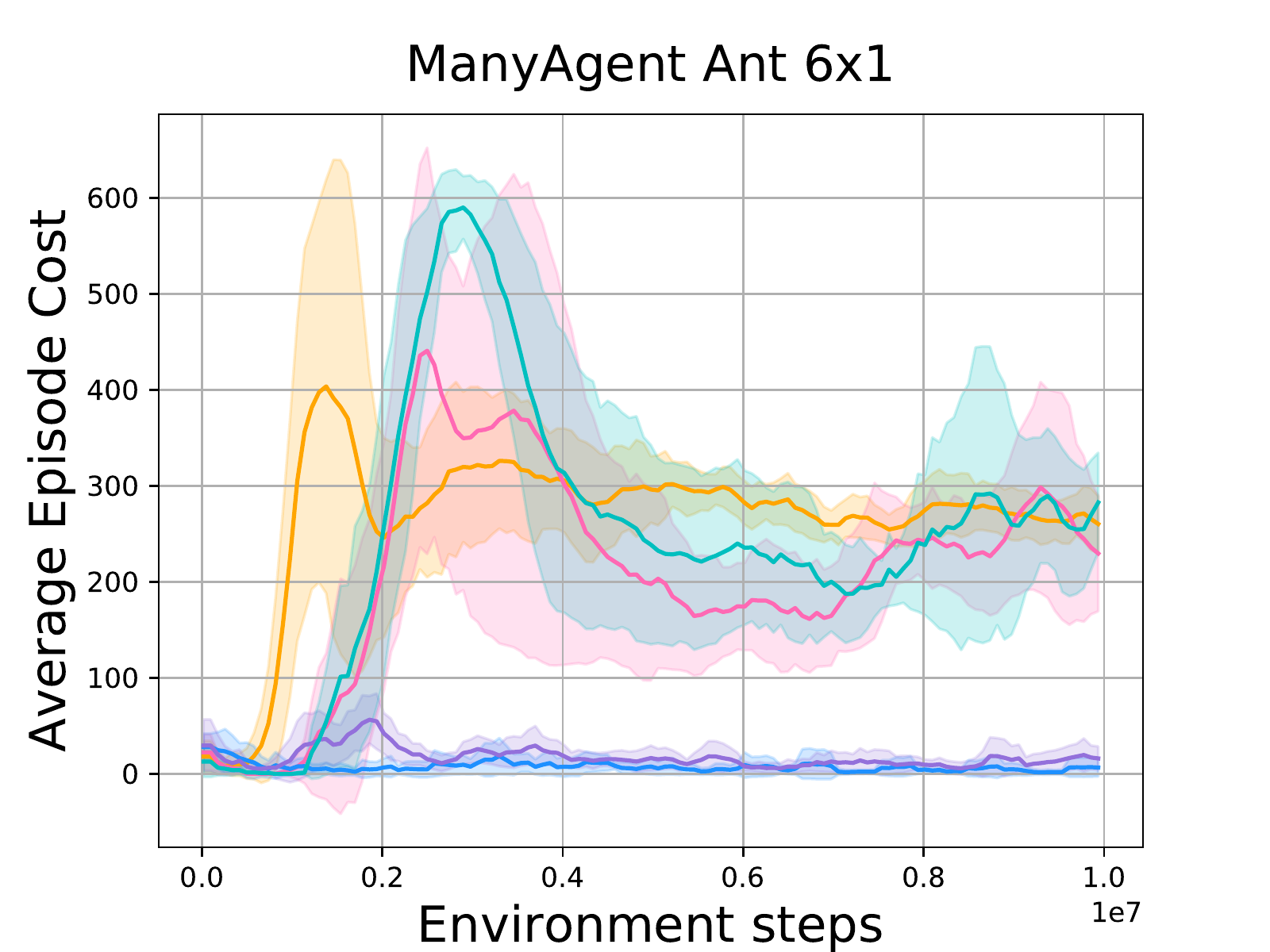}
\includegraphics[width=0.24\linewidth]{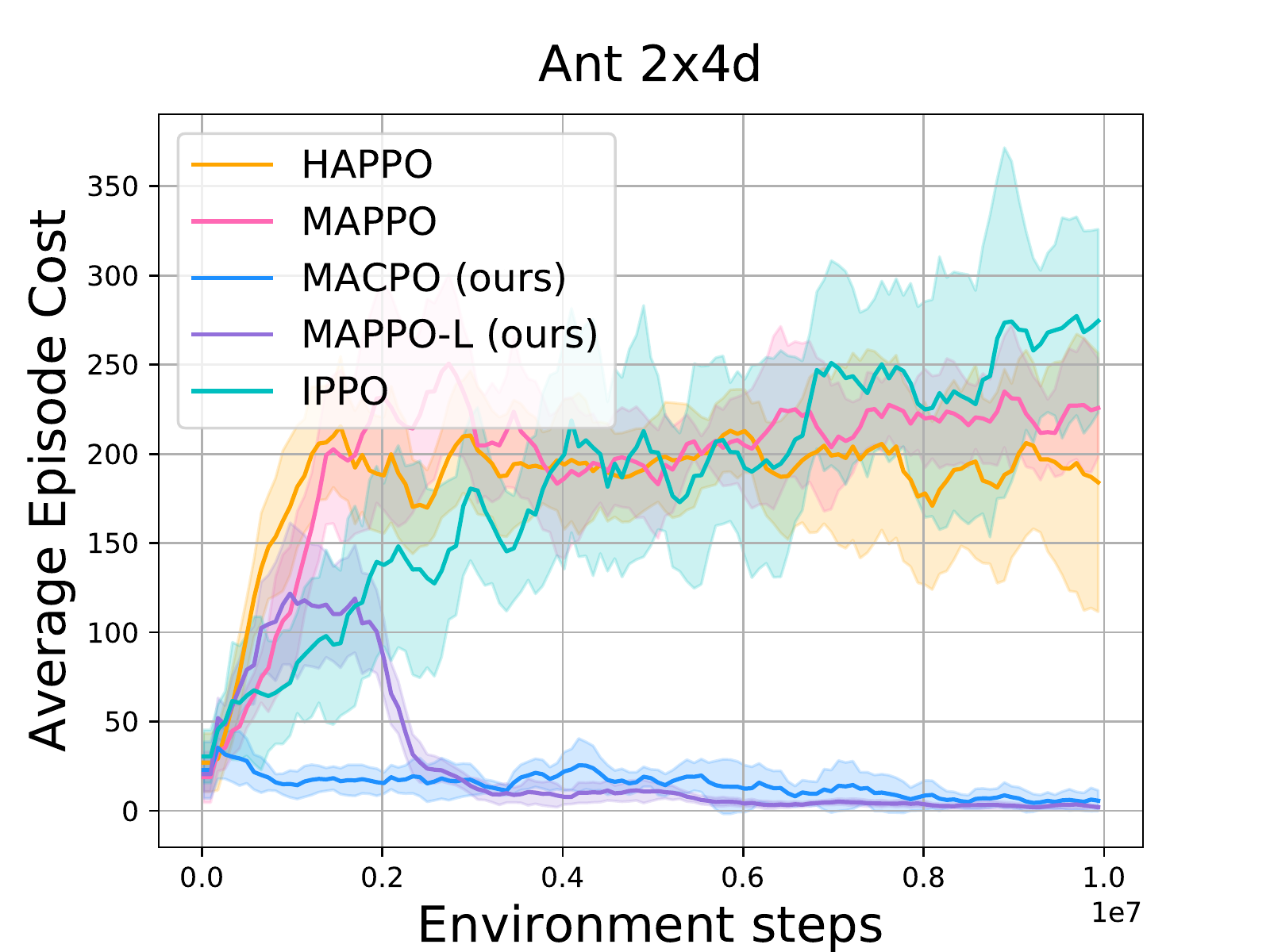}

\includegraphics[width=0.24\linewidth]{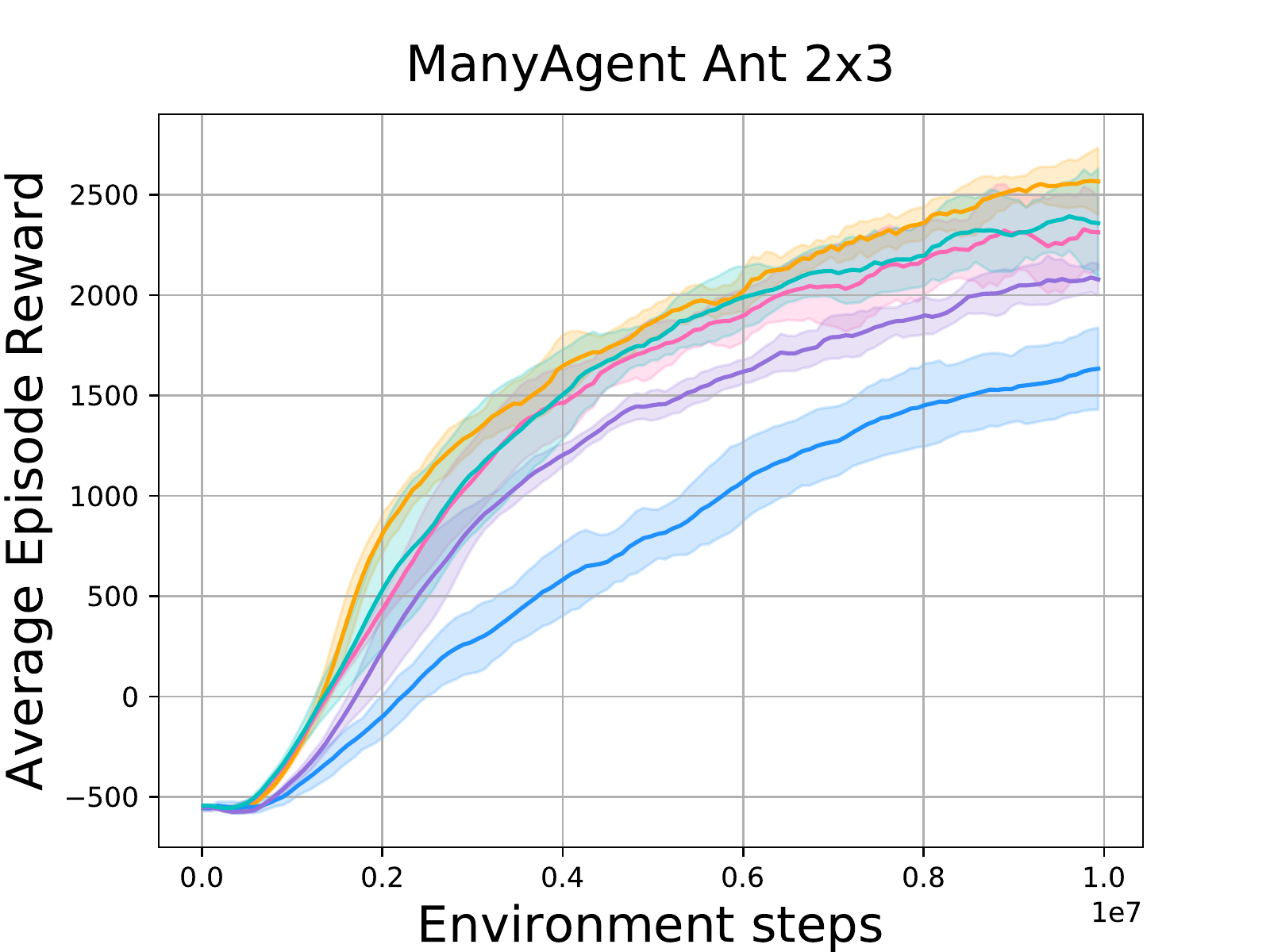}
\includegraphics[width=0.24\linewidth]{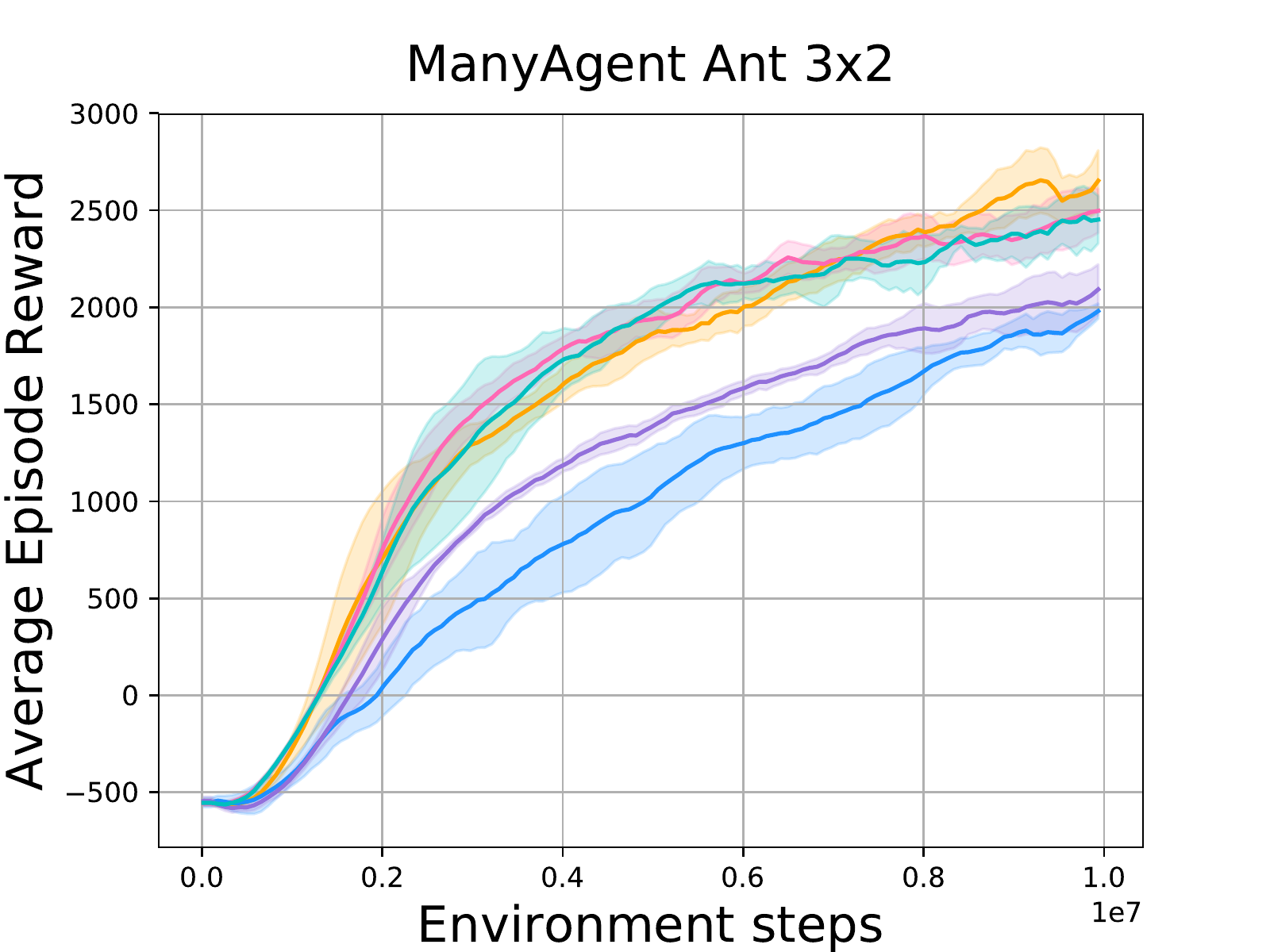}
\includegraphics[width=0.24\linewidth]{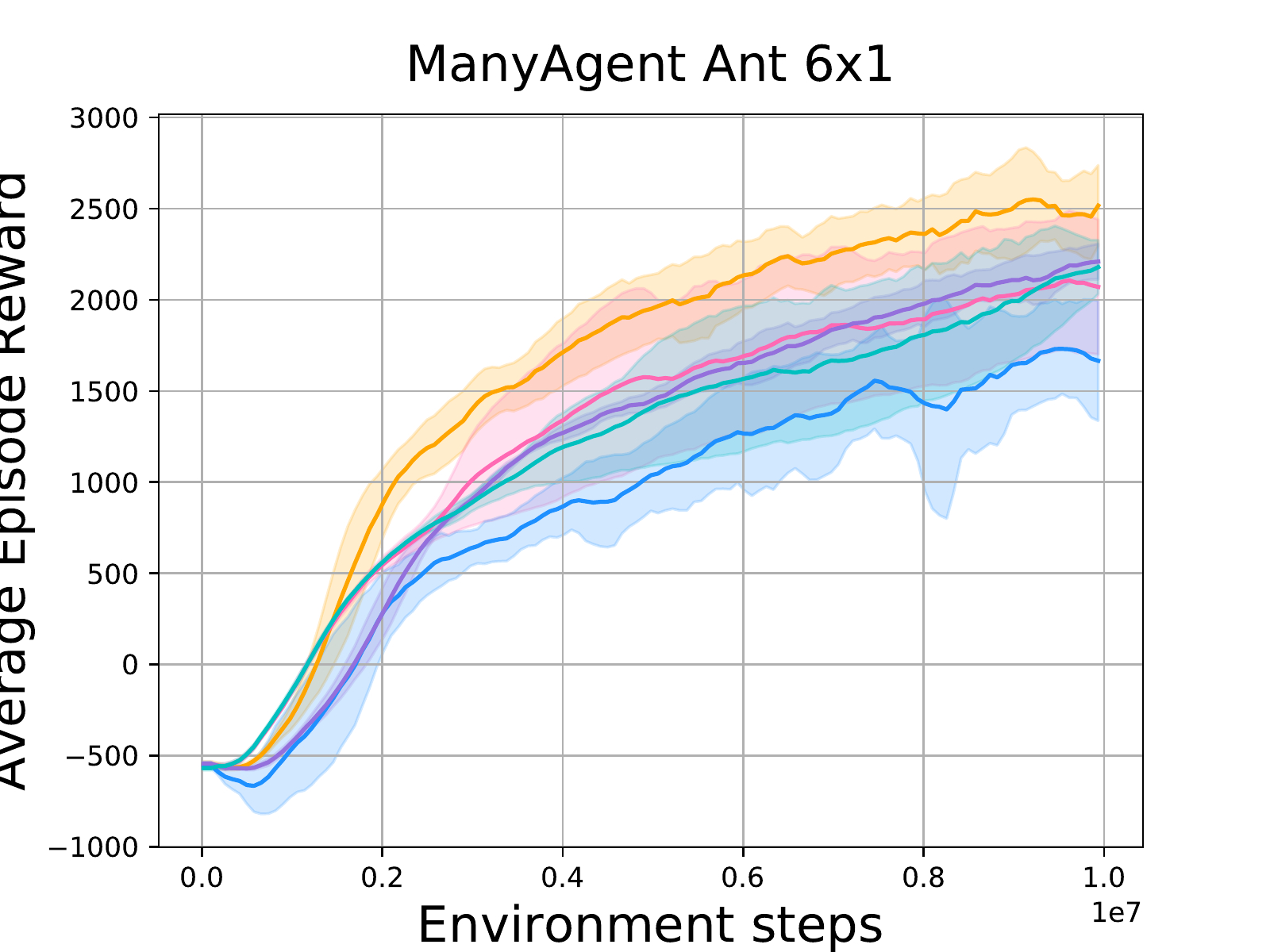}
\includegraphics[width=0.24\linewidth]{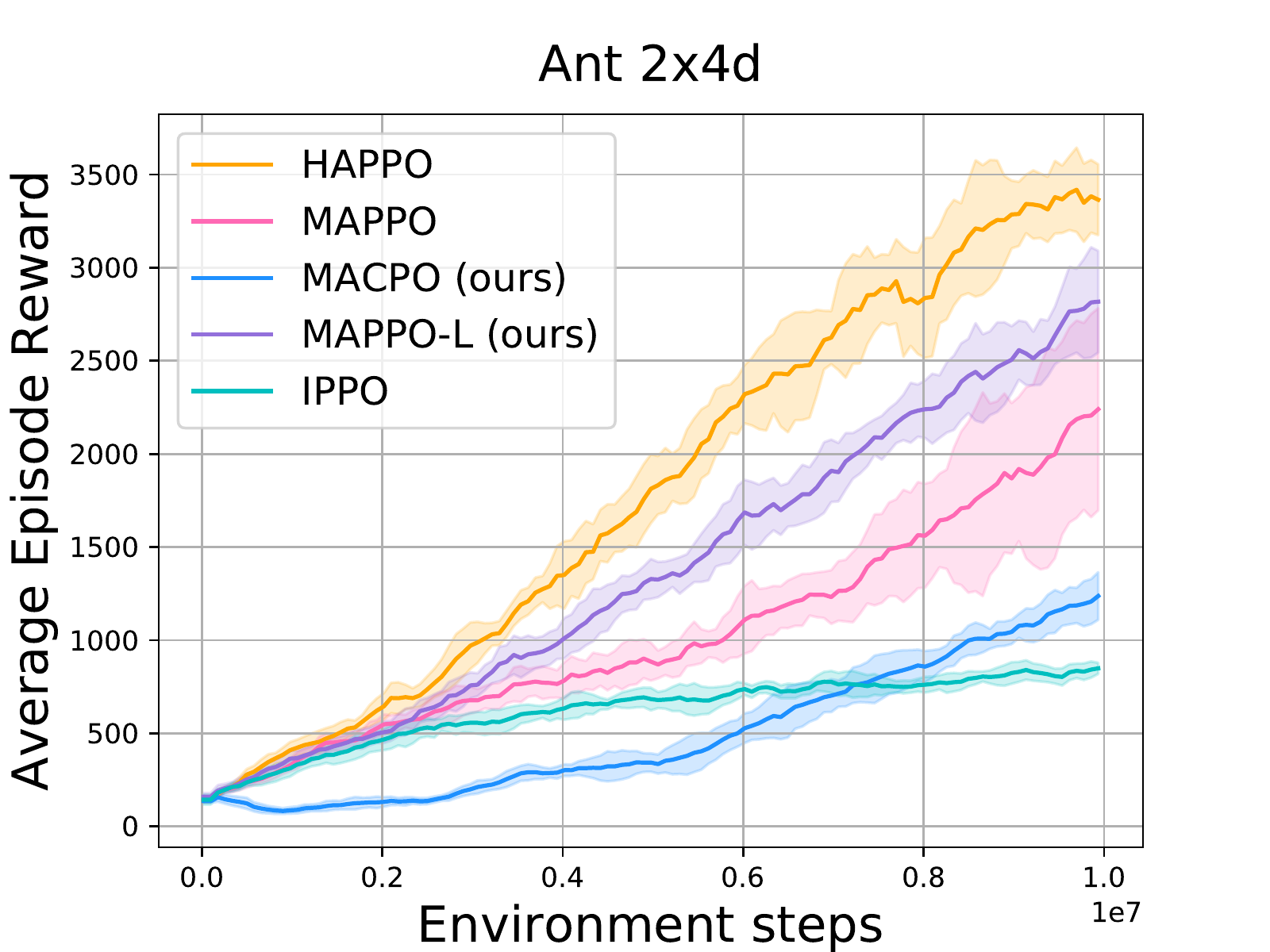}
}



\subcaptionbox{ 4x2-Agent (1st column), 2x3-Agent (2nd column), 3x2-Agent (3rd column), 6x1-Agent (4th column)}[1.\linewidth]
{
\includegraphics[width=0.24\linewidth]{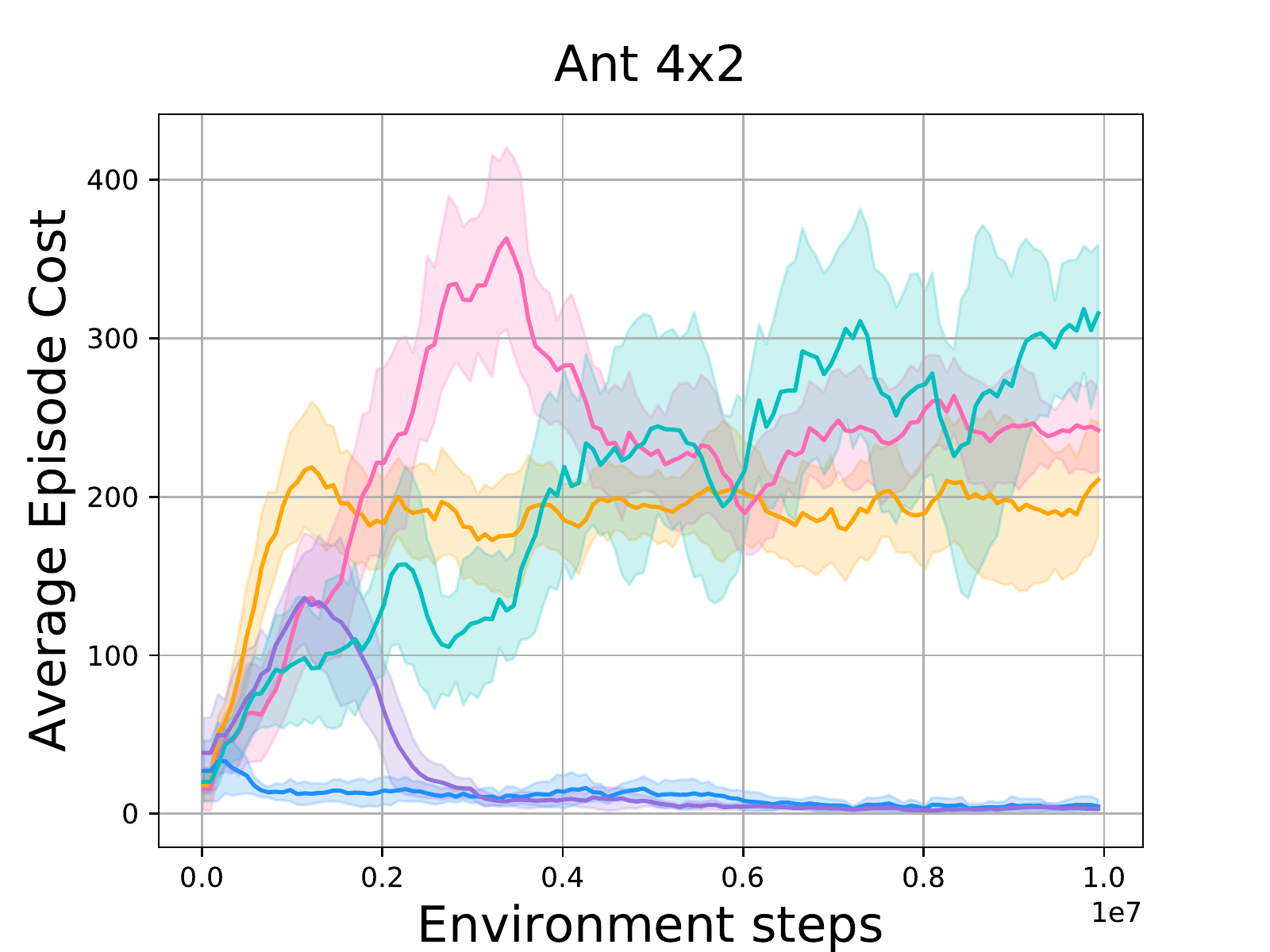}
\includegraphics[width=0.24\linewidth]{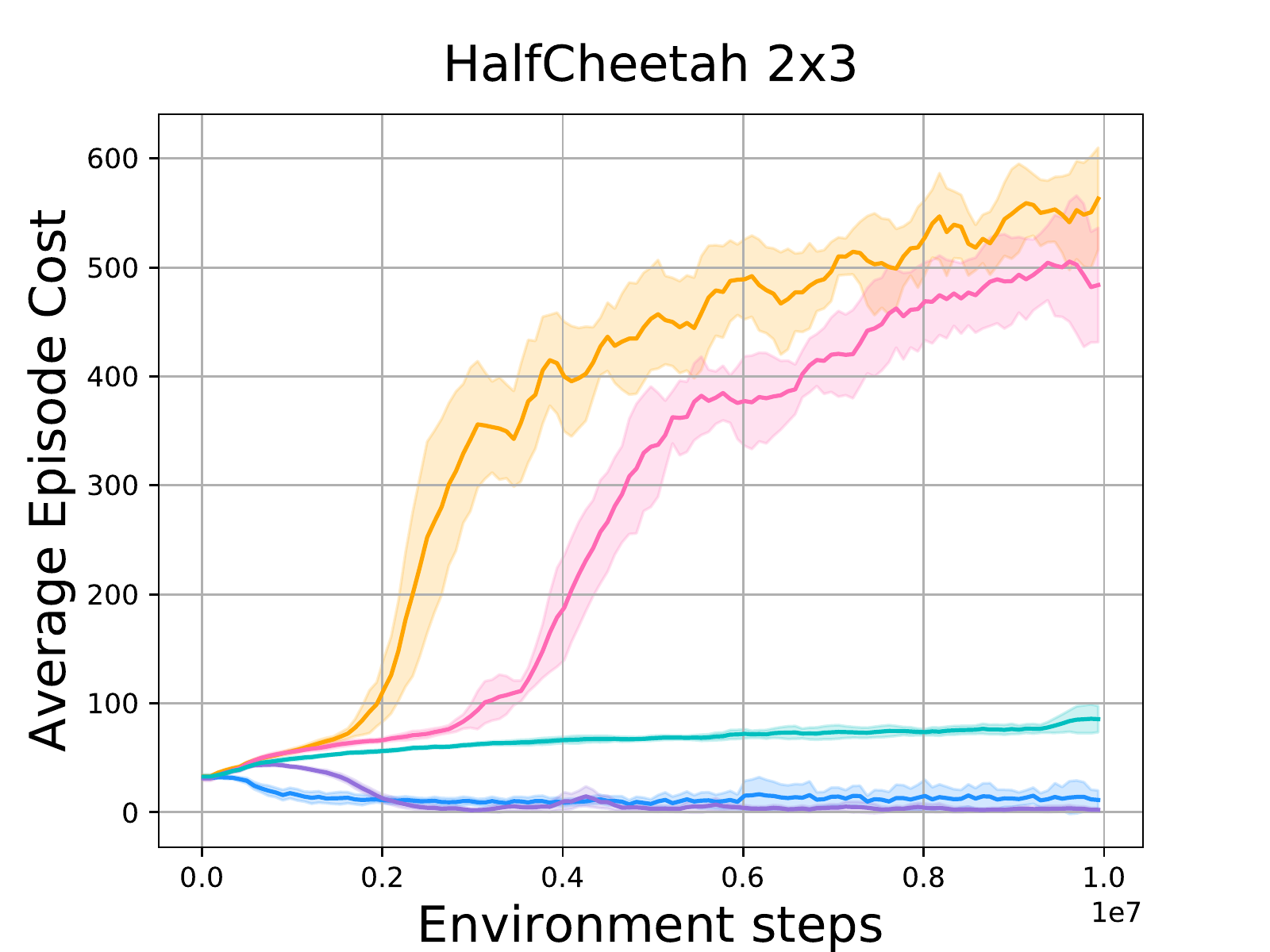}
\includegraphics[width=0.24\linewidth]{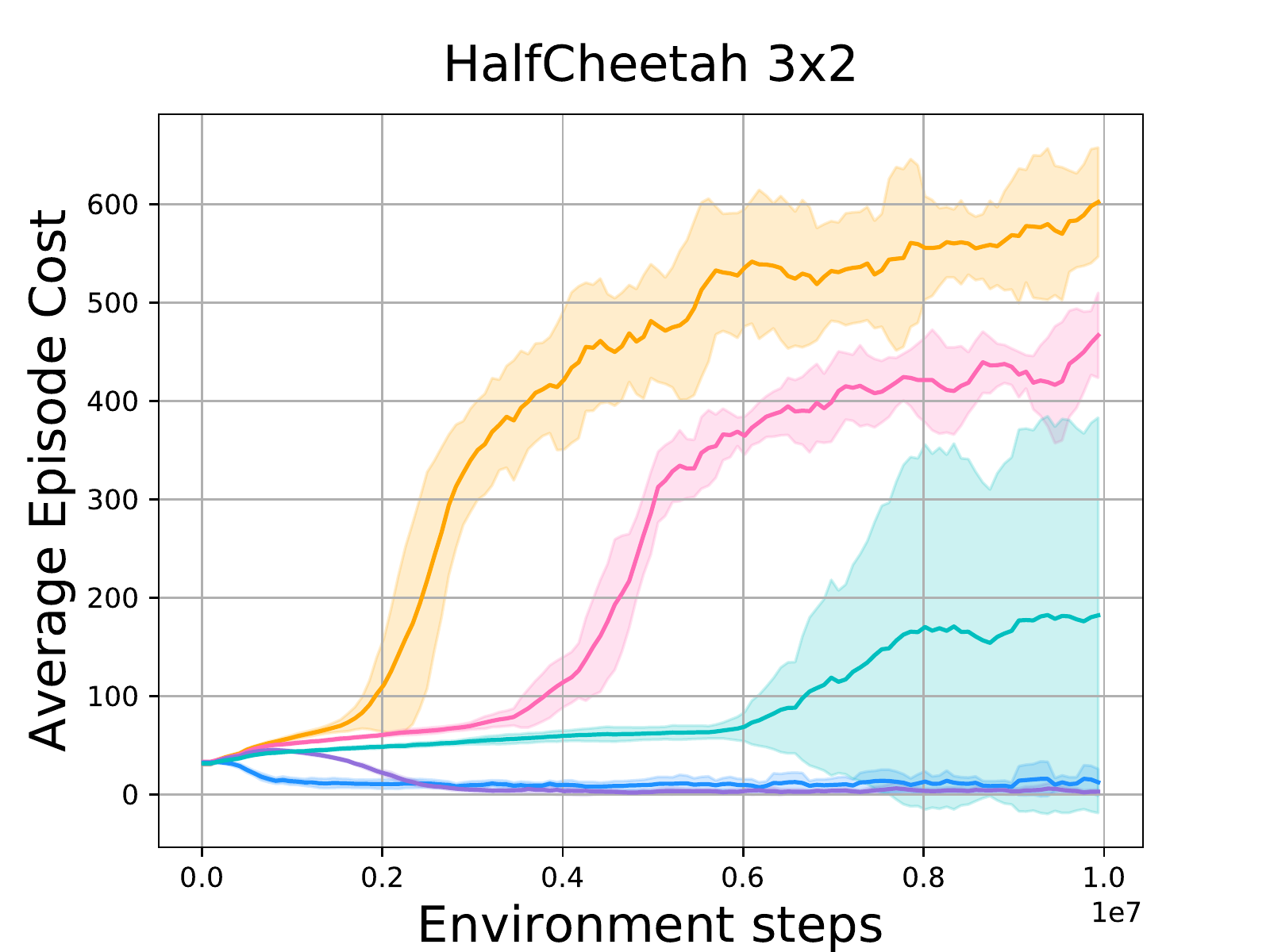}
\includegraphics[width=0.24\linewidth]{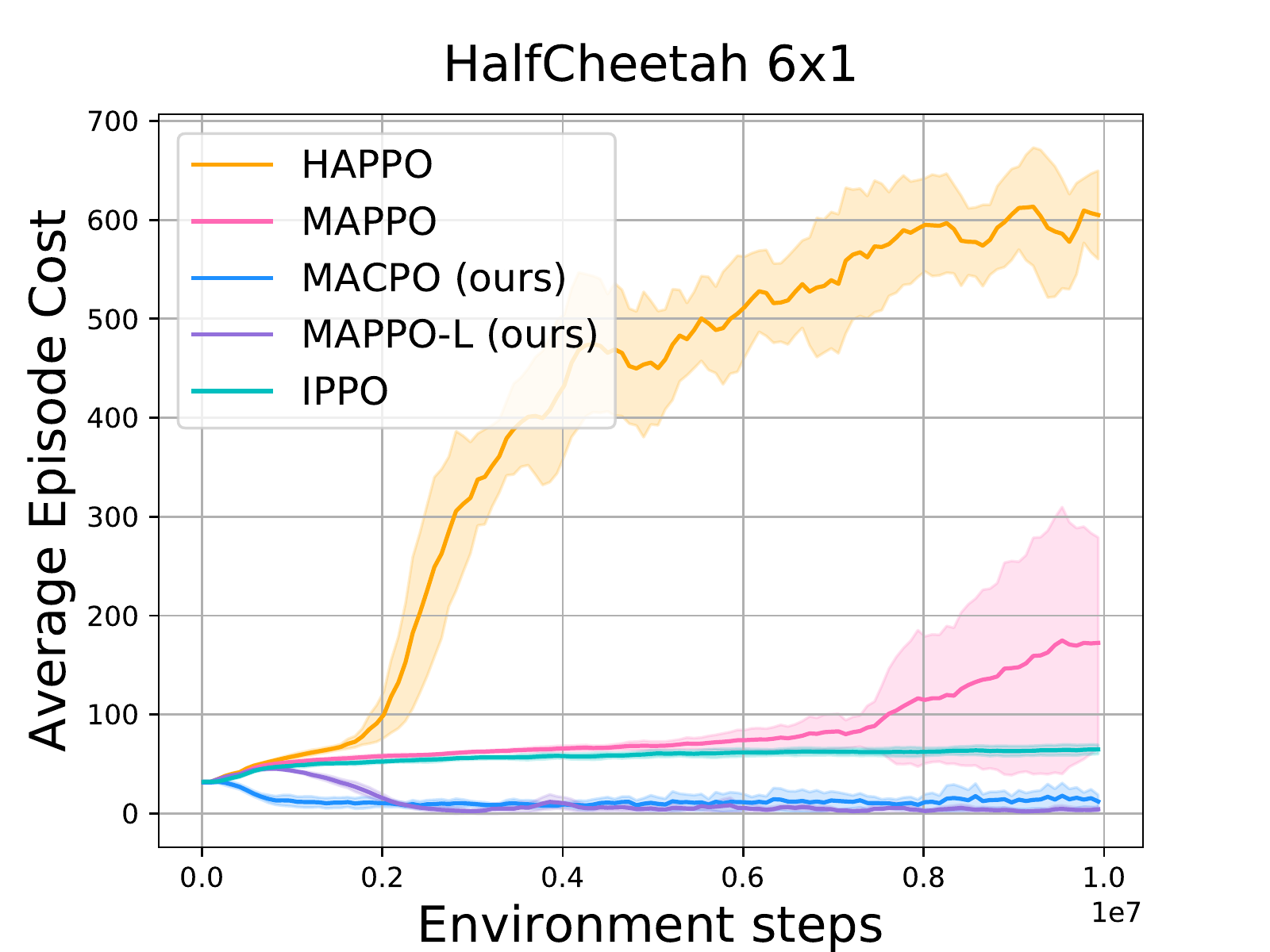}

\includegraphics[width=0.24\linewidth]{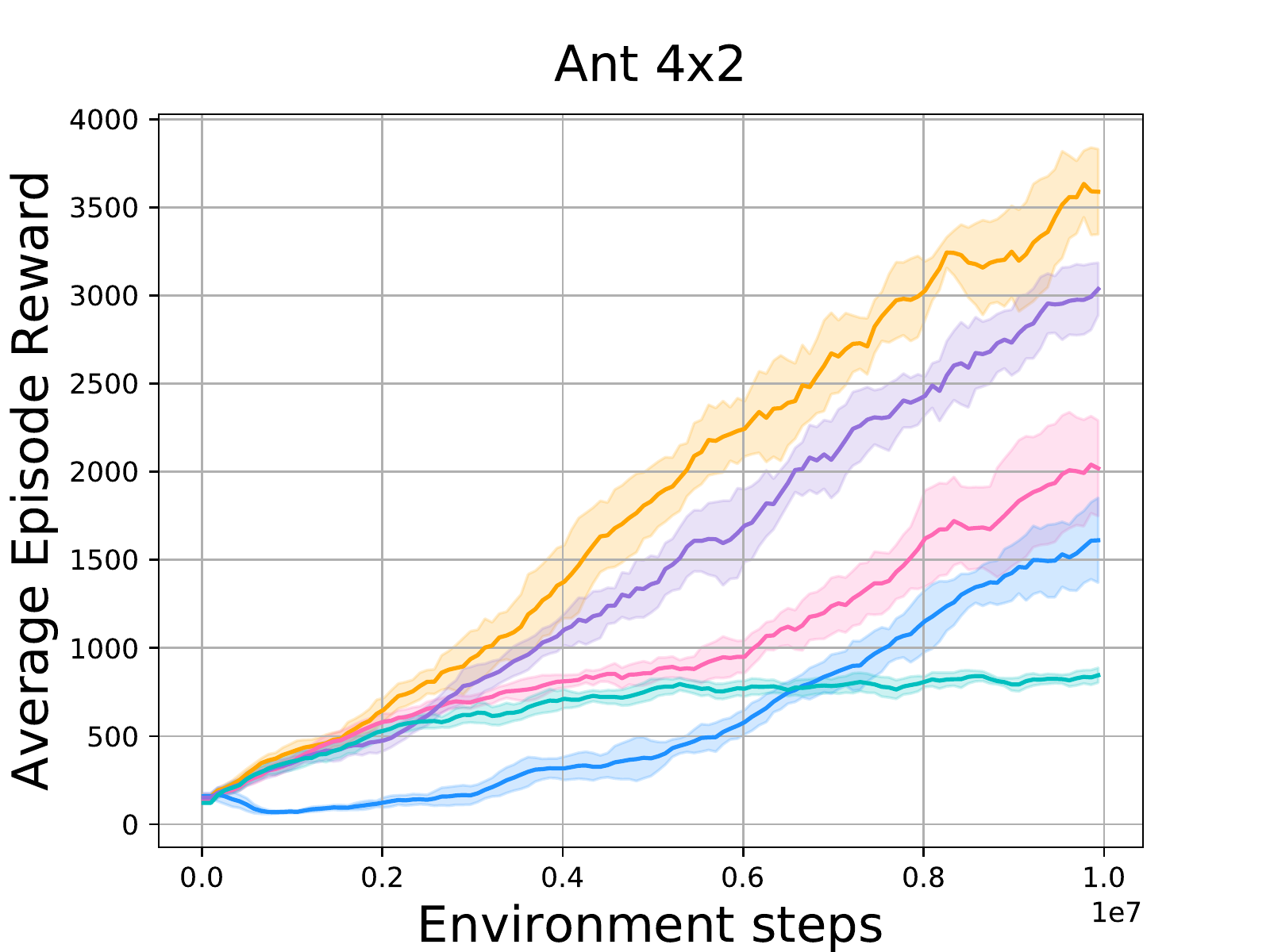}
\includegraphics[width=0.24\linewidth]{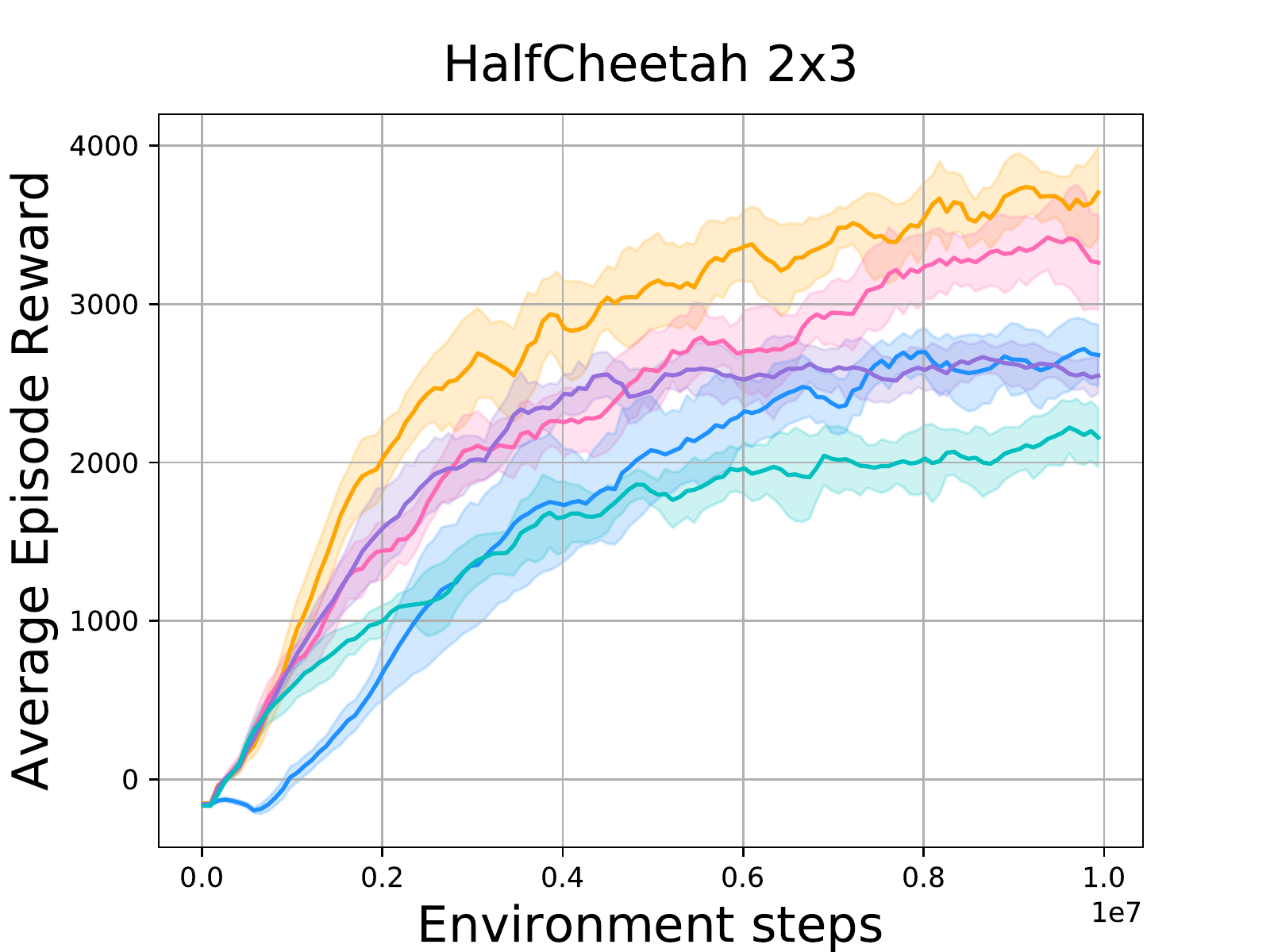}
\includegraphics[width=0.24\linewidth]{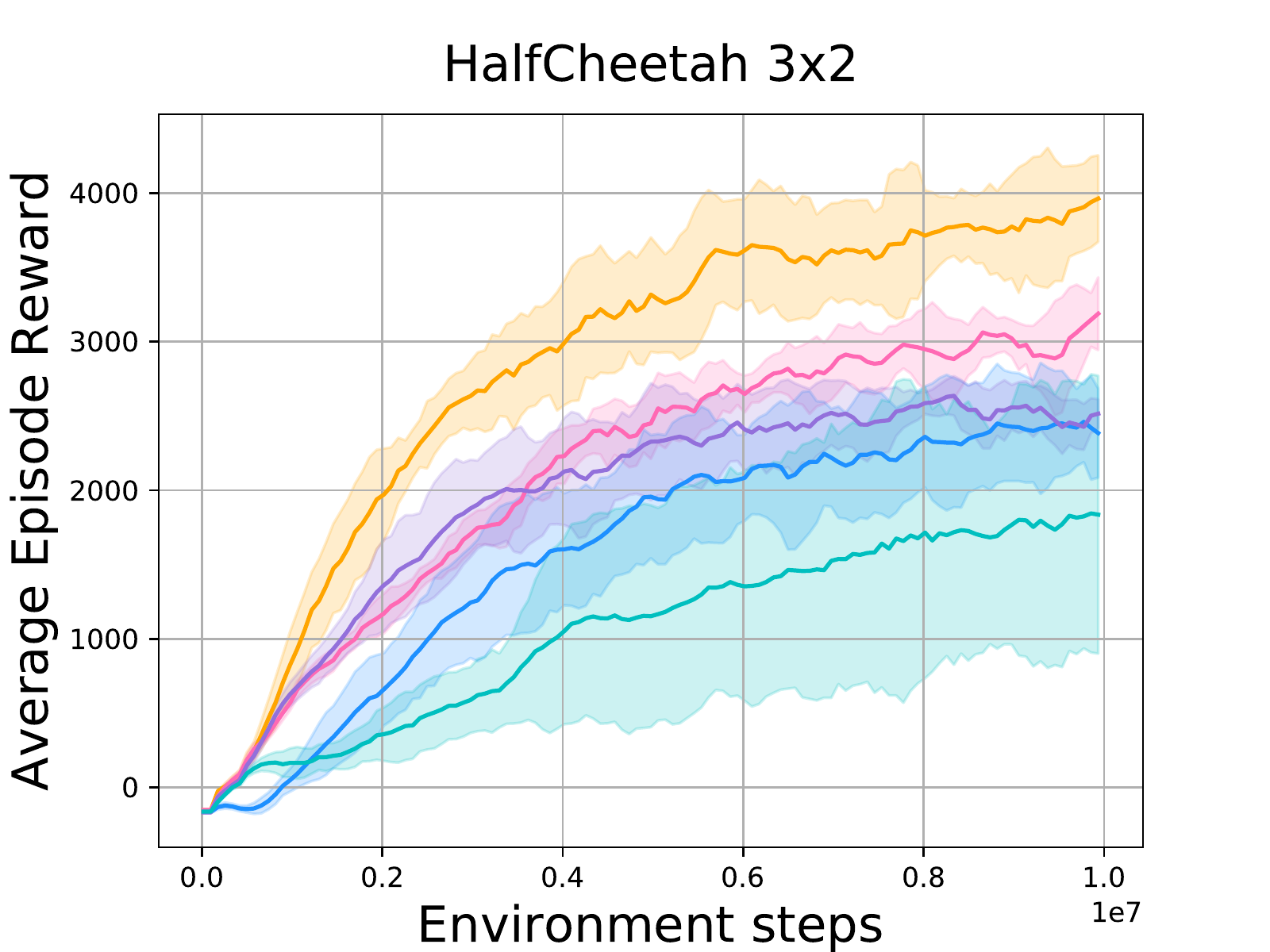}
\includegraphics[width=0.24\linewidth]{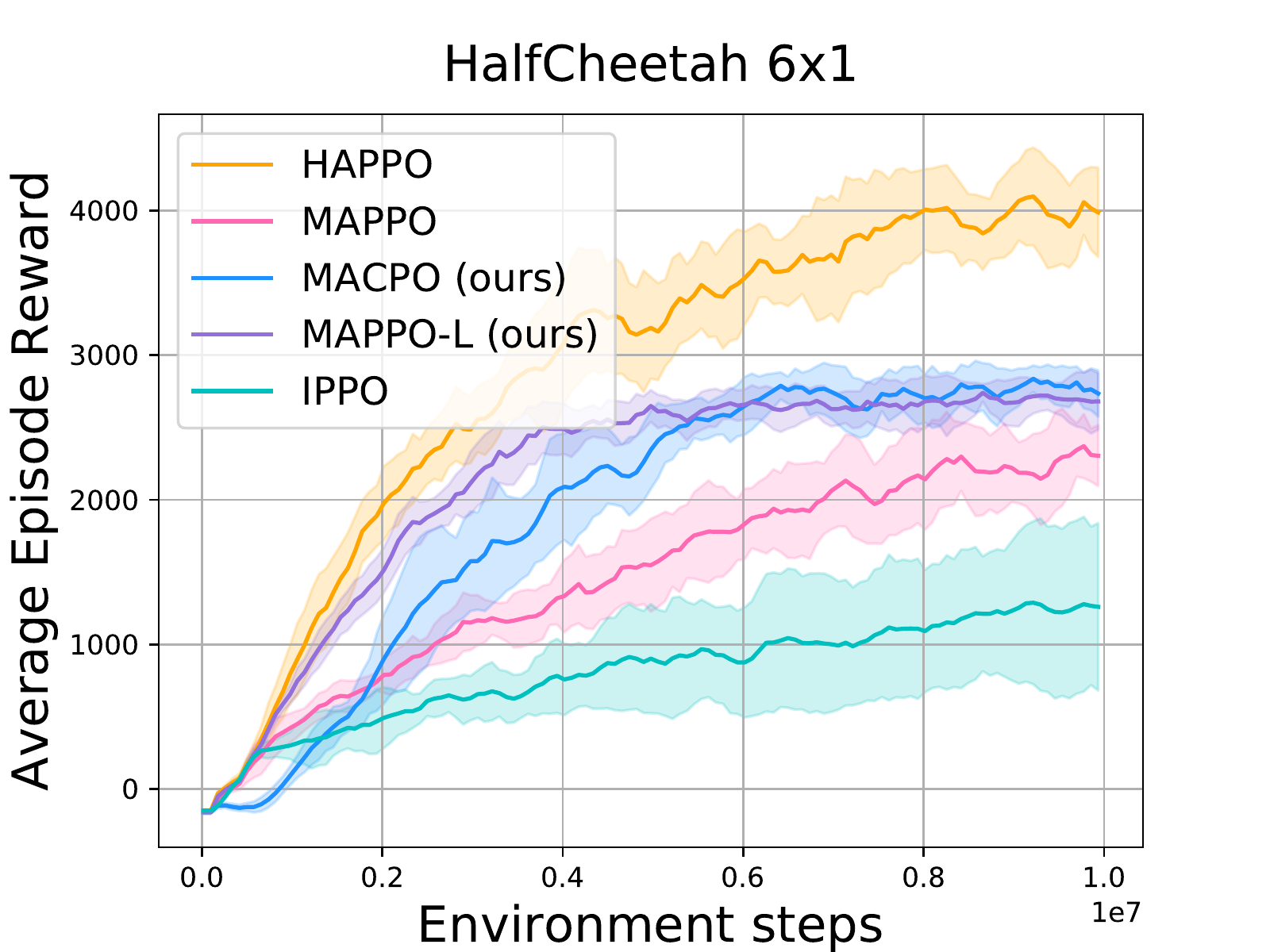}
}
 	
 	\caption{\normalsize Performance comparisons on tasks of Safe ManyAgent Ant, Ant, and HalfCheetah in terms of cost (the first row of each subfigure) and reward (the second row of each subfigure). 
 	Our methods consistently achieve  zero costs, thus satisfying safe constraints,  on all tasks. In terms of reward, our methods outperform IPPO and MAPPO on some tasks  but  underperform  HAPPO, which  is also an unsafe algorithm. We  show more results on the ManyAgent Ant  in Appendix~\ref{appendix:Experiments of-Different-Safe-ManyAgent-Ant-Environments}. 
 	} 
 	\label{fig:safe_mujoco_manyant_ant_halfCheetah}
 \end{figure*}

\section{Experiments}
Although MARL researchers have long had a variety of environments to test different algorithms,  such as StarCraftII \citep{samvelyan2019starcraft} and Multi-Agent MuJoCo \citep{peng2020facmac}, no public safe MARL benchmark has been proposed; this  impedes researchers from evaluating and benchmarking safety-aware multi-agent learning methods. 
As a side contribution of this paper, we introduce SMAMuJoCo and SMARobosuite, two safety-aware extensions of the MuJoCo~\citep{todorov2012mujoco} and Robosuite~\citep{zhu2020robosuite} environments, which are designed for safe MARL research. SMAMuJoCo and SMARobosuite are fully cooperative, continuous, and decentralised benchmarks considering constraints for safety. In many real-world applications, modular robots can be quickly paired and assembled for different tasks  \citep{althoff2019effortless}. 

 \begin{figure}[H]
	\centering
	\vspace{5pt}
	\begin{minipage}[t]{1\linewidth}
		\centering
		\includegraphics[width=.7\textwidth]{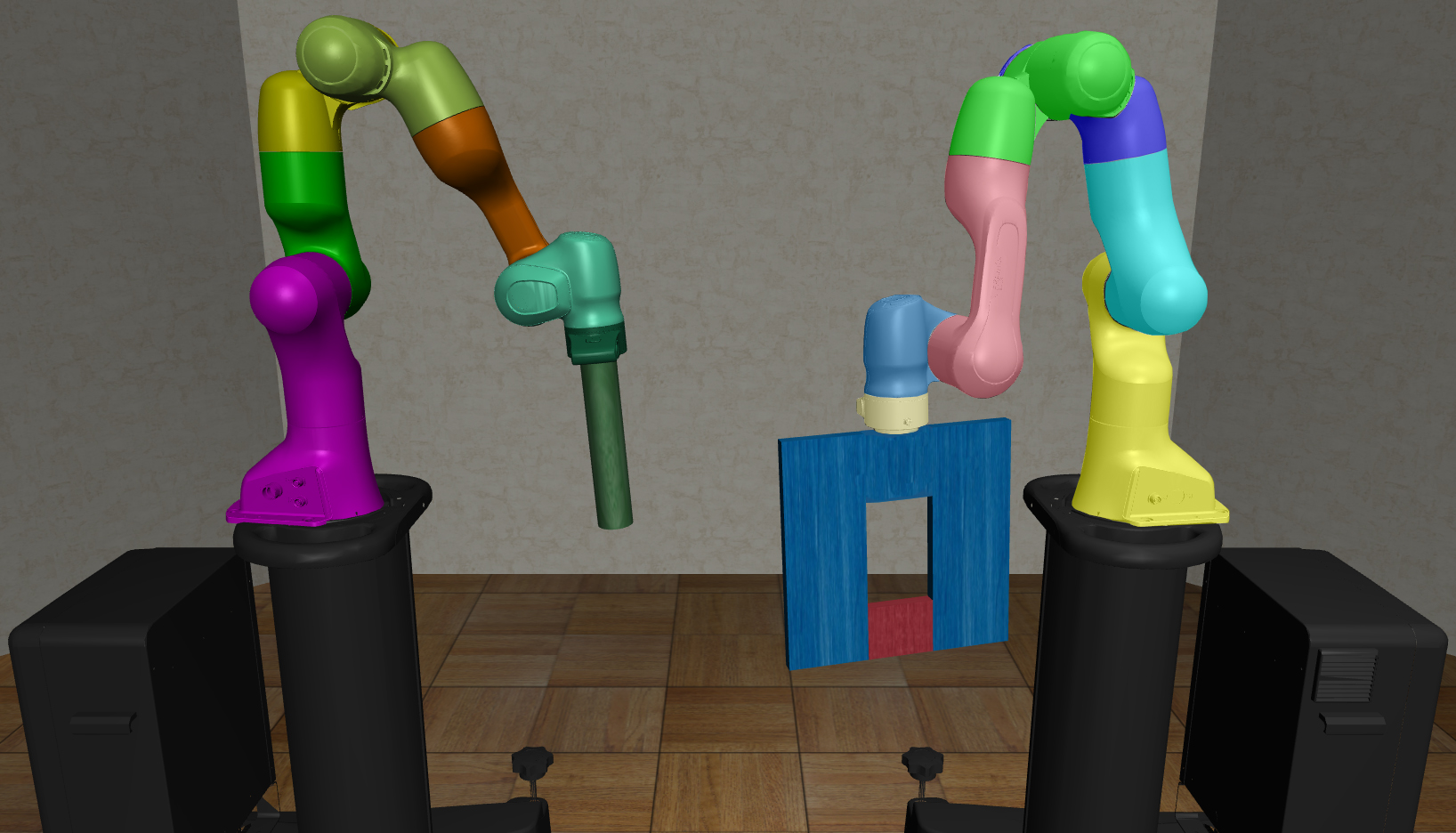} 
	\end{minipage}%
	
	\caption{  {An example task in Safe Multi-Agent Robosuite Environment. Body parts of different colours of robots are controlled by different agents. Agents cooperate to manipulate the robotic arm to achieve Peg-in-Hole task, at the same time, avoiding crashing into unsafe red areas.}}
	\label{fig:safety-robosuite-two-arm-PegInHole}
\end{figure}
\vspace{-10pt}

 \begin{figure}[H]
	\centering
	\begin{minipage}[t]{1\linewidth}
		\centering
		\includegraphics[width=.49\textwidth]{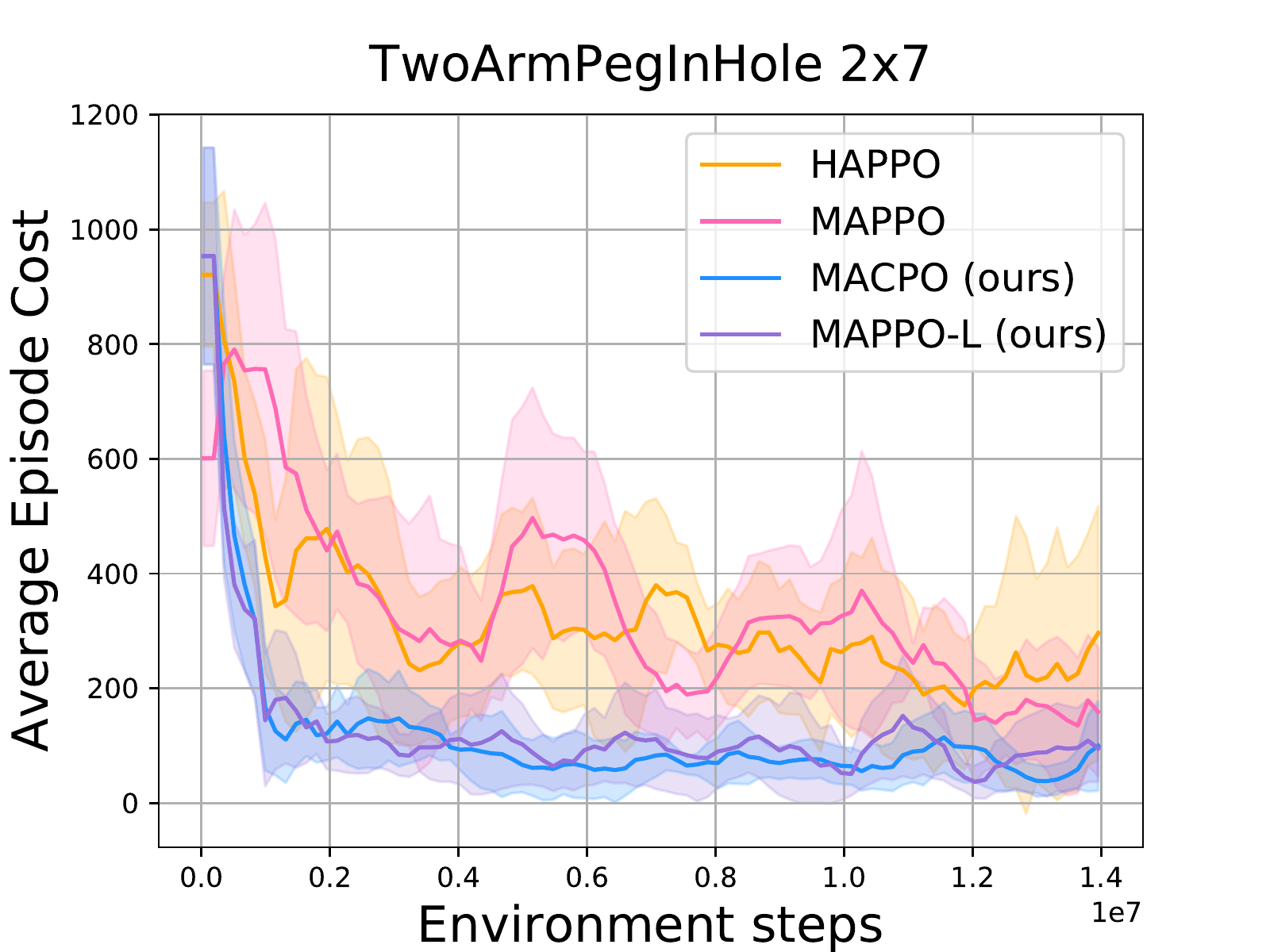}
		\includegraphics[width=.49\textwidth]{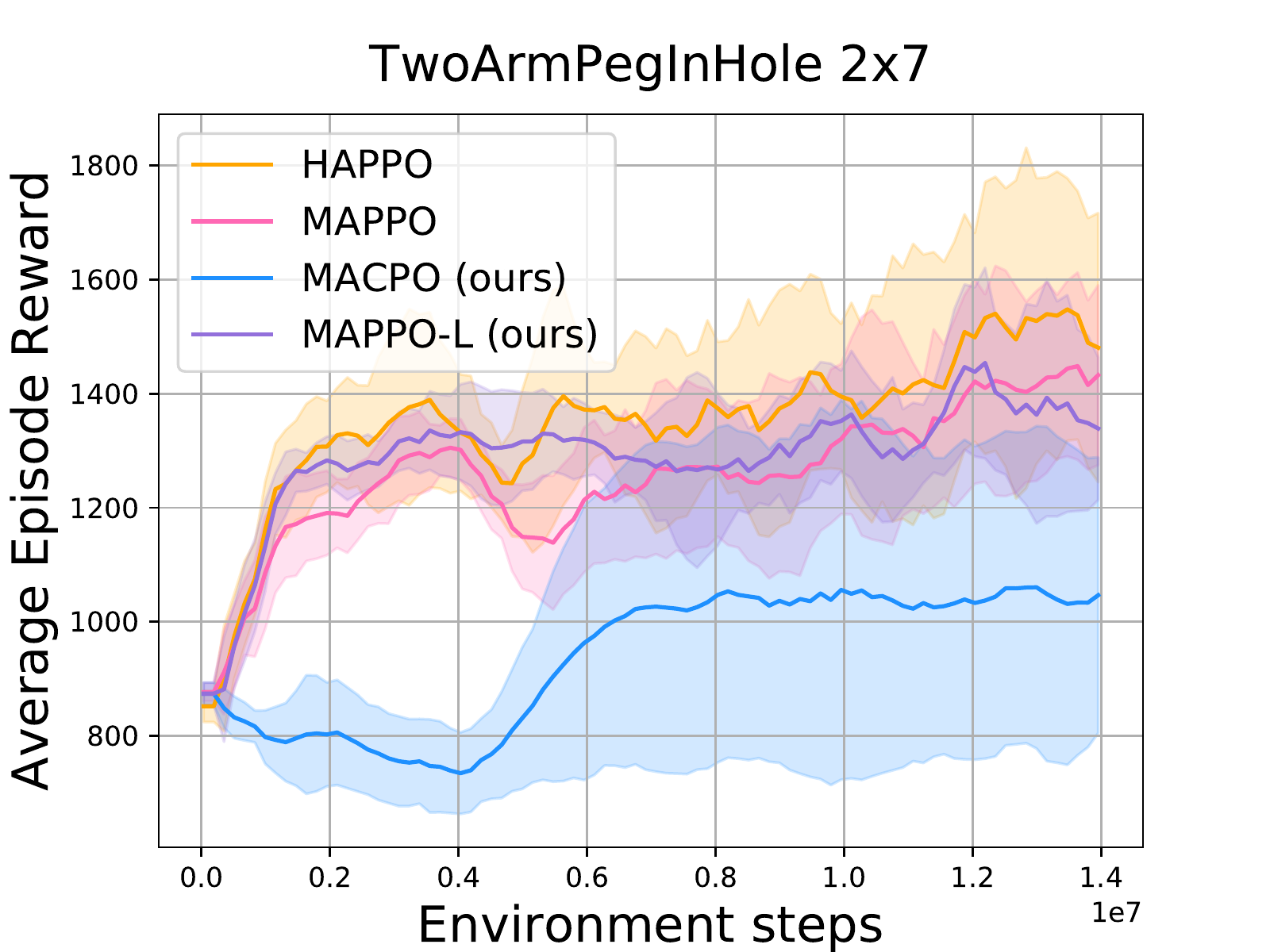}
	\end{minipage}%
	\caption{  {Performance comparisons on tasks of Safe TwoArmPegInHole in terms of cost (left) and reward (right).  Our methods consistently achieve almost zero costs, thus satisfying safe constraints.}}
 	\vspace{-5pt}
	\label{fig:experimental-results-robosuite-two-arm-PegInHole}
\end{figure}

We show example tasks in Figure \ref{fig:Safety-Ant-Environment} and \ref{fig:safety-robosuite-two-arm-PegInHole}, in our environment, safety-aware agents have to learn  not only skilful manipulations of a robot, but also to avoid crashing into unsafe obstacles and positions. W refer to
Appendix \ref{appendix:Introduction-Safe-MAMujoco} \& \ref{appendix:safe-multi-agent-robosuite} for more  descriptions of the two proposed environments.


We use SMAMuJoCo and SMARobosuite to examine if our proposed MACPO and MAPPO-Lagrangian agents can satisfy their safety constraints and cooperatively learn to achieve high rewards, compared to existing MARL algorithms. 
Notably, our proposed methods adopt  two different approaches for achieving safety.  MACPO reaches safety via \textsl{hard}  constraints and backtracking line search, while  MAPPO-Lagrangian maintains a rather \textsl{soft} safety awareness by performing gradient descents on the clip objective.  

Figure~\ref{fig:safe_mujoco_manyant_ant_halfCheetah}  \&   \ref{fig:experimental-results-robosuite-two-arm-PegInHole}  show the cost and the reward performance comparisons between MACPO, MAPPO-Lagrangian, MAPPO~\citep{yu2021surprising}, IPPO~\citep{de2020independent}, and HAPPO \citep{kuba2021trust} algorithms on series of continuous control tasks. Results are reported with $5$ random seeds.  
Detailed hyperparameter settings for each algorithm are described in Appendix \ref{appendix:Details-Settings-Experiments}. 
Figure   \ref{fig:safe_mujoco_manyant_ant_halfCheetah} \& \ref{fig:experimental-results-robosuite-two-arm-PegInHole}  should be interpreted at  three-folds; each subfigure represents a different robot, within each subfigure, three task setups in terms of multi-agent control are considered, for each task,  we plot the cost curves  (the lower the better) in the upper row, and plot the reward curves (the higher the better) in the bottom row.  
These results reveal that both MACPO and MAPPO-Lagrangian, even if initialised with \textbf{unsafe} policies, quickly learn to satisfy safety constraints, and keep their explorations within the feasible policy space. This stands in contrast to IPPO, MAPPO, and HAPPO which largely violate the constraints thus being unsafe. Notably, the cost curves reveal a subtle difference between MACPO and MAPPO-Lagrangian caused by their different approaches (hard/soft) to constrained optimisation. The hard restrictions of MACPO ascertain absolute stability in decreasing the cost value. In contrast, the fixed number of steps ($5$ in our experiments) of gradient ascent on a Lagrangian objective, performed by MAPPO-Lagrangian, can be insufficient to solve the constrained problem exactly at early stages of training. However, the method quickly stabilises and does not deviate from the safe region. If necessary, one could achieve this exactness by increasing the number of steps of the aforementioned gradient ascent.

Furthermore, our algorithms achieve comparable reward scores; both methods are often better than IPPO. In general, the performance (in terms of reward) of MAPPO-Lagrangian is better than of MACPO; moreover,  MAPPO-Lagrangian outperforms the unconstrained MAPPO on challenging \textsl{Ant} tasks. We note that on none of the tasks the reward of HAPPO was exceeded, though it is unsafe. 

\section{Conclusion}
In this paper, we tackled multi-agent policy optimisation problems with safety constraints. Central to  our findings is the safe multi-agent policy iteration procedure that attains theoretically-justified monotonic improvement guarantee and constraints satisfaction guarantee at every iteration during training. Based on this,  
we proposed two safe MARL algorithms: MACPO and MAPPO-Lagrangian.  To demonstrate their effectiveness, we introduced two new benchmark suites of SMAMuJoCo and SMARobosuite, and compared our methods against strong MARL baselines. Results show that both of our methods can significantly outperform existing  state-of-the-art methods such as IPPO, MAPPO and HAPPO in terms of safety, meanwhile maintaining comparable performance in terms of reward. 

To our knowledge, our solutions are the first of its kind in the safe MARL domain. In the future, we plan to test the performance of our algorithms in physical settings, thus taking a step towards the deployment of sage, artificial intelligence techniques in practical robotics.

\normalem 







\nocite{langley00}

\bibliography{example_paper}
\bibliographystyle{icml2022}

\newpage
\appendix
\onecolumn



\section{Proofs of preliminary results}
\vspace{-5pt}
\label{appendix:preliminaries}

\maadlemma*

\begin{proof}
We write the multi-agent advantage as in its definition, and expand it in a telescoping sum.
\begin{align}
    A_{\boldsymbol{\pi}}^{i_{1: h}}\left(s, \va^{i_{1: h}}\right) &= Q_{\boldsymbol{\pi}}^{i_{1: h}}\left(s, \va^{i_{1: h}}\right) - V_{\boldsymbol{\pi}}\left(s\right)
    \nonumber\\
    &= \sum_{j=1}^{h} \left[ Q_{\boldsymbol{\pi}}^{i_{1: j}}\left(s, \va^{i_{1: j}}\right) - Q_{\boldsymbol{\pi}}^{i_{1: j-1}}\left(s, \va^{i_{1: j-1}}\right)\right]
    \nonumber\\
    &= \sum_{j=1}^{h} A_{\boldsymbol{\pi}}^{i_j}\left(s, \va^{i_{1: j-1}}, a^{i_j}\right). \nonumber
\end{align}
\end{proof}

\surrogatecostlemma*

\begin{proof}
From the proof of Theorem 1 from \cite{schulman2015trust} (in particular, equations (41)-(45)), applied to joint policies $\boldsymbol{\pi}$ and $\boldsymbol{\bar{\pi}}$, we conclude that
\begin{align}
    &J^i_j(\boldsymbol{\bar{\pi}}) \leq J^i_j(\boldsymbol{\pi}) + \E_{\rs\sim\rho_{\boldsymbol{\pi}}, \rva\sim\boldsymbol{\bar{\pi}} }
    \left[ A^i_{j, \boldsymbol{\pi}}(\rs, \ra^i) \right] + \frac{4\alpha^2\gamma\max_{s, a^i} |A^i_{j, \boldsymbol{\pi}}(s, a^i)|}{(1-\gamma)^2}, \nonumber\\
    &\text{where} \ \ \ \ \  \alpha = D^{\text{max}}_{\text{TV}}\left( \boldsymbol{\pi}, \boldsymbol{\bar{\pi}}\right) =\max_s D_{\text{TV}}\left( \boldsymbol{\pi}(\cdot|s), \boldsymbol{\bar{\pi}}(\cdot|s)\right).\nonumber
\end{align}
Using the inequality $D_{\text{TV}}(p, q)^2 \leq D_{\text{KL}}(p, q)$ \citep{pollard2000asymptopia, schulman2015trust}, we obtain
\begin{align}
    &J^i_j(\boldsymbol{\bar{\pi}}) \leq J^i_j(\boldsymbol{\pi}) + \E_{\rs\sim\rho_{\boldsymbol{\pi}}, \rva\sim\boldsymbol{\bar{\pi}} }
    \left[ A^i_{j, \boldsymbol{\pi}}(\rs, \ra^i) \right] + \frac{4\gamma\max_{s, a^i} |A^i_{j, \boldsymbol{\pi}}(s, a^i)|}{(1-\gamma)^2}D^{\text{max}}_{\text{KL}}(\boldsymbol{\pi}, \boldsymbol{\bar{\pi}}), \nonumber\\
    &\text{where} \ \ \ \ \  D^{\text{max}}_{\text{KL}}(\boldsymbol{\pi}, \boldsymbol{\bar{\pi}}) = \max_s D_{\text{KL}}\left(\boldsymbol{\pi}(\cdot|s), \boldsymbol{\bar{\pi}}(\cdot|s)\right). \nonumber 
\end{align}
Notice now that we have $\E_{\rs\sim\rho_{\boldsymbol{\pi}}, \rva\sim\boldsymbol{\bar{\pi}} }
    \left[ A^i_{j, \boldsymbol{\pi}}(\rs, \ra^i) \right] 
    = \E_{\rs\sim\rho_{\boldsymbol{\pi}}, \ra^i\sim\bar{\pi}^i }
    \left[ A^i_{j, \boldsymbol{\pi}}(\rs, \ra^i) \right]$, and
\begin{align}
    \label{eq:kl-inequality}
    &D_{\text{KL}}^{\text{max}}(\boldsymbol{\pi}, 
    \boldsymbol{\bar{\pi}}) =\max_{s}D_{\text{KL}}\left(\boldsymbol{\pi}(\cdot|s), \boldsymbol{\bar{\pi}}(\cdot|s)\right) = 
    \max_s\left( \sum_{l=1}^{n}D_{\text{KL}}\left( \pi^l(\cdot|s), \bar{\pi}^l(\cdot|s) \right) \right)\nonumber\\
    &\leq \sum_{l=1}^{n}\max_{s}D_{\text{KL}}\left( \pi^l(\cdot|s), \bar{\pi}^l(\cdot|s) \right) = \sum_{l=1}^{n}D^{\text{max}}_{\text{KL}}\left( \pi^l, \bar{\pi}^l\right).
\end{align}
Setting $\nu^i_j = \frac{4\gamma\max_{s, a^i} |A^i_{j, \boldsymbol{\pi}}(s, a^i)|}{(1-\gamma)^2}$, we finally obtain
\begin{align}
    J^i_j(\boldsymbol{\bar{\pi}}) \leq J^i_j(\boldsymbol{\pi}) + L^i_{j, \boldsymbol{\pi}}\left( \bar{\pi}^i \right) + \nu^i_j\sum_{l=1}^{n}D^{\text{max}}_{\text{KL}}\left( \pi^l, \bar{\pi}^l\right).\nonumber
\end{align}

\end{proof}

\clearpage
\section{Auxiliary Results for Algorithm \ref{algorithm:theoretical-safe-matrpo}}
\label{appendix:results-safe-matrpo}

\begin{remark}
\label{remark:safe-updates}
In Algorithm \ref{algorithm:theoretical-safe-matrpo}, we compute the size of KL constraint as 
\begin{align}
    \delta^{i_h} = \min\Bigg\{ \min_{l \leq h-1}\min_{1\leq j \leq m^l}\frac{ c^{i_l}_j - J^{i_l}_j(\boldsymbol{\pi}_k) 
    - L^{i_l}_{j, \boldsymbol{\pi}_k}(\pi^{i_l}_{k+1})
    -\nu^{i_l}_j\sum_{u=1}^{h-1}D^{\text{max}}_{\text{KL}}(\pi^{u}_k, \pi^{u}_{k+1}) }{ \nu^{i_l}_j }, \nonumber\\
    \min_{l \geq h+1}\min_{1\leq j \leq m^l}\frac{ c^{i_l}_j - J^{i_l}_j(\boldsymbol{\pi}_k) - \nu^{i_l}_j\sum_{u=1}^{h-1}D^{\text{max}}_{\text{KL}}(\pi^{u}_k, \pi^{u}_{k+1}) }{ \nu^{i_l}_j }\Bigg\}.\nonumber
\end{align}
Note that $\delta^{i_1}$ (i.e., $h=1$) is guaranteed to be non-negative if $\boldsymbol{\pi}_k$ satisfies safety constraints; that is because then $c^{i_l}_j \geq J^{i_l}_j(\boldsymbol{\pi}_k)$ for all $l$ and $j$, and the set $\{l \ | \ l\textless h\}$ is empty.

This formula for $\delta^{i_h}$, combined with Lemma \ref{lemma:surrogate-cost}, assures that the policies $\pi^{i_h}$ within $\delta^{i_h}$ max-KL distance from $\pi^{i_h}_k$ will not violate other agents' safety constraints, as long as the base joint policy $\boldsymbol{\pi}_k$ did not violate them (which assures $\delta^{i_1}\geq 0$). To see this, notice that for every $l= 1, \dots, h-1$, and $j=1, \dots, m^l$,
\begin{align}
    &D^{\text{max}}_{\text{KL}}(\pi^{i_h}_k, \pi^{i_h}) \leq \delta^{i_h} \leq \frac{ c^{i_l}_j - J^{i_l}_j(\boldsymbol{\pi}_k)
    - L^{i_l}_{j, \boldsymbol{\pi}_k}(\pi^{i_l}_{k+1})
    -\nu^{i_l}_j\sum_{u=1}^{h-1}D^{\text{max}}_{\text{KL}}(\pi^{u}_k, \pi^{u}_{k+1}) }{ \nu^{i_l}_j },\nonumber\\
    &\text{implies} \ \ J^{i_l}_j(\boldsymbol{\pi}_k) + L^{i_l}_{j, \boldsymbol{\pi}_k}(\pi^{i_l}_{k+1}) + \nu^{i_l}_j\sum_{u=1}^{h-1}D^{\text{max}}_{\text{KL}}(\pi^{u}_k, \pi^{u}_{k+1}) + \nu^{i_l}_j D^{\text{max}}_{\text{KL}}(\pi^{i_h}_k, \pi^{i_h}) \leq c^{i_l}_j.\nonumber
\end{align}
By Lemma \ref{lemma:surrogate-cost}, the left-hand side of the above inequality is an upper bound of {\small $J^{i_l}_j(\boldsymbol{\pi}_{k+1}^{i_{1:h-1}},\pi^{i_h},\boldsymbol{\pi}^{-i_{1:h}}_k)$}, which implies that the update of agent $i_h$ doesn't violate the constraint of $J^{i_l}_j$. The fact that the constraints of $J^{i_l}_j$ for $l\geq h+1$ are not violated, i.e., 
\begin{align}
    \label{eq:future-safety}
    J^{i_l}_j(\boldsymbol{\pi}_k) + \nu^{i_l}_j\sum_{u=1}^{h-1}D_{\text{KL}}^{\text{max}}(\pi^u_k, \pi^u_{k+1}) + \nu^{i_l}_j D_{\text{KL}}^{\text{max}}(\pi^{i_h}_k, \pi^{i_h}) \leq c^{i_l}_j,
\end{align}
is analogous.
\end{remark}


\thsafematrpo*

\begin{proof}
Safety constraints are assured to be met by Remark \ref{remark:safe-updates}. It suffices to show the monotonic improvement property. Notice that at every iteration $k$ of Algorithm \ref{algorithm:theoretical-safe-matrpo}, $\pi^{i_h}_k \in \overline{\Pi}^{i_h}$. Clearly $D^{\text{max}}_{\text{KL}}(\pi^{i_h}_k, \pi^{i_h}_k) = 0 \leq \delta^{i_h}$. Moreover,
\begin{align}
    J^{i_h}_j(\boldsymbol{\pi}_k) + L^{i_h}_{j, \boldsymbol{\pi}_k}(\pi^{i_h}_k) + \nu^{i_h}_j D^{\text{max}}_{\text{KL}}(\pi^{i_h}_k, \pi^{i_h}_k) = J^{i_h}_j(\boldsymbol{\pi}_k) \leq c^{i_h}_j - \nu^{i_h}_j \sum_{l=1}^{h-1}D^{\text{max}}_{\text{KL}}(\pi^{i_l}_{k}, \pi^{i_l}_{k+1}),\nonumber
\end{align}
where the inequality is guaranteed by updates of previous agents, as described in Remark \ref{remark:safe-updates} (Inequality \ref{eq:future-safety}). By Theorem 1 from \cite{schulman2015trust}, we have
\begin{align}
        \label{eq:sad-trpo-monotonic}
        &J(\boldsymbol{\pi}_{k+1}) \geq
        J(\boldsymbol{\pi}_k) +
        \E_{\rs\sim\rho_{\boldsymbol{\pi}_k}, \rva\sim\boldsymbol{\pi}_{k+1}}\left[ A_{\boldsymbol{\pi}_k}(\rs, \rva)\right]- \nu D_{\text{KL}}^{\text{max}}(\boldsymbol{\pi}_{k}, \boldsymbol{\pi}_{k+1}),\nonumber\\
        &\text{which by Equation \ref{eq:kl-inequality} is lower-bounded by}\nonumber\\
        &\geq J(\boldsymbol{\pi}_k) +
        \E_{\rs\sim\rho_{\boldsymbol{\pi}_k}, \rva\sim\boldsymbol{\pi}_{k+1}}\left[ A_{\boldsymbol{\pi}_k}(\rs, \rva)\right] - \sum_{h=1}^{n}\nu D_{\text{KL}}^{\text{max}}(\pi^{i_{h}}_{k}, \pi^{i_{h}}_{k+1})\nonumber\\
        &\text{which by Lemma \ref{lemma:maad} equals}\nonumber\\
        &=  J(\boldsymbol{\pi}_k) +
        \sum_{h=1}^{n}\E_{\rs\sim\rho_{\boldsymbol{\pi}_k}, \rva^{i_{1:h}}\sim\boldsymbol{\pi}^{i_{1:h}}_{k+1}}\left[ A^{i_h}_{\boldsymbol{\pi}_k}(\rs, \rva^{i_{1:h-1}}, \ra^{i_h})\right] - \sum_{h=1}^{n}\nu D_{\text{KL}}^{\text{max}}(\pi^{i_{h}}_{k}, \pi^{i_{h}}_{k+1})\nonumber\\
        &= J(\boldsymbol{\pi}_{k}) + 
        \sum_{h=1}^{n}\left( L^{i_{1:h}}_{\boldsymbol{\pi}_{k}}(\boldsymbol{\pi}^{i_{1:h-1}}_{k+1},\pi^{i_{h}}_{k+1}) 
        - \nu D_{\text{KL}}^{\text{max}}(\pi^{i_{h}}_{k}, \pi^{i_{h}}_{k+1})\right),\\
        &\text{and as for every $h$, $\pi^{i_{h}}_{k+1}$ is the argmax, this is lower-bounded by}\nonumber\\
        &\geq J(\boldsymbol{\pi}_{k}) + 
        \sum_{h=1}^{n}\left( L^{i_{1:h}}_{\boldsymbol{\pi}_{k}}(\boldsymbol{\pi}^{i_{1:h-1}}_{k+1},\pi^{i_{h}}_{k}) 
        - \nu D_{\text{KL}}^{\text{max}}(\pi^{i_{h}}_{k}, \pi^{i_h}_{k})\right),\nonumber\\
        &\text{which, as follows from Definition \ref{definition:localsurrogate}, equals}\nonumber\\
        &= \mathcal{J}(\boldsymbol{\pi_{k}}) + \sum_{h=1}^{n}0 =
        J(\boldsymbol{\pi}_{k}), \ \text{which finishes the proof.}\nonumber
    \end{align}
\end{proof}

\clearpage

\section{Auxiliary Results for Implementation of MACPO}
\label{appendix:auxiliary-results-implementation}

\begin{theorem}
The solution to the following problem

\begin{align}
\vp*=\min_{x}  \vg^{T} \vx \nonumber \\
\text { s.t. } \vb^{T} \vx+c \leq 0 \nonumber\\
 \vx^{T} \mH \vx \leq \delta, \nonumber
\end{align}

where $\vg, \vb, \vx \in \mathbb{R}^{n}, c, \delta \in \mathbb{R}, \delta>0, \mH \in \mathbb{S}^{n}$, and $\mH \succ 0 .$ When there is at least one strictly feasible point, the optimal point $\vx^{*}$ satisfies:
\begin{equation}
\vx^{*}=-\frac{1}{\lambda_{*}} \mH^{-1}\left(\vg^{T}+ v_{*} \vb\right)\nonumber
\end{equation}

where $\lambda_{*}$ and $v_{*}$ are defined by

\begin{align}
v_{*} &=\left(\frac{\lambda_{*} c- r}{s}\right)_{+} \nonumber\\
\lambda_{*} &=\arg \max _{\lambda \geq 0}\left\{\begin{array}{ll}
f_{a}(\lambda) \triangleq \frac{1}{2 \lambda}\left(\frac{r^2}{s}-q\right)+\frac{\lambda}{2}\left(\frac{c^2}{s}-\delta\right)-\frac{rc}{s} & \text { if } \lambda c-r>0 \nonumber\\
f_{b}(\lambda) \triangleq-\frac{1}{2}\left(\frac{q}{\lambda}+\lambda \delta\right) & \text { otherwise }
\end{array}\right.\nonumber
\end{align}

where $\vq=\vg^{T} \mH^{-1} \vg$, $\vr=\vg^{T} \mH^{-1} \vb$, and $\vs=\vb^{T} \mH^{-1} \vb$.

Furthermore, let $\Lambda_{a} \triangleq\{\lambda \mid \lambda c-r>0, \lambda \geq 0\}$, and $\Lambda_{b} \triangleq \{\lambda \mid \lambda c - r \leq 0, \lambda \geq 0\} .$ The value of $\lambda_{*}$ satisfies

\begin{align}
\lambda_{*} \in\left\{\lambda^{a}_{*} \triangleq \operatorname{Proj}\left(\sqrt{\frac{q-r^{2} / s}{\delta-c^{2} / s}}, \Lambda_a\right), \lambda^{b}_{*} \triangleq \operatorname{Proj}\left(\sqrt{\frac{q}{\delta}}, \Lambda_{b}\right)\right\}
\end{align}

where $\lambda_{*}=\lambda^{a}_{*}$ if $f_{a}\left(\lambda^{a}_{*}\right)>f_{b}\left(\lambda^{b}_{*}\right)$ and $\lambda_{*}=\lambda_{*}$ otherwise, and $\operatorname{Proj}(a, \mS)$ is the projection of a point $x$ on to a set $\mS .$ Note the projection of a point $x \in \mathbb{R}$ onto a convex segment of $\mathbb{R},[a, b]$, has value $\operatorname{Proj}(x,[a, b])=\max (a, \min (b, x))$.

\end{theorem}

\begin{proof}
See \citep[Appendix 10.2]{achiam2017constrained}.
\end{proof}

\clearpage

\section{MACPO}
\label{appendix:mactrpo}
\begin{algorithm}[th!]
    \caption{MACPO}
    \label{algo:safe-MATRPO}
    \begin{algorithmic}[1]
    
    \STATE \textbf{Input:} Stepsize $\alpha$, batch size $B$, number of: agents $n$, episodes $K$, steps per episode $T$, possible steps in line search $L$.\\
    \STATE \textbf{Initialize:} Actor networks $\{\theta^{i}_{0}, \ \forall i\in \mathcal{N}\}$, Global V-value network $\{\phi_{0}\}$, Individual $V^i$-cost networks $\{\phi_{j,0}^i\}_{i=1:n, j=1:m}$, Replay buffer $\mathcal{B}$\\
    
    \FOR{$k = 0,1,\dots,K-1$}
        \STATE Collect a set of trajectories by running the joint policy $\boldsymbol{\pi}_{\vtheta_{k}} =(\pi^{1}_{\theta^{1}_{k}}, \dots,  \pi^{n}_{\theta^{n}_{k}})$.
        \STATE Push transitions $\{(o^i_{t},a^i_{t},o^i_{t+1},r_t), \forall i\in \mathcal{N},t\in T\}$ into $\mathcal{B}$.
        \STATE Sample a random minibatch of $M$ transitions from $\mathcal{B}$.
        \STATE Compute advantage function $\hat{A}(\rs, \rva)$ based on global V-value network with GAE.
        \STATE Compute cost-advantage functions $\hat{A}^i_j(\rs, \ra^i)$ based on individual $V^i$-cost critics with GAE.
        \STATE Draw a random permutation of agents $i_{1:n}$.
        \STATE Set $M^{i_{1}}(\rs, \rva) = \hat{A}(\rs, \rva)$.
        \FOR{agent $i_{h} = i_{1}, \dots, i_{n}$}
            \STATE Estimate the gradient of the agent's maximisation objective\\
            \begin{center}
                $\hat{\vg}^{i_{h}}_{k} =  \frac{1}{B}\sum\limits^{B}_{b=1}\sum\limits^{T}_{t=1}\nabla_{\theta^{i_{h}}_{k} }\log\pi^{i_{h}}_{\theta^{i_{h}}_k}\left( a^{i_{h}}_{t}\mid o_{t}^{i_{h}} \right)M^{i_{1:h}}(s_{t}, \va_{t})$.
            \end{center}
            \FOR{$j=1, \dots, m^{i_h}$}
                \STATE Estimate the gradient of the agent's $j^{\text{th}}$ cost
                \begin{center}
                    $\hat{\vb}^{i_h}_j = \frac{1}{B}\sum\limits_{b=1}^{B}\sum\limits_{t=1}^{T}\nabla_{\theta^{i_h}_k}\log\pi^{i_h}_{\theta^{i_h}_k}\left(a^{i_h}_t|o^{i_h}_t\right)\hat{A}^{i_h}_j(s_t, a^{i_h}_t)$.
                \end{center}
            \ENDFOR
            
            \STATE Set $\hat{\mB}^{i_h} = \left[ \hat{\vb}^{i_h}_1, \dots, \hat{\vb}^{i_h}_m \right]$.
            
            \STATE Compute $\hat{\mH}^{i_h}_k$, the Hessian of the average KL-divergence\\
            \begin{center}
                $\frac{1}{BT}\sum\limits_{b=1}^{B}\sum\limits_{t=1}^{T}D_{\text{KL}}\left(\pi^{i_h}_{\theta^{i_h}_k}(\cdot|o^{i_h}_t), \pi^{i_h}_{\theta^{i_h}}(\cdot|o^{i_h}_t)\right)$.
            \end{center}
            
            \STATE Solve the dual (\ref{eq:dual-parameterized-policy-equation}) for $\lambda_*^{i_h}$, $\textbf{v}^{i_h}_*$.
            
            Use the conjugate gradient algorithm to compute the update direction
            \begin{center}
                $\vx^{i_h}_k = (\hat{\mH}^{i_h}_k)^{-1} \left(\vg^{i_h}_k - \hat{\mB}^{i_h}\textbf{v}^{i_h}_* \right)$,
            \end{center}
           
            \STATE Update agent $i_h$'s policy by\\
            \begin{center}
                $ \theta^{i_h}_{k+1} = \theta^{i_h}_k  + \frac{\alpha^j} {\lambda^{i_h}_*} \hat{\vx}^{i_h}_k$,
            \end{center}
            where $j\in\{0, 1, \dots, L\}$ is the smallest such $j$ which improves the sample loss, and satisfies the sample constraints, found by the backtracking line search.
            \IF{the approximate is not feasible}
            \STATE Use equation  ~(\ref{eq:no-proof-recover-policy-from-unfeasible-point}) to recover policy $\theta^{i_h}_{k+1}$ from unfeasible points.
            \ENDIF 
            \STATE Compute $M^{i_{1:h+1}}(\rs, \rva) = \frac{\pi^{i_{h}}_{\theta^{i_{h}}_{k+1}}\left(\ra^{i_{h}} \mid \ro^{i_{h}}\right)}{\pi^{i_{h}}_{\theta^{i_{h}}_{k}}\left(\ra^{i_{h}} \mid \ro^{i_{h}}\right)} M^{i_{1:h}}(\rs_{t}, \rva_{t})$.  // {Unless $h=n$.}
        \ENDFOR
        
        \STATE Update V-value network by following formula:
        \STATE $\phi_{k+1}=\arg \min _{\phi} \frac{1}{\mathcal{N}}\frac{1}{T} \sum\limits^{\mathcal{N}}_{n =1 } \sum\limits_{t=0}^{T}\left(V_{\phi}(s_{t}) - \hat{R_{t}}\right)^{2}$
    \ENDFOR
    \end{algorithmic}
\end{algorithm}


\clearpage

\section{MAPPO-Lagrangian}
\label{appendix:mappo-lagrangian}
\begin{algorithm}[H]
    \caption{MAPPO-Lagrangian}
    \label{MAPPO-Lagrangian}
    \begin{algorithmic}[1]
    
    \STATE \textbf{Input:} Stepsizes $\alpha_\theta, \alpha_\lambda$, batch size $B$, number of: agents $n$, episodes $K$, steps per episode $T$, discount factor $\gamma$.\\
    \STATE \textbf{Initialize:} Actor networks $\{\theta^{i}_{0}, \ \forall i\in \mathcal{N}\}$, Global V-value network $\{\phi_{0}\}$,\\ V-cost networks $\{ \phi^{i}_{j, 0}\}^{i\in\mathcal{N}}_{1\leq j\leq m^i}$, Replay buffer $\mathcal{B}$.\\
    
    \FOR{$k = 0,1,\dots,K-1$}
        \STATE Collect a set of trajectories by running the joint policy $\boldsymbol{\pi}_{\vtheta_{k}} =(\pi^{1}_{\theta^{1}_{k}}, \dots,  \pi^{n}_{\theta^{n}_{k}})$.
        \STATE Push transitions $\{(o^i_{t},a^i_{t},o^i_{t+1},r_t), \forall i\in \mathcal{N},t\in T\}$ into $\mathcal{B}$.
        \STATE Sample a random minibatch of $B$ transitions from $\mathcal{B}$.
        \STATE Compute advantage function $\hat{A}(\rs, \rva)$ based on global V-value network with GAE.
        \STATE Compute cost advantage functions $\hat{A}^i_j(\rs, \ra^i)$ for all agents and costs, \\
        based on V-cost networks with GAE.
        \STATE Draw a random permutation of agents $i_{1:n}$.
        \STATE Set $M^{i_{1}}(\rs, \rva) = \hat{A}(\rs, \rva)$.
        \FOR{agent $i_{h} = i_{1}, \dots, i_{n}$}
            \STATE Initialise a policy parameter $\theta^{i_h} = \theta^{i_h}_k$, \\
            and Lagrangian multipliers $\lambda^{i_h}_j=0$, $\forall j= 1, \dots, m^{i_h}$.
            \STATE Make the Lagrangian modification step of objective construction\\
            \begin{center}
                $M^{i_h, (\lambda)}(\rs_t, \rva_t) = M^{i_h}(\rs_t, \rva_t) - \sum\limits_{j=1}^{n}\lambda^{i_h}_j \hat{A}^{i_h}_j(\rs_t, \ra^{i_h}_t)$.
            \end{center}
            \FOR{$e=1, \dots, e_{\text{PPO}}$}
                \STATE Differentiate the Lagrangian PPO-Clip objective\\
                $\Delta_{\theta^{i_h}} = \nabla_{\theta^{i_h}} \frac{1}{B} \sum\limits^{B}_{b =1} \sum\limits_{t=0}^{T} \min \left( \frac{\pi^{i_h}_{\theta^{i_h}}\left(a^{i_h}_t \mid o^{i_h}_t\right)}{\pi^{i_h}_{\theta^{i_h}_k}\left(a^{i_h}_t \mid o^{i_h}_t\right)}  M^{i_h, (\lambda)}(s_t, \va_t), \ 
                \text{clip}\left( \frac{\pi^{i_h}_{\theta^{i_h}}\left(a^{i_h}_t \mid o^{i_h}_t\right)}{\pi^{i_h}_{\theta^{i_h}_k}\left(a^{i_h}_t \mid o^{i_h}_t\right)}, 1\pm\epsilon\right) M^{i_h, (\lambda)}(s_t, \va_t)\right)$.
                \STATE Update temprorarily the actor paramaters
                \begin{center}
                    $\theta^{i_h} \gets \theta^{i_h} + \alpha_\theta \Delta_{\theta^{i_h}}$.
                \end{center}
                \FOR{$j=1, \dots, m^{i_h}$}
                \STATE Approximate the constraint violation\\
                \begin{center}
                    $d^{i_h}_j = \frac{1}{BT}\sum\limits_{b=1}^{B}\sum\limits_{t=1}^{T} \hat{V}^{i_h}_{j}(s_t) - c^{i_h}_j$.
                \end{center}
                \STATE Differentiate the constraint\\
                \begin{center}
                    $\Delta{\lambda^{i_h}_j} = 
                     \frac{-1}{B}\sum\limits_{b=1}^{B}\left( d^{i_h}_j(1-\gamma)
                    + \sum\limits_{t=0}^{T} \frac{\pi^{i_h}_{\theta^{i_h}}(a^{i_h}_t|o^{i_h}_t)}{\pi^{i_h}_{\theta^{i_h}_k}(a^{i_h}_t|o^{i_h}_t)}\hat{A}^{i_h}_{j}(s_t, a^{i_h}_t)\right)$.
                \end{center}
                \ENDFOR
                \FOR{$j=1, \dots, m^{i_h}$}
                    \STATE Update temporarily the Lagrangian multiplier\\
                    \begin{center}
                        $\lambda^{i_h}_j \gets
                         \text{ReLU}\left(\lambda^{i_h}_j - \alpha_\lambda \Delta\lambda^{i_h}_j\right)$.
                     \end{center}
                \ENDFOR
            \ENDFOR
            \STATE Update the actor parameter $\theta^{i_h}_{k+1} = \theta^{i_h}$.
            \STATE Compute $M^{i_{h+1}}(\rs, \rva) = \frac{\pi^{i_{h}}_{\theta^{i_{h}}_{k+1}}\left(\ra^{i_{h}} \mid \ro^{i_{h}}\right)}{\pi^{i_{h}}_{\theta^{i_{h}}_{k}}\left(\ra^{i_{h}} \mid \ro^{i_{h}}\right)} M^{i_{h}}(\rs, \rva)$.  // {Unless $h=n$.}  
        \ENDFOR
        
        \STATE Update V-value network (and V-cost networks analogously) by following formula:
        \STATE $\phi_{k+1}=\arg \min _{\phi} \frac{1}{BT} \sum\limits^{B}_{b =1 } \sum\limits_{t=0}^{T}\left(V_{\phi}(s_{t}) - \hat{R_{t}}\right)^{2}$
    \ENDFOR
    \end{algorithmic}
\end{algorithm}

\section{Safe Multi-Agent MuJoCo}
\label{appendix:Introduction-Safe-MAMujoco}
Safe Multi-Agent MuJoCo (SMAMuJoCo) is an extension of Multi-Agent MuJoCo \citep{peng2020facmac}. In particular, the background environment, agents, physics simulator, and the reward function are preserved. However, as oppose to its predecessor, SMAMuJoCo environments come with obstacles, like walls or bombs. Furthermore, with the increasing risk of an agent stumbling upon an obstacle, the environment emits cost \citep{1606.01540}. According to the scheme from \cite{zanger2021safe}, we characterise the cost functions for each task below.


\begin{figure*}[htbp!]
 \centering
\subcaptionbox{}
{
\includegraphics[width=0.44\linewidth]{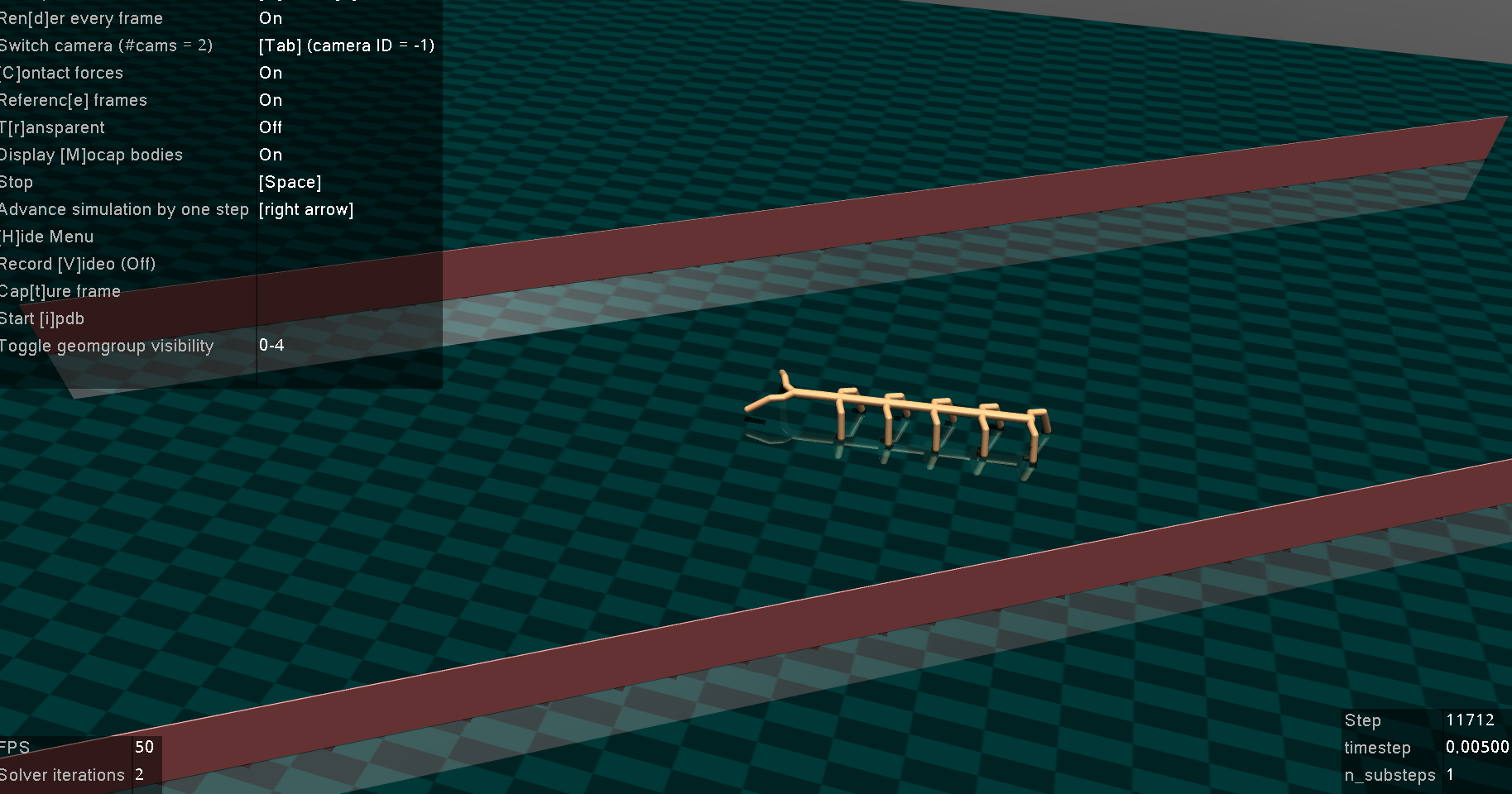}
}\quad \quad \quad
\subcaptionbox{}
{
\hspace{-23pt}
\includegraphics[width=0.44\linewidth]{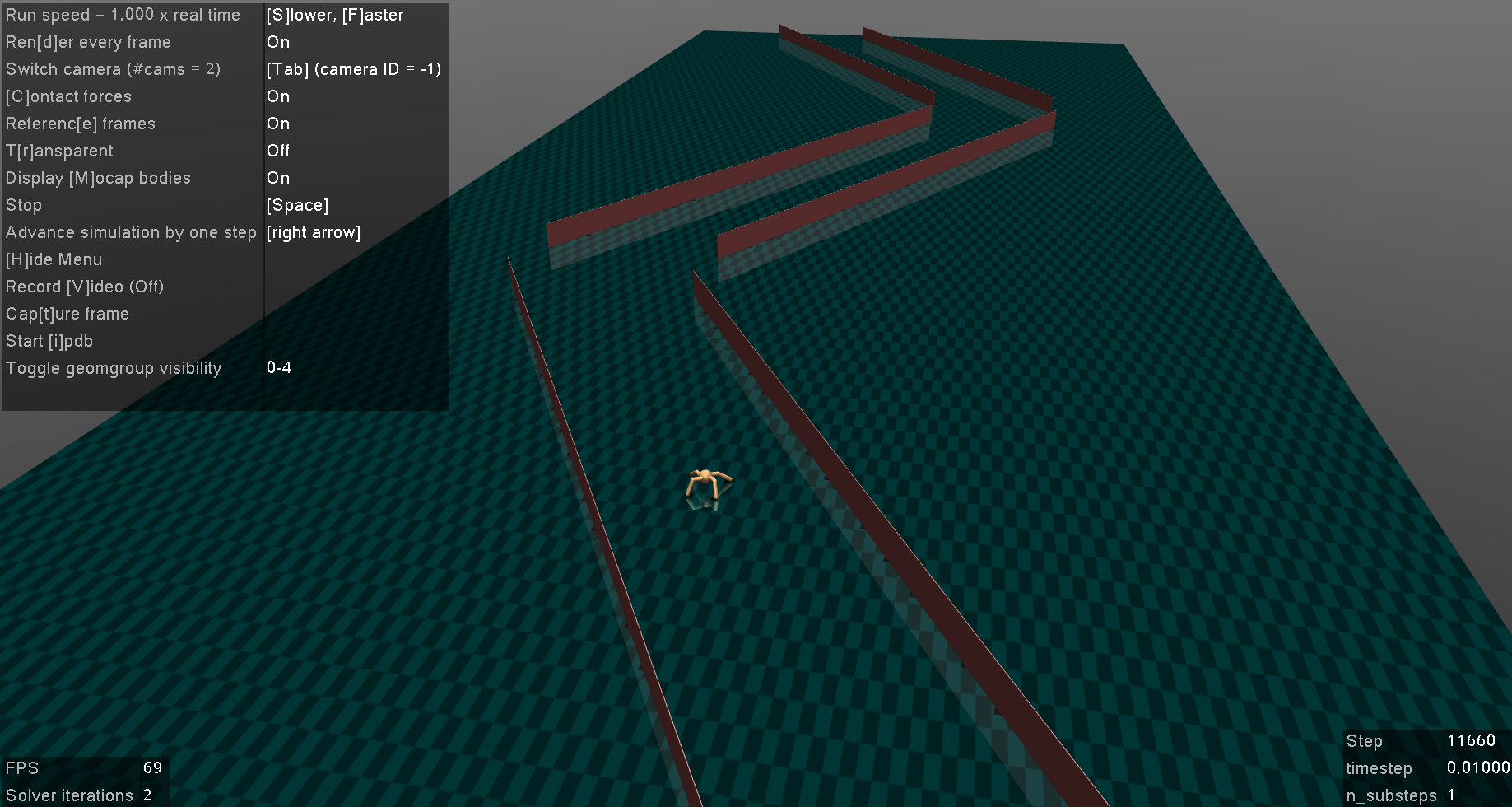}
}\quad \quad \quad
\subcaptionbox{}
{
\includegraphics[width=0.44\linewidth]{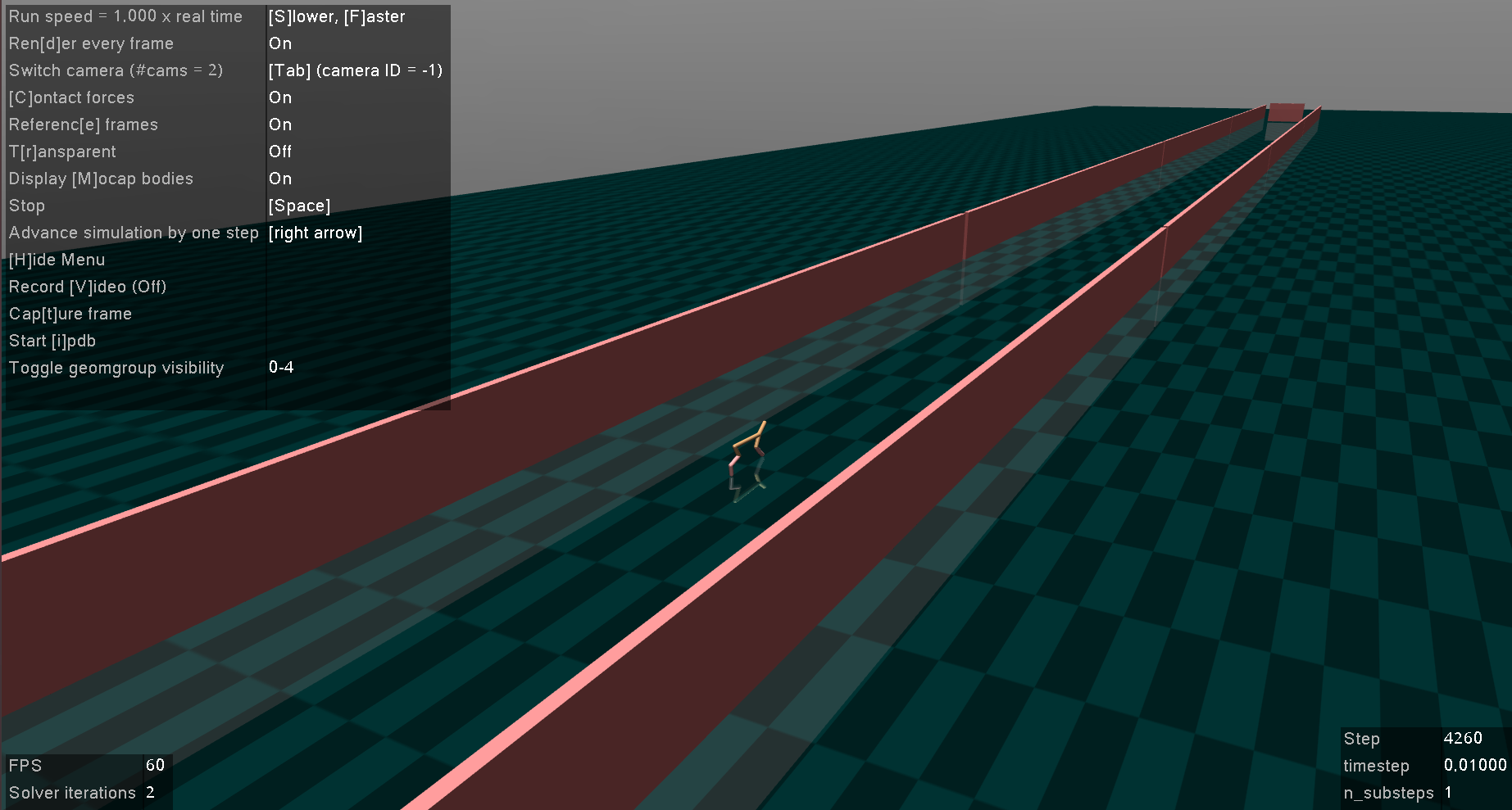}
}
\subcaptionbox{}
{\hspace{6pt}
\includegraphics[width=0.45\linewidth]{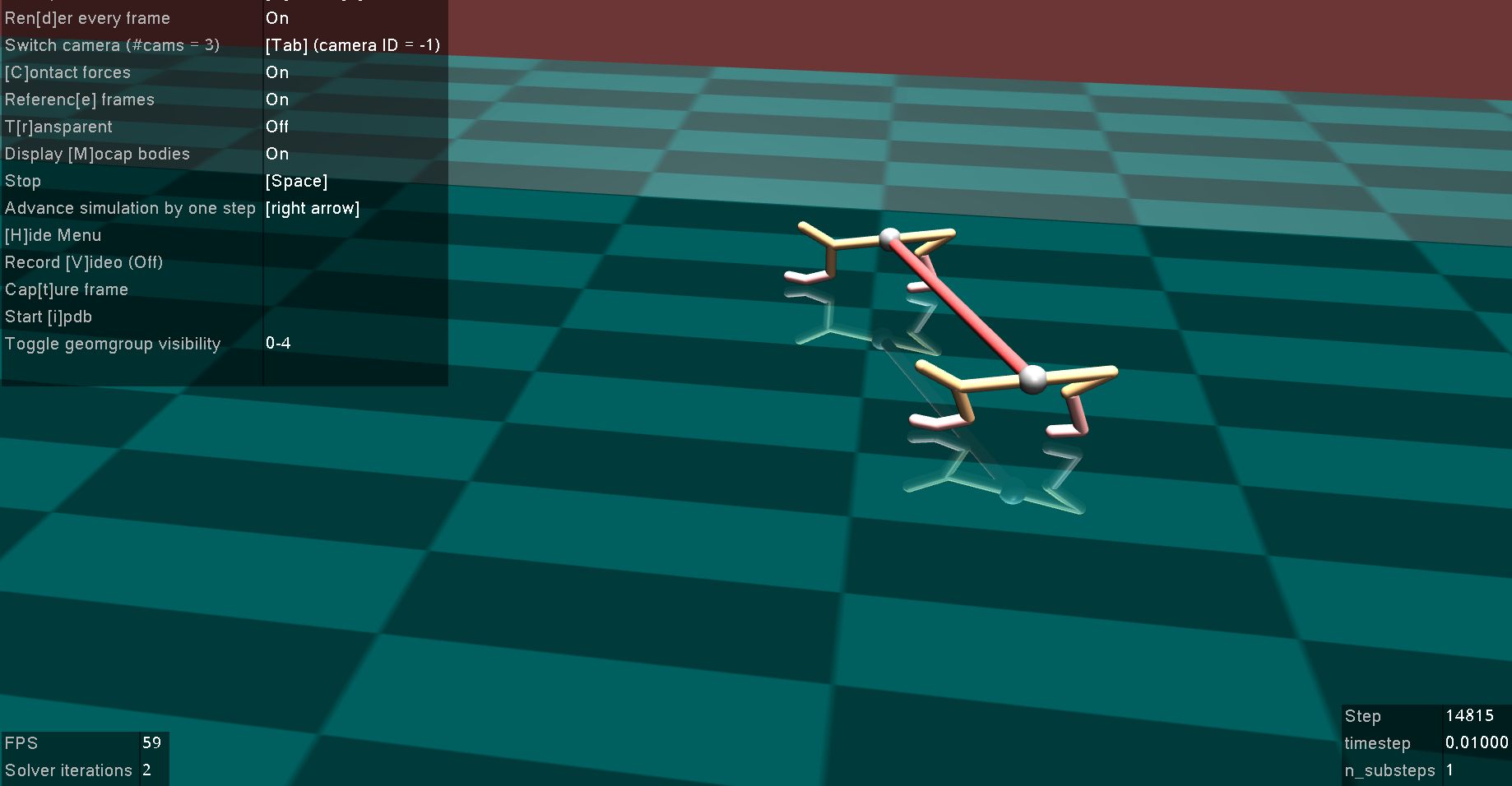}
}
 	\caption{\normalsize Example tasks in SMAMuJoCo Environment. (a): Safe 2x3-ManyAgent Ant with a corridor, (b): Safe 4x2-Ant with three corridors, (c): Safe 2x3-HalfCheetah with a moving obstacle, (d): Safe p1p-CoupleHalfCheetah with a moving obstacle. 
 	} 
 	\label{appendix-fig:Safety-mujoco-all-Environment}
 \end{figure*} 

\subsubsection*{ManyAgent Ant \& Ant}
The width of the corridor set by two walls is $9m$(ManyAgent Ant), The width of the corridor set by three  folding line walls with an angle of 30 degrees is $10 m$(Ant). The environment emits the cost of $1$ for an agent, if the distance between the robot and the wall is less than $1.8 m$, or when the robot topples over (see Figure \ref{appendix-fig:Safety-mujoco-all-Environment} (a) for ManyAgent-Ant task, (b) for Ant task). This can be described as
\begin{align}
\rc_t= \begin{cases}0, & \text { for } \quad 0.2 \leq \boldsymbol{z}_{\text {torso }, t+1} \leq 1.0  \text { and } \left\|\boldsymbol{x}_{\text {torso }, t+1}-\boldsymbol{x}_{\text {wall }}\right\|_{2} \geq 1.8 \nonumber \\ 1, & \text { otherwise } \nonumber.\end{cases}
\end{align}
 where $\boldsymbol{z}_{\text {torso }, t+1}$ is the robot's torso's $z$-coordinate, and $\boldsymbol{x}_{\text {torso }, t+1}$ is the robot's torso's $x$-coordinate, at time $t+1$; $\boldsymbol{x}_{\text {wall }}$ is the $x$-coordinate of the wall.

\subsubsection*{HalfCheetah \& Couple HalfCheetah}
In these tasks, the agents move inside a corridor (which constraints their movement, but does not induce costs). Together with them, there are bombs moving inside the corridor. If an agent finds itself too close to a bomb, the distance between an agent and a bomb is less than $9m$, a cost of $1$ will be emitted (see Figure \ref{appendix-fig:Safety-mujoco-all-Environment} (c) for HalfCheetah task, (d) for CoupleHalfCheetah task).


\begin{align}
\rc_t= \begin{cases}0, & \text { for } \quad  \left\|\boldsymbol{y}_{\text {torso }, t+1}-\boldsymbol{y}_{\text {obstacle}}\right\|_{2} \geq 9 \nonumber \\ 1, & \text { otherwise } \nonumber.\end{cases}
\end{align}
where $\boldsymbol{y}_{\text {torso }, t+1}$ is the $y$-coordinate of the robot's torso,  and $\boldsymbol{y}_{\text {obstacle }}$ is the $y$-coordinate of the moving obstacle.

\clearpage

\section{Safe Multi-Agent Robosuite}
\label{appendix:safe-multi-agent-robosuite}

Safe Multi-Agent Robosuite (SMARobosuite) is developed for MARL community, which we hope it can promote the progress of safe MARL research, and it's on the basis of Robosuite \cite{zhu2020robosuite} which is a popular robotic arm benchmark for single-agent reinforcement learning. SMARobosuite is a fully cooperative, continuous, and decentralised benchmark considering constraints of robot safety, where we use Franka robots to achieve each task, each agent can observe  partial environmental information (such as the velocity and position), SMARobosuite can be easily used for modular robots, make robots have good robustness and scalability. When communication bandwidth is limited, or some joints of robotic arms are broken which causes malfunctional communication, SMARobosuite can still work well. The reward setting is the same as Robosuite \cite{zhu2020robosuite}. The following is the 
cost design for each task.

\begin{figure*}[htbp!]
 \centering
\subcaptionbox{}
{
\includegraphics[width=0.302\linewidth]{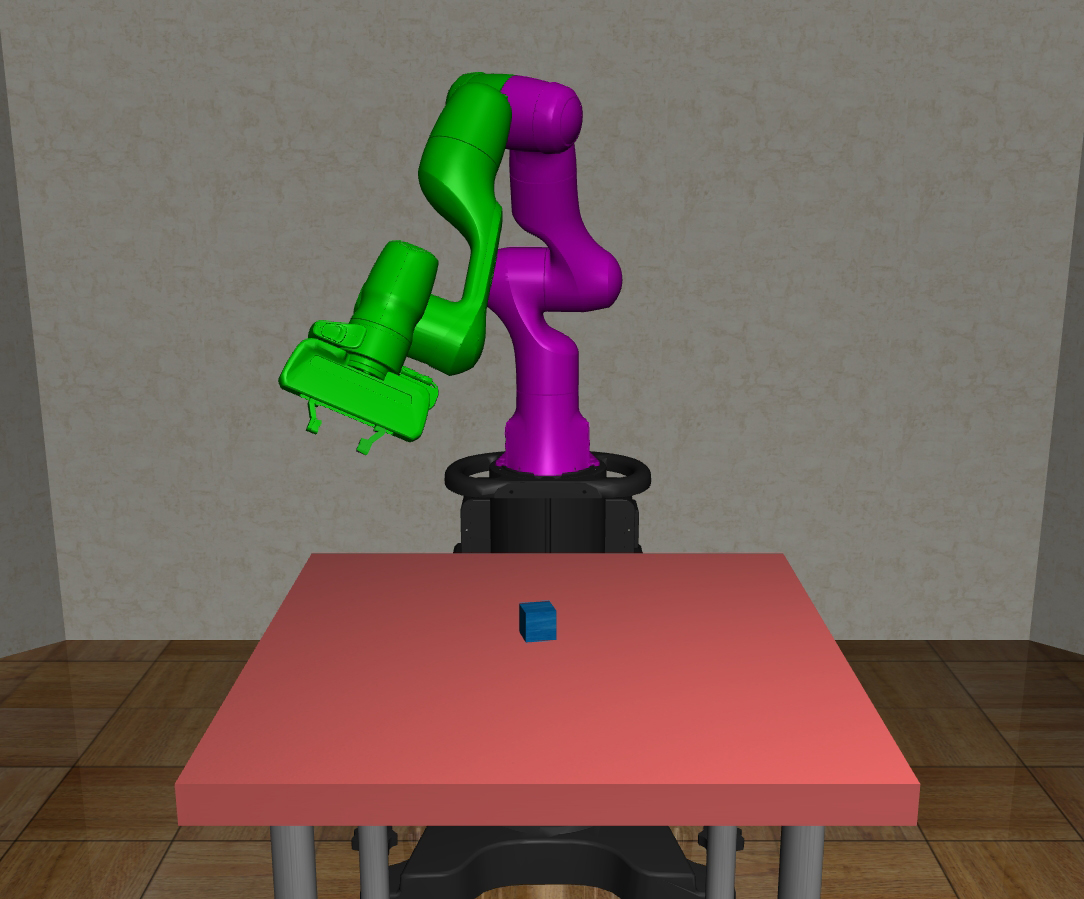}
}\quad \quad \quad
\subcaptionbox{}
{
\hspace{-27pt}
\includegraphics[width=0.273\linewidth]{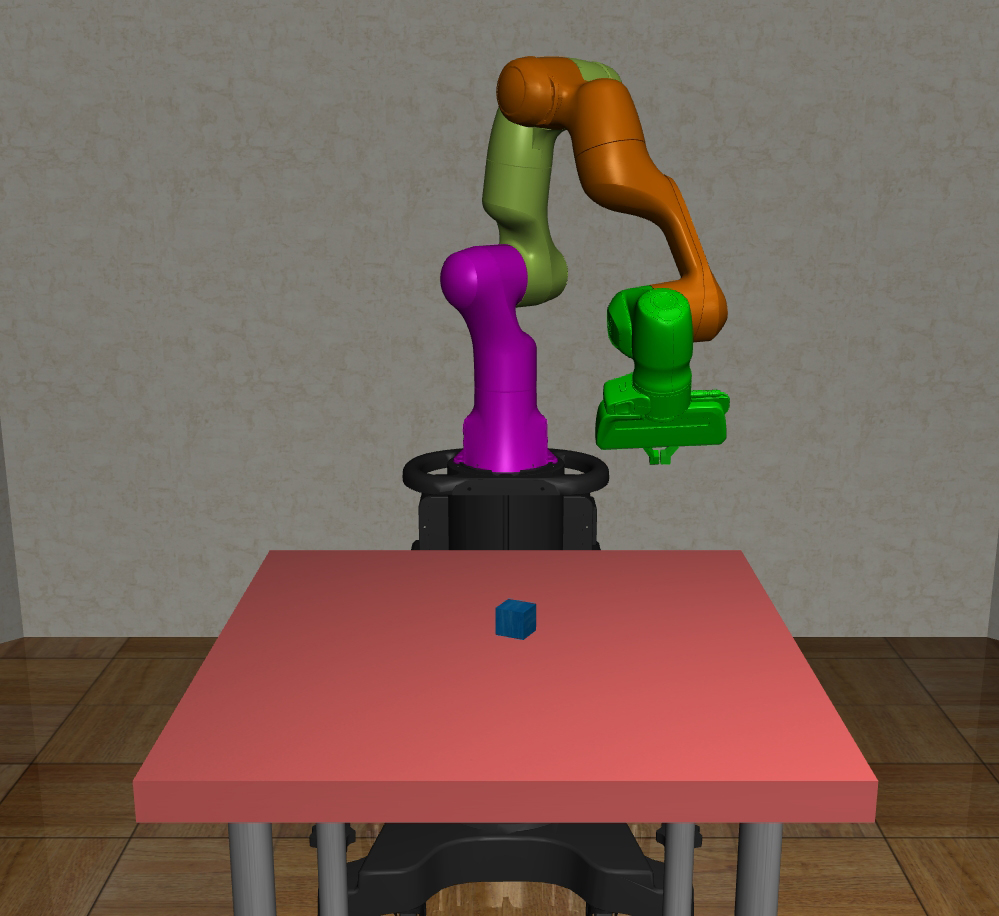}
}\quad \quad \quad
\subcaptionbox{}
{
\hspace{-27pt}
\includegraphics[width=0.302\linewidth]{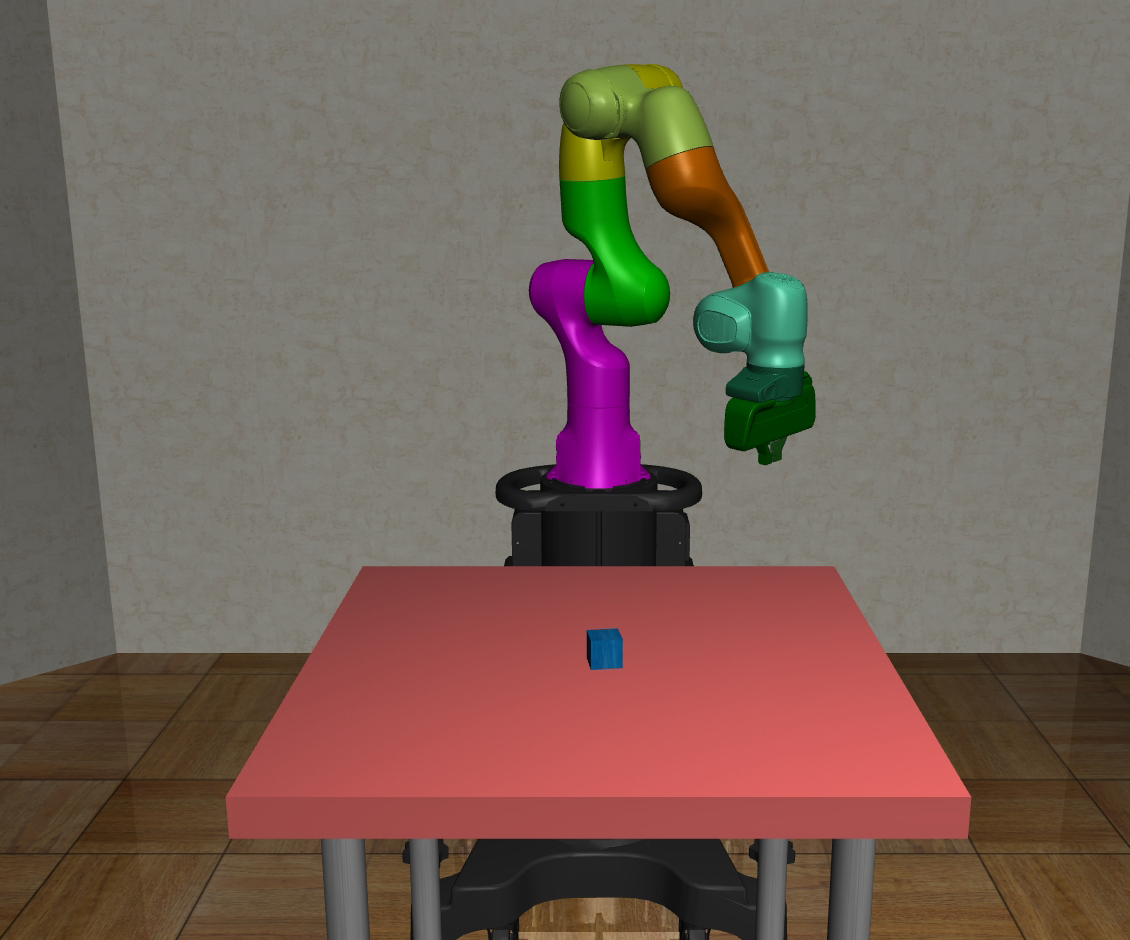}
}

\subcaptionbox{}
{
\includegraphics[width=0.278\linewidth]{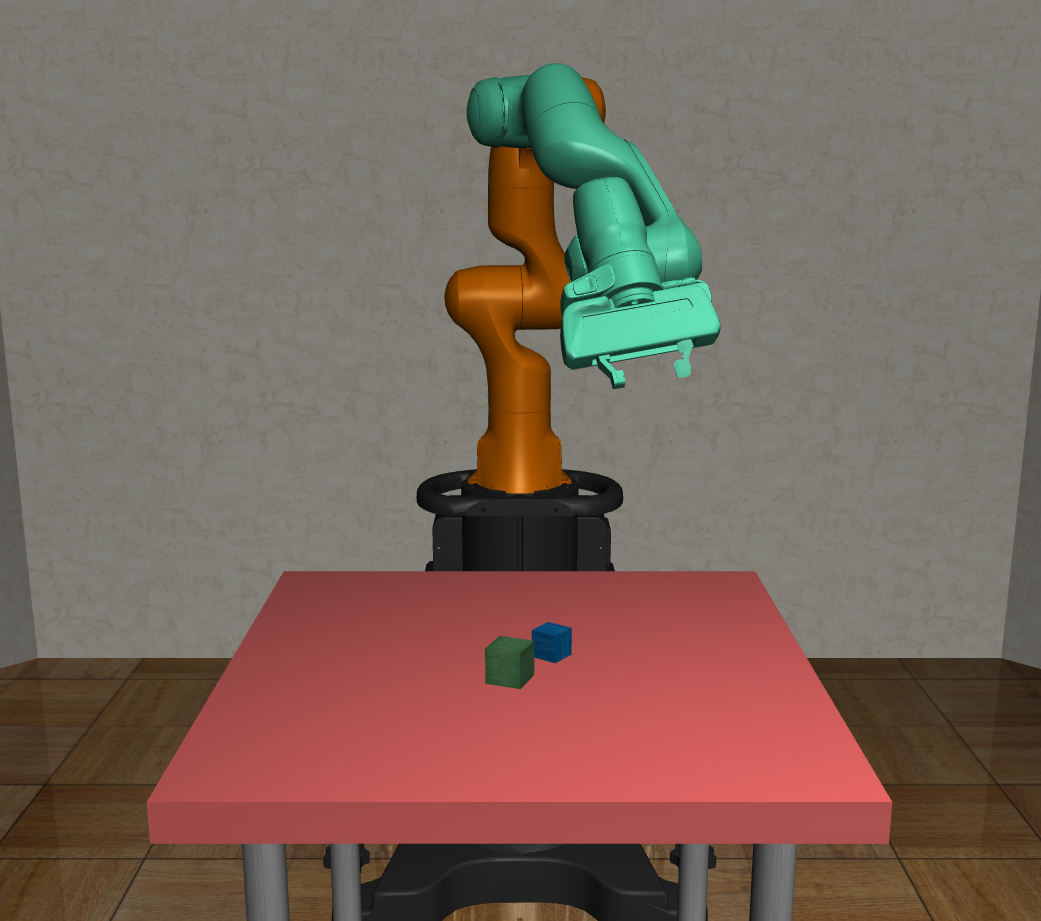}
}\quad \quad \quad
\subcaptionbox{}
{
\hspace{-27pt}
\includegraphics[width=0.31\linewidth]{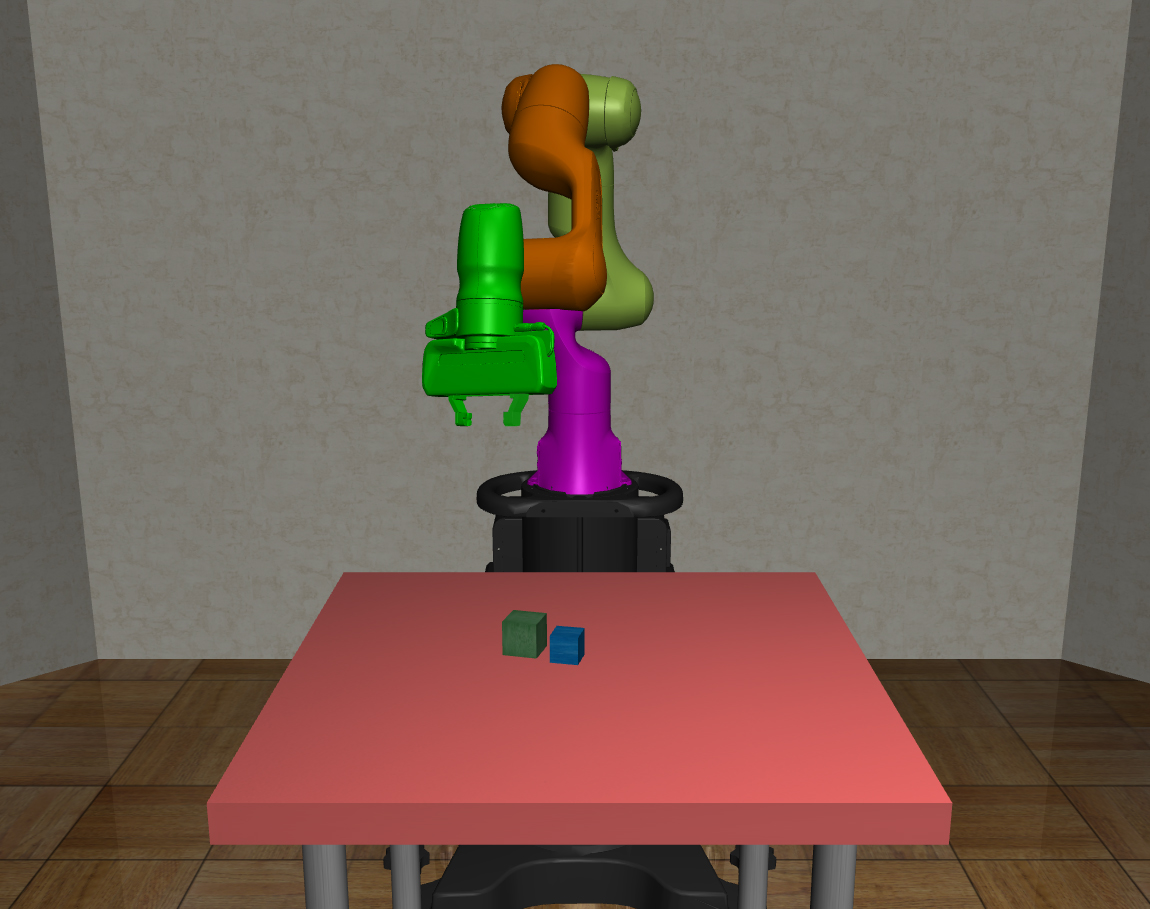}
}\quad \quad \quad
\subcaptionbox{}
{
\hspace{-27pt}
\includegraphics[width=0.294\linewidth]{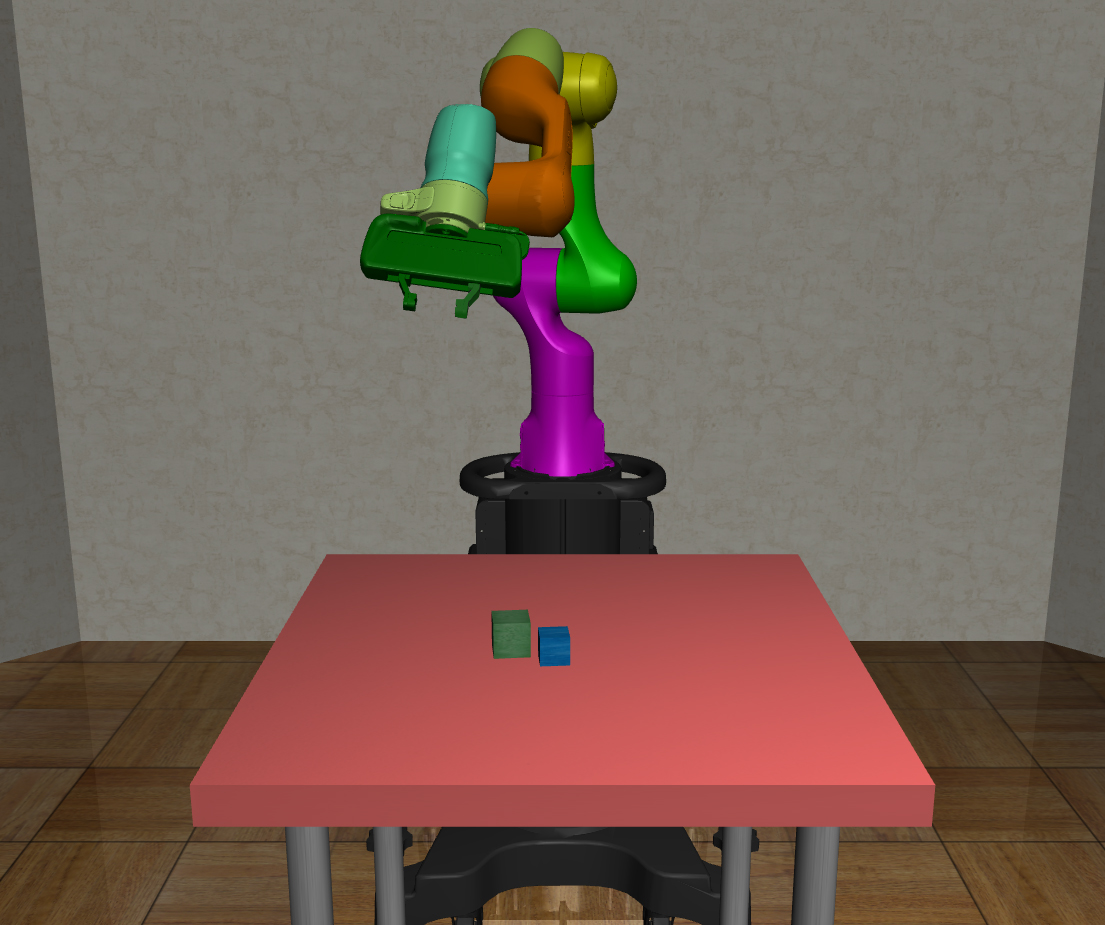}
}
\subcaptionbox{}
{
\includegraphics[width=0.387\linewidth]{Figures/panda-tasks/example-panda/robosuite-twoarmpeginhole-14-color-one-obstacle.jpg}
}
\subcaptionbox{}
{
\includegraphics[width=0.376\linewidth]{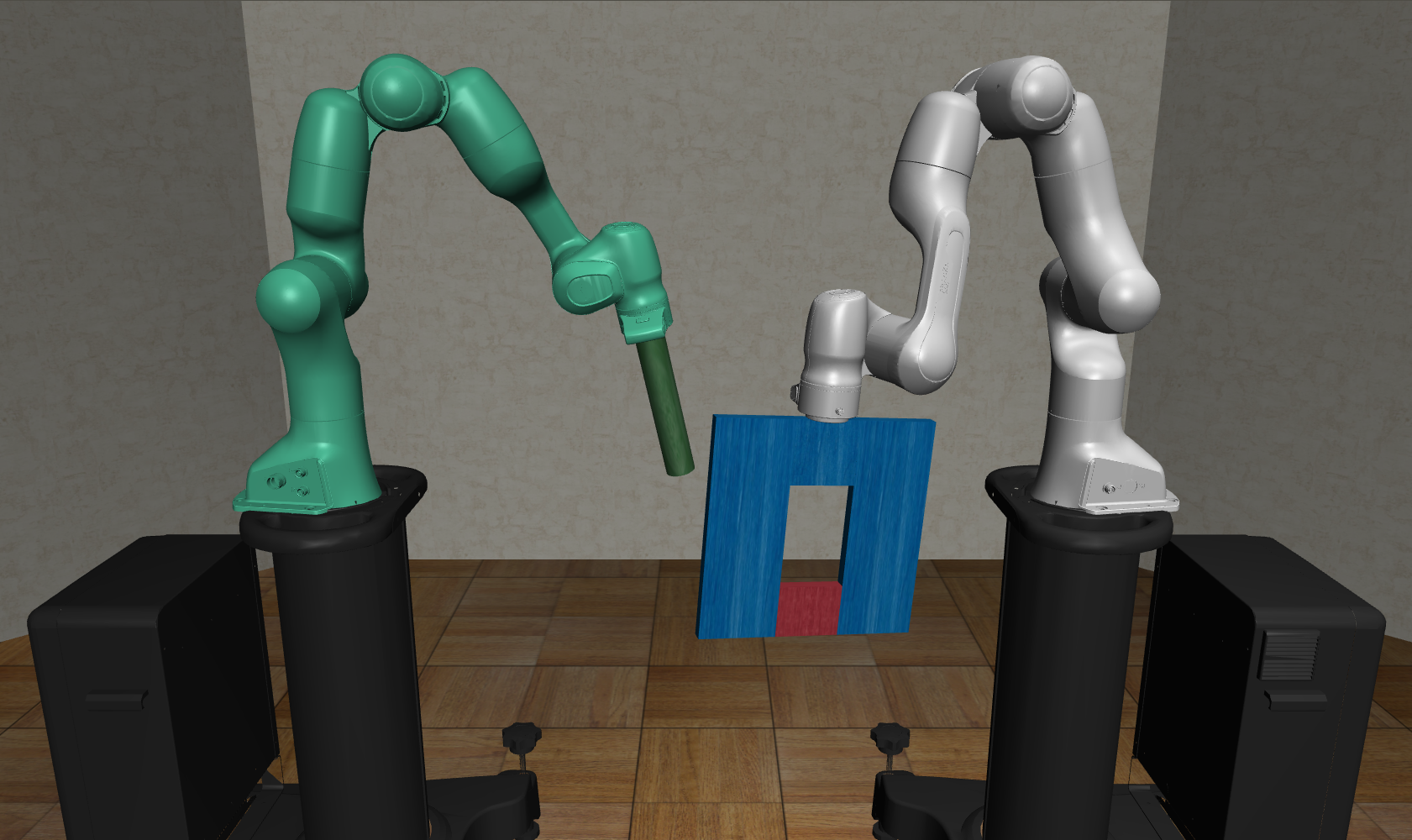}
}
 	\caption{\normalsize Example tasks in SMARobosuite Environment. (a): Safe 2x4-Lift, (b): Safe 4x2-Lift,  (c): Safe 8x1-Lift, (d): Safe 2x4-Stack, (e): Safe 4x2-Stack,  (f): Safe 8x1-Stack, (g): Safe 14x1-TwoArmPegInHole, (h): Safe 2x7-TwoArmPegInHole. Body parts of different colours of robots are controlled by different agents. Agents jointly learn to manipulate the robot, while avoiding crashing into unsafe red areas. 
 	} 
 	\label{appendix-fig:Safety-robosuite-lift-twoarmpeginhole}
 \end{figure*} 

\subsubsection*{Lift \& Stack}
In the tasks, the agents fully cooperatively learn to manipulate robots to grasp  and lift the blue object a certain a distance (\textit{e.g.} at least $0.1~m$) between table and blue object, at the same time, the robot end gripper must keep a certain distance (\textit{e.g.} $0.02~m$) from the table and avoid touch the red table, or it will emit a cost of 1 (see Figure \ref{appendix-fig:Safety-robosuite-lift-twoarmpeginhole} (a), (b), and (c) for Lift, (d), (e), and (f) for Stack).

\begin{align}
\rc_t= \begin{cases}1.0, & \text { for } \quad  |\boldsymbol{p}_{\text {eef }, t+1}-\boldsymbol{z}_{\text {obstacle}}| \leq 0.02 \nonumber \\ 0, & \text { otherwise } \nonumber.\end{cases}
\end{align}
where $\boldsymbol{p}_{\text {eef }, t+1}$ is the $z$-coordinate of the robot's end gripper ,  and $\boldsymbol{z}_{\text {obstacle }}$ is the $z$-coordinate of the  obstacle.

\subsubsection*{TwoArmPegInHole}
 Robots learn to achieve Peg-In-Hole task, specifically, agents need to learn to cooperate fully insert the peg into the
hole, when the peg touches the red areas or the distance between peg and the red areas are less than a certain distance, specifically, if the parallel distance between the robot's peg and robot's hole centre is less than $- 0.11~m$, we think peg reaches the unsafe areas, and it will cause a cost of 1 (see Figure \ref{appendix-fig:Safety-robosuite-lift-twoarmpeginhole} (g) and (h)).


\begin{align}
\rc_t= \begin{cases}1.0, & \text { for } \quad  \boldsymbol{p}_{\text {pd}, t+1}\leq -0.11 \nonumber \\ 0, & \text { otherwise } \nonumber.\end{cases}
\end{align}
where $\boldsymbol{p}_{\text {pd }, t+1}$ is the parallel distance between the robot's peg and robot's hole centre.

\clearpage

\section{Experiments on Safe Many-Agent Ant Environments}
\label{appendix:Experiments of-Different-Safe-ManyAgent-Ant-Environments}

We provide additional results on the Safe Many-Agent ant tasks.

The width of the corridor is $12 m$; its walls fold at the angle of 30 degrees. The environment emits the cost of $1$ for an agent, if the distance between the robot and the wall is less than $1.8 m$, or when the robot topples over. This can be described as
\begin{align}
\rc_t= \begin{cases}0, & \text { for } \quad 0.2 \leq \boldsymbol{z}_{\text {torso }, t+1} \leq 1.0  \text { and } \left\|\boldsymbol{x}_{\text {torso }, t+1}-\boldsymbol{x}_{\text {wall }}\right\|_{2} \geq 1.8 \nonumber \\ 1, & \text { otherwise } \nonumber.\end{cases}
\end{align}
 where $\boldsymbol{z}_{\text {torso }, t+1}$ is the robot's torso's $z$-coordinate, and $\boldsymbol{x}_{\text {torso }, t+1}$ is the robot's torso's $x$-coordinate, at time $t+1$; $\boldsymbol{x}_{\text {wall }}$ is the $x$-coordinate of the wall.

  \begin{figure*}[htbp!]
 \centering
{
\includegraphics[width=0.48\linewidth]{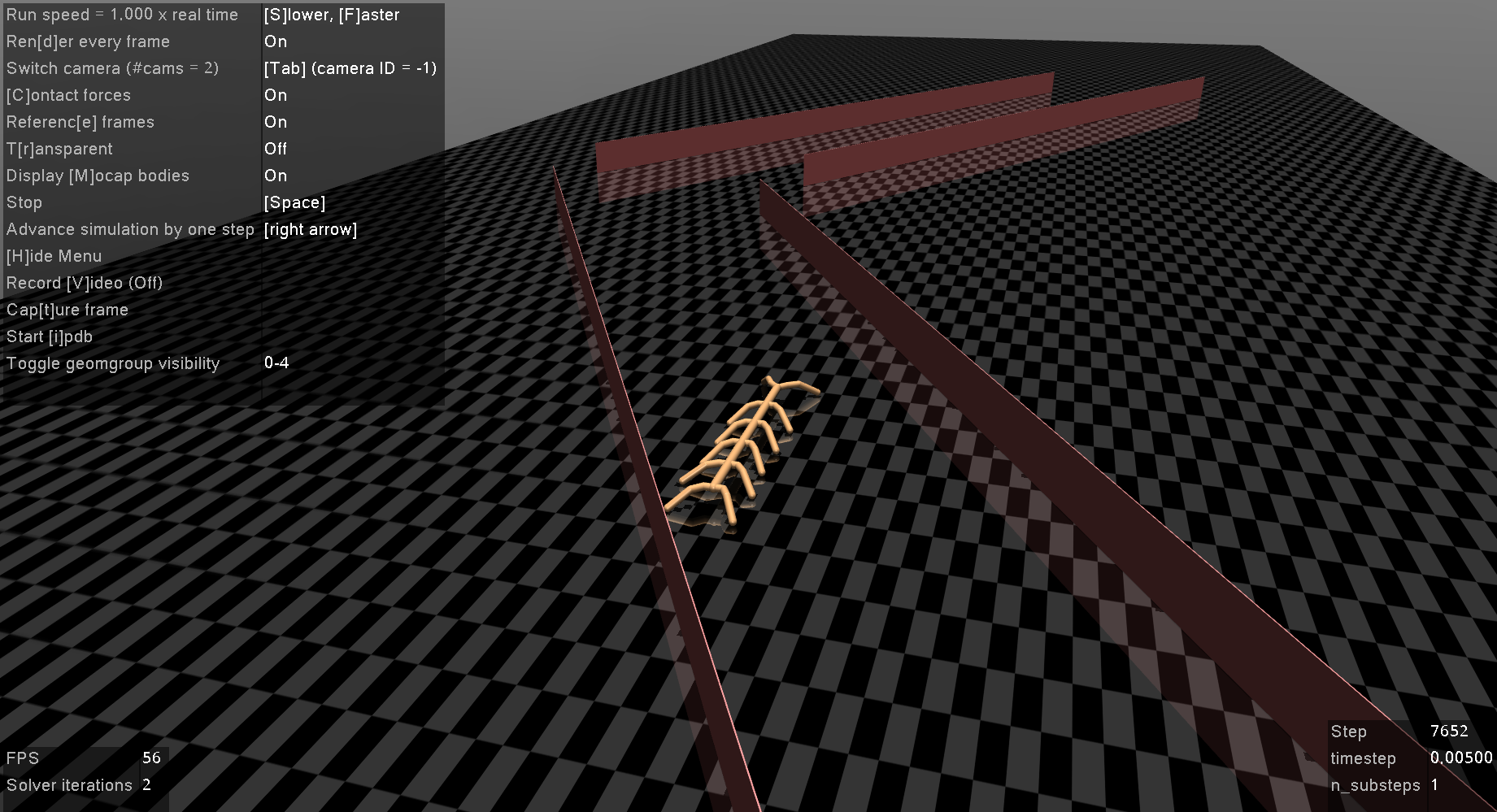}
}
    \vspace{-0pt}
 	\caption{\normalsize Many-Agent Ant 3x2 with two folding line walls. 
 	} 
 	\label{fig:Experiments of-Different-manyagent-ant}
 \end{figure*}

 \begin{figure*}[htbp]
 \centering
{
\includegraphics[width=0.32\linewidth]{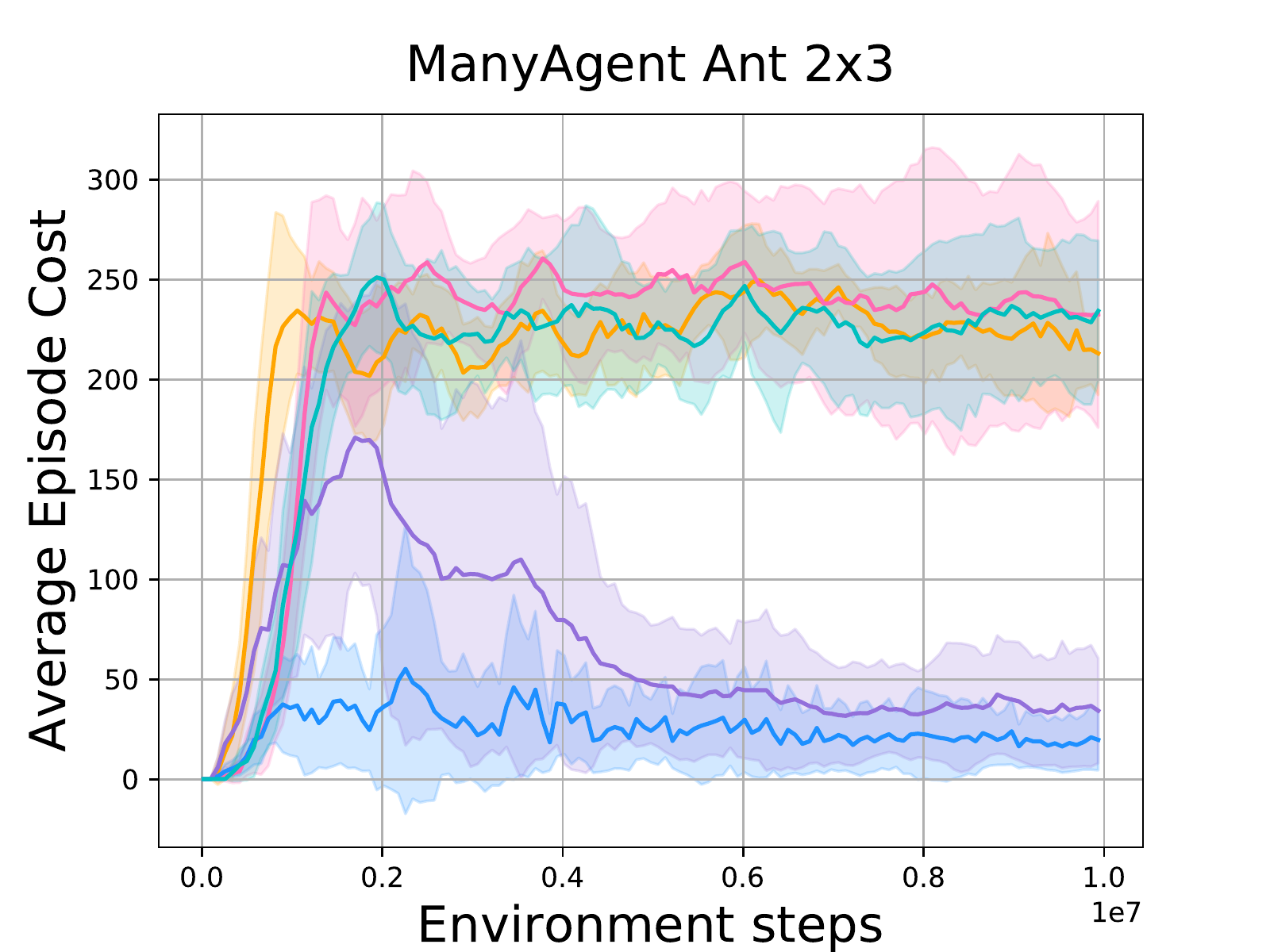}
\includegraphics[width=0.32\linewidth]{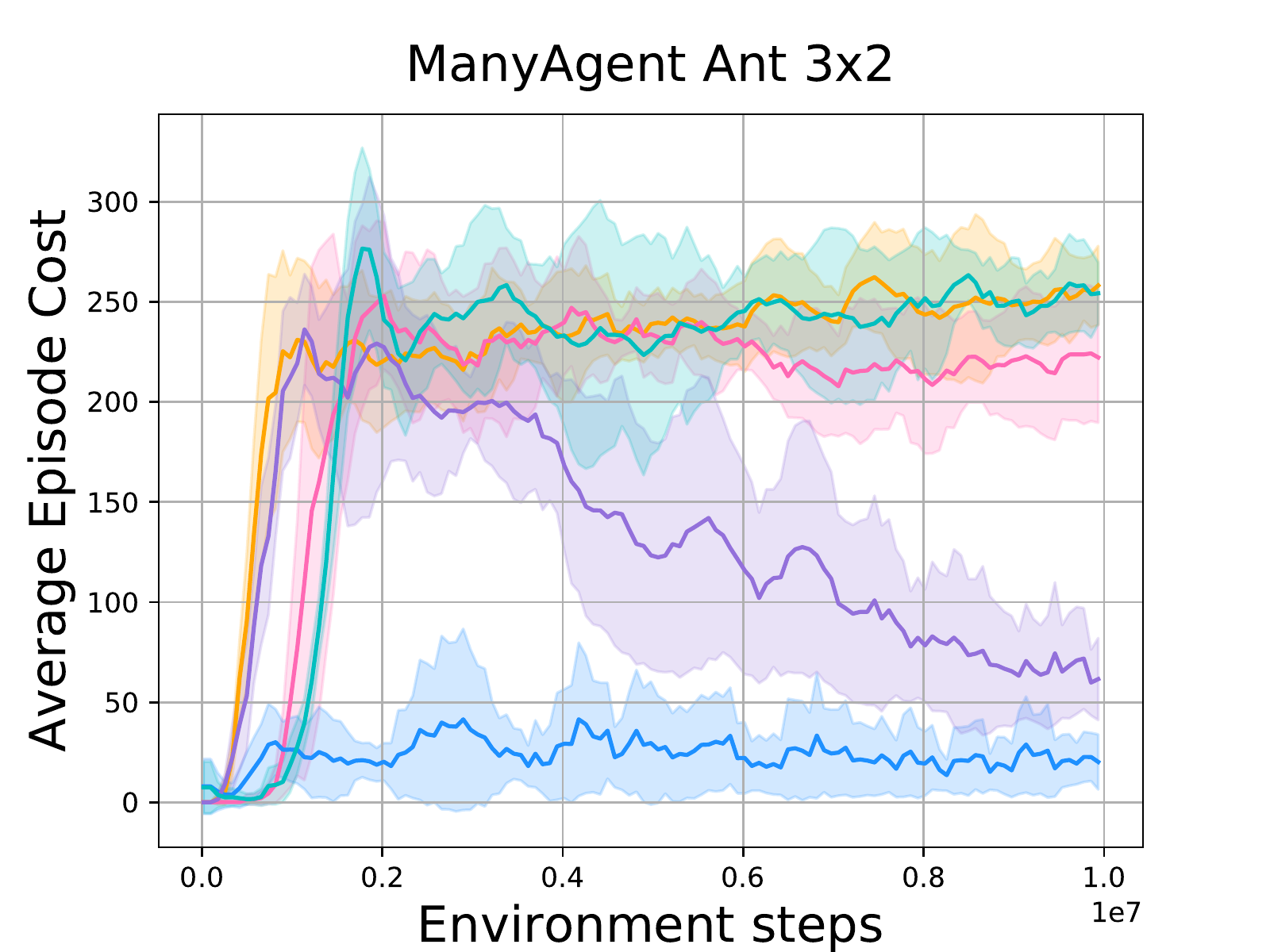}
\includegraphics[width=0.32\linewidth]{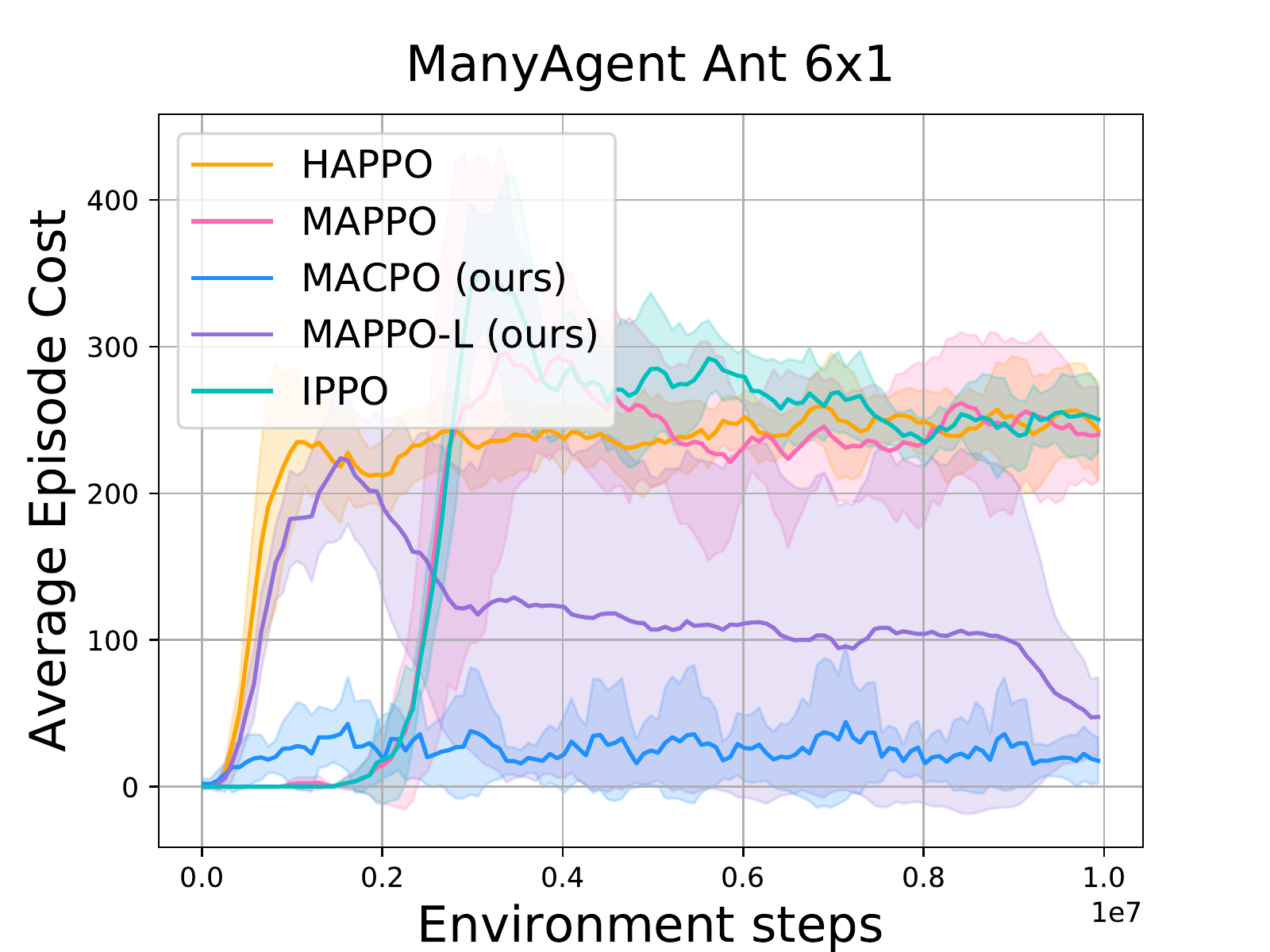}

\includegraphics[width=0.32\linewidth]{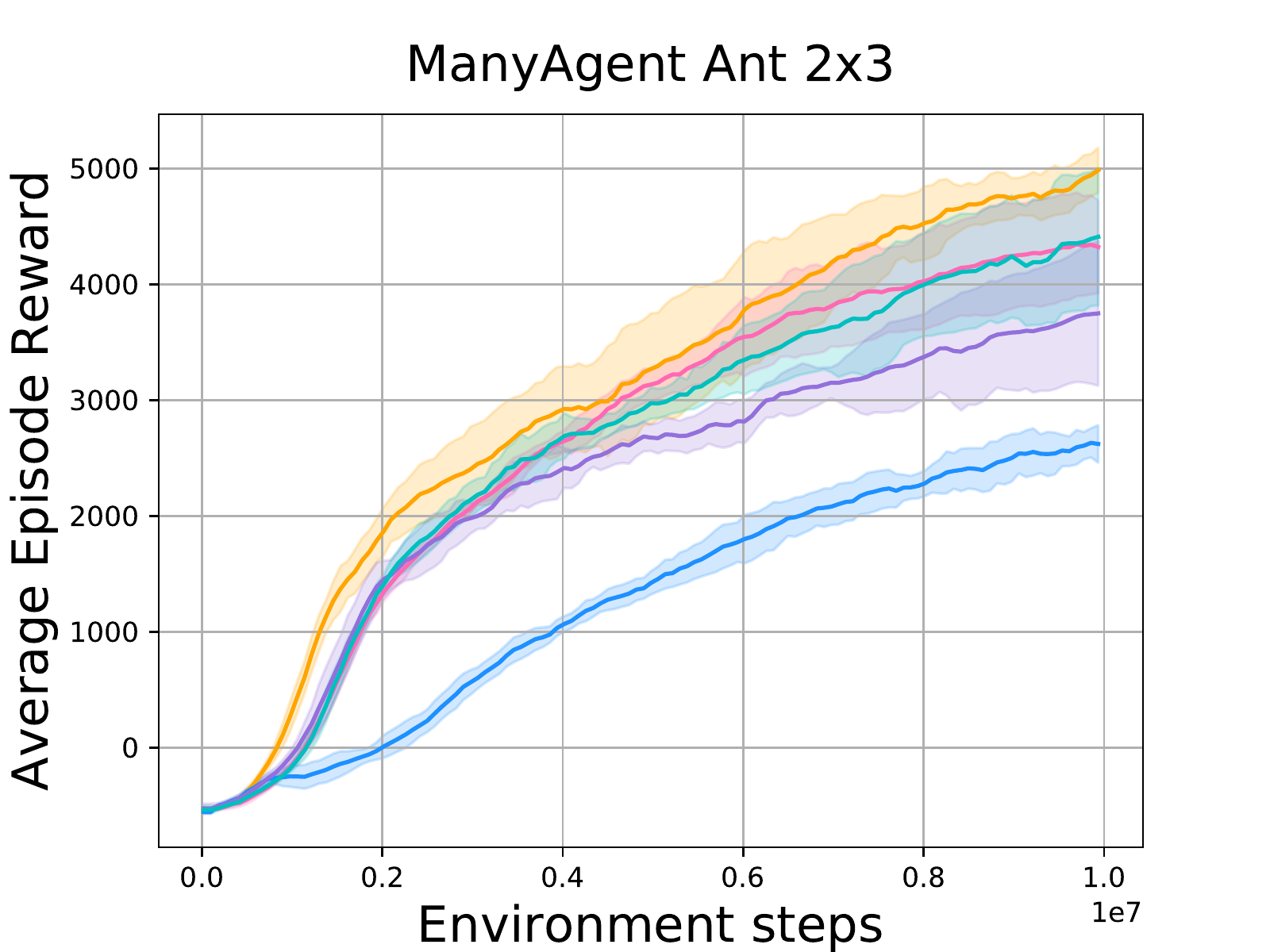}
\includegraphics[width=0.32\linewidth]{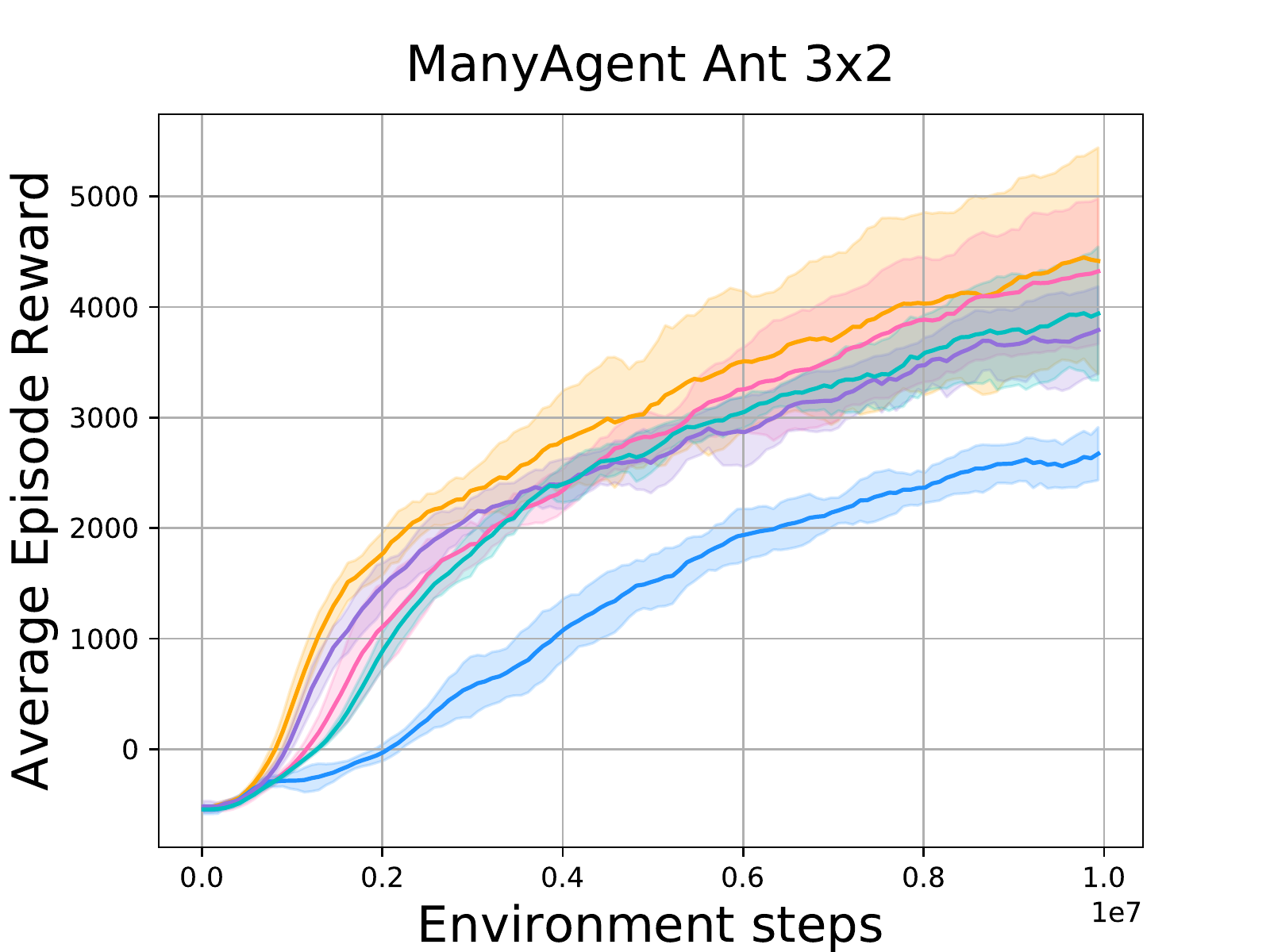}
\includegraphics[width=0.32\linewidth]{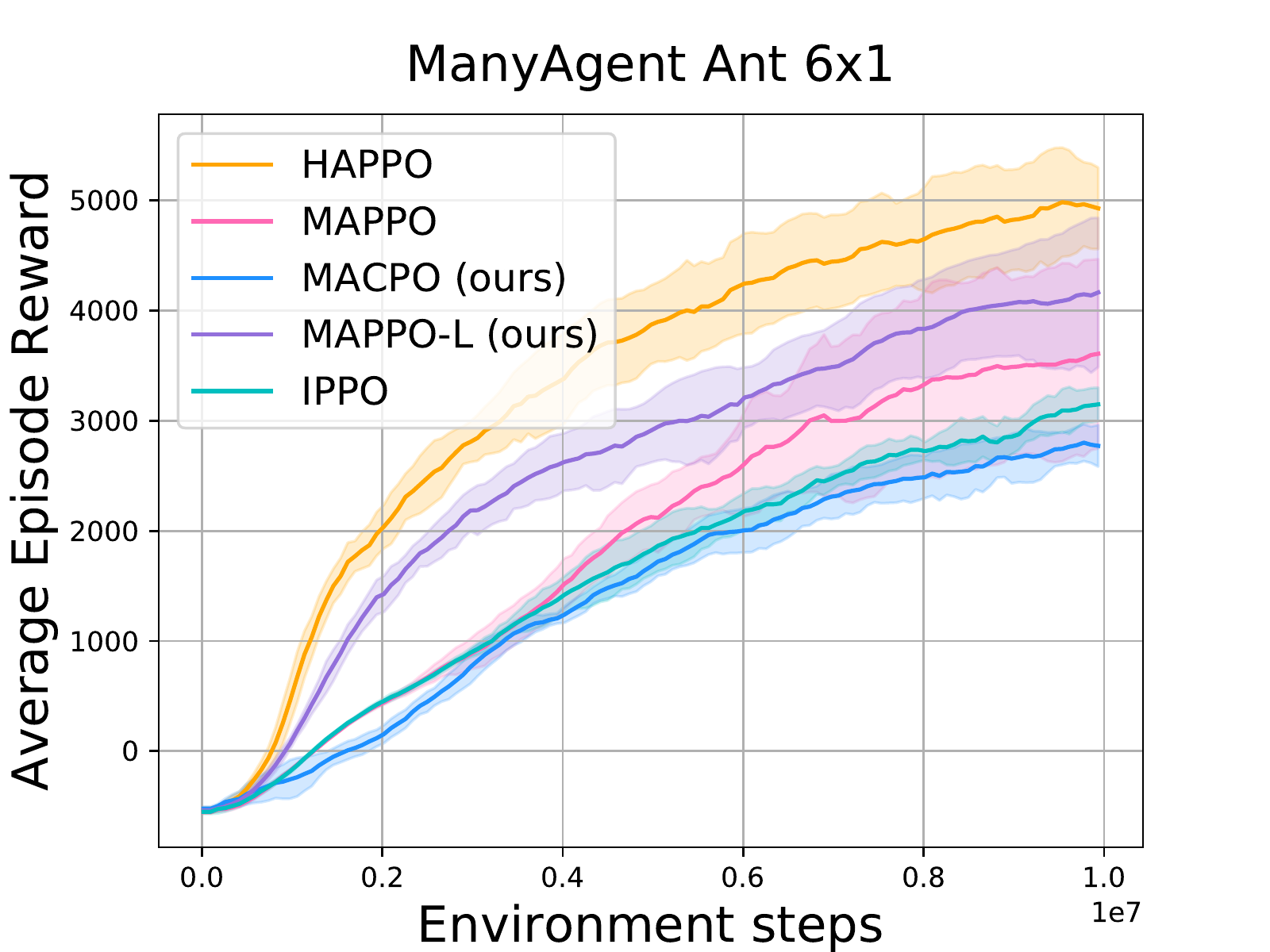}

}

    \vspace{-0pt}
 	\caption{\normalsize Performance comparisons on tasks of Safe ManyAgent Ant in terms of cost (the first row) and reward (the second row). The  safety constraint  values is set to $10$. Our algorithms are the only ones that learn the safety constraints, while achieving satisfying performance in terms of the reward.
 	} 
 	\label{fig:Results-of-Experiments of-Different-manyagent-ant}
 \end{figure*}

\clearpage

\section{Details of Settings for Experiments}
\label{appendix:Details-Settings-Experiments}


In this section, we introduce the details of settings for our experiments. The code is available at \url{https://github.com/chauncygu/Multi-Agent-Constrained-Policy-Optimisation.git}




\begin{table}[htbp]
 \renewcommand{\arraystretch}{1.2}
  \centering
  \begin{threeparttable}
  \label{table3}
    \begin{tabular}{cc|cc|cc}
    \toprule
    hyperparameters & value & hyperparameters & value & hyperparameters & value\\
    \midrule
     critic lr &  5e-3    &    optimizer & Adam       &  num mini-batch & 40 \\
      gamma & 0.99       &       optim eps & 1e-5    &     batch size & 16000\\
      gain & 0.01        &        hidden layer & 1    &     training threads & 4 \\
    std y coef & 0.5    &         actor network & mlp    &     rollout threads & 16\\
  std x coef & 1        &       eval episodes & 32    &   episode length & [1000, 2000]\\
    activation & ReLU      &     hidden layer dim & 64  & max grad norm & 10  \\
    \bottomrule
    \end{tabular}
    \end{threeparttable}
\caption{Common hyperparameters used  for MAPPO-Lagrangian, MAPPO, HAPPO, IPPO, and MACPO in the Safe Multi-Agent MuJoCo and Safe Multi-Agent Robosuite domains (episode length: 1000 used for SMAMuJoCo,2000 used for SMARobosuite ).}
\end{table}

\begin{table}[!htb]
 \renewcommand{\arraystretch}{1.2}
  \centering
  \begin{threeparttable}
  
  \label{table4}
    \begin{tabular}{c|ccccc}
    \toprule
    Algorithms & MAPPO-Lagrangian & MAPPO & HAPPO & IPPO & MACPO\\
    \midrule
    actor lr&  9e-5 &  9e-5 & 9e-5 & 9e-5 & /\\
    ppo epoch & 5& 5 & 5 & 5 & / \\
    kl-threshold & /& / & / & / & [0.0065, 1e-3] \\
    ppo-clip & 0.2 & 0.2 & 0.2 & 0.2 & / \\
    Lagrangian coef & 0.78 & / & / & / & / \\
    Lagrangian lr & 1e-3 & / & / & / & / \\
    fraction & / & / & / & / & 0.5 \\
    fraction coef & / & / & / &  / & 0.27 \\
    \bottomrule
    \end{tabular}
    
    \end{threeparttable}
\caption{Different hyperparameters used for MAPPO-Lagrangian, MAPPO, HAPPO, IPPO, and MACPO (kl-threshold: 0.0065 used for SMAMuJoCo, 1e-3  used for SMARobosuite)  in the SMAMuJoCo and SMARobosuite domains.}
\end{table}

  

    

\begin{table}[!htbp]
 \renewcommand{\arraystretch}{1.2}
  \centering
  \begin{threeparttable}
  
  \label{table6}
    \begin{tabular}{cc|cc|cc}
    \toprule
    task & value & task & value & task & value\\
    \midrule
    Ant(2x4) & 0.2             &       Ant(4x2) & 0.2                  &    Ant(2x4d) & 0.2   \\
    HalfCheetah(2x3) &  5     &       HalfCheetah(3x2) & 5          &  HalfCheetah(6x1)& 5  \\
    ManyAgent Ant(2x3) & 1          &       ManyAgent Ant(3x2) & 1              &     ManyAgent Ant(6x1) & 1 \\

    \bottomrule
    \end{tabular}
    
    \end{threeparttable}
\caption{Safety bound used for MACPO in the SMAMuJoCo  and SMARobosuite (30 used for TWoArmPegInHole) domains.}
\end{table}

\end{document}